\documentclass[10pt,journal,compsoc]{IEEEtran}

\usepackage{amsmath,amssymb,mathtools,amsopn,amsfonts,pdfpages,datetime,setspace,bm}
\usepackage[nice]{nicefrac}
\usepackage{amsthm,thmtools,thm-restate}
\usepackage{array}
\newcolumntype{C}[1]{>{\centering\arraybackslash$}m{#1}<{$}}
\usepackage{algorithmic}
\usepackage{algorithm}
\usepackage{enumitem}
\usepackage{graphicx}
\usepackage{hyperref}
\usepackage{multicol}
\usepackage{multirow}
\usepackage{tikz}
\usetikzlibrary{arrows,shapes,snakes,automata,backgrounds,petri}
\tikzstyle{observation} = [shading=axis, shading angle=10]
\usepackage{subcaption}
\usepackage{todonotes}
\usepackage[nameinlink, noabbrev]{cleveref}
\usepackage{changes}
\newtheorem{lemma}{Lemma}

\newtheorem{definition}{Definition}

\newtheorem*{remark*}{Remark} 

\newtheorem*{lemma*}{Lemma}
\newtheoremstyle{break}
  {\topsep}{\topsep}%
  {\itshape}{}%
  {\bfseries}{}%
  {\newline}{}%
\theoremstyle{break}

\declaretheorem[name=Conjecture,]{conj}

\declaretheorem[
    name=Definition, 
    numberwithin=subsection,
    refname={definition,definitions},
    Refname={Definition,Definitions},
]{defi}

\declaretheorem[
    name=Proposition, 
    numberwithin=subsection,
    refname={proposition,propositions},
    Refname={Proposition,Propositions},
]{prp}

\declaretheorem[
    name=Proposition, 
    refname={proposition,propositions},
    Refname={Proposition,Propositions},
]{prop}

\newcommand{\argmin}{\operatornamewithlimits{argmin}}

\newcommand{\defeq}{\vcentcolon=}
\newcommand{\eqdef}{=\vcentcolon}
\newcommand{\ErgoVEC}{ErgoVEC}
\newcommand{\GramVEC}{GramVEC}
\newcommand{\GramPMI}{GramPMI}
\newcommand{\NucGramVEC}{NucGramVEC}
\newcommand{\GramErgoVEC}{GramErgoVEC}
\newcommand{\ErgoPMI}{ErgoPMI}
\newcommand{\GramErgoPMI}{GramErgoPMI}
\newcommand{\NucGramErgoVEC}{NucGramErgoVEC}

\ifCLASSOPTIONcompsoc
  \usepackage[nocompress]{cite}
\else
  \usepackage{cite}
\fi

\newif\ifcomments
\commentstrue

\ifcomments
    \usepackage{ulem}
    \usepackage{xcolor}
    \def\piedit#1{{$\!$\color{orange} [\textbf{PI:} #1]}}
    \def\dsedit#1{{$\!$\color{purple} [\textbf{DS:} #1]}}

    \definecolor{applegreen}{rgb}{0.55, 0.71, 0.0}
    \def\cledit#1{{$\!$\color{applegreen} [\textbf{CL:} #1]}}
	\definecolor{ballblue}{rgb}{0.13, 0.67, 0.8}
	\definecolor{lightred}{rgb}{0.8, 0.13, 0.97}    
	
\else 
    \def\piedit#1{}

     \def\cledit#1{}
     \def\dsedit#1{}
\fi

\allowdisplaybreaks

\begin{document}
\title{Ergodic Limits, Relaxations, and Geometric Properties of Random Walk Node Embeddings}

\author{Christy Lin,
        Daniel Sussman,
        and~Prakash Ishwar
\IEEEcompsocitemizethanks{\IEEEcompsocthanksitem Christy Lin is with the Division of Systems Engineering, College of Engineering, Boston University, 15 St Marys St, Boston, MA, 02215. \protect\\

cy93lin@bu.edu
\IEEEcompsocthanksitem Daniel Sussman is with the Department of Mathematics \& Statistics, College of Arts and Sciences, Boston University, 111 Cummington Mall, Boston, MA, 02215. 
sussman@bu.edu
\IEEEcompsocthanksitem Prakash Ishwar is with the Division of Systems Engineering and Department of Electrical \& Computer Engineering, College of Engineering, Boston University, 8 St Marys St, Boston, MA, 02215.
pi@bu.edu
}
}

\IEEEtitleabstractindextext{%
\begin{abstract}
Random walk based node embedding algorithms learn vector representations of nodes by optimizing an objective function of node embedding vectors and skip-bigram statistics computed from random walks on the network. They have been applied to many supervised learning problems such as link prediction and node classification and have demonstrated state-of-the-art performance. Yet, their properties remain poorly understood.
This paper studies properties of random walk based node embeddings in the unsupervised setting of discovering hidden block structure in the network, i.e., learning node representations whose cluster structure in Euclidean space reflects their adjacency structure within the network.
We characterize the ergodic limits of the embedding objective, its generalization, and related convex relaxations to derive corresponding non-randomized versions of the node embedding objectives. We also characterize the optimal node embedding Grammians of the non-randomized objectives for the expected graph of a two-community Stochastic Block Model (SBM). We prove that the solution Grammian has rank $1$ for a suitable nuclear norm relaxation of the non-randomized objective.
Comprehensive experimental results on SBM random networks reveal that our non-randomized ergodic objectives yield node embeddings whose distribution is Gaussian-like, centered at the node embeddings of the expected network within each community, and concentrate in the linear degree-scaling regime as the number of nodes increases.
\end{abstract}
}

\maketitle

\IEEEdisplaynontitleabstractindextext

\IEEEpeerreviewmaketitle


\IEEEraisesectionheading{\section{Introduction}
\label{Sec:introduction}}

\IEEEPARstart{M}{ost} statistical and computational tools originally developed for vector-valued data do not leverage the unique structured form of network data. Tools that exploit the graph-structure of network data could be custom-made for each network problem. A powerful alternative, however, is to develop a Euclidean-space embedding of a network that enables methods and tools developed for Euclidean-space data to effectively reason about various network properties. 

{\bf Node embedding} algorithms~\cite{cai2018comprehensive} aim to map nodes of a given graph into points in Euclidean space (i.e., vectors in $\mathbb{R}^d$) such that their relative positions capture their propensities for adjacency within the network. These embeddings 

make it possible to apply to network data, tools and algorithms from multivariate statistics and machine learning that were developed for Euclidean-space data. For example, with suitable embeddings, node classification, community detection, and vertex nomination problems reduce, respectively, to standard classification, clustering, and ranking problems. Therefore, developing new node embedding algorithms, establishing the theoretical properties of these embeddings, and demonstrating how connectivity properties are reflected in the embedding space is fundamental to developing principled network inference procedures.

{\bf Random walk embeddings}~\cite{perozzi2014deepwalk,yang2016revisiting,grover2016node2vec,ding2017node}  are a class of recently developed node embedding techniques which use random walks on graphs to capture notions of proximity between nodes. They may be viewed as network counterparts of techniques used for learning {\bf word embeddings}~\cite{mikolov2013distributed, pennington2014glove} in the field of natural language processing. In fact, by viewing samples of random walks in the network as sentences, with nodes playing the role of words, word embeddings can be directly applied to networks to yield node embeddings. Nodes which appear nearby within a sample of a random walk are analogous to words that appear nearby within a sentence.  
Word embeddings have been found to accurately capture the relationships between words and have been highly successful in several natural language processing tasks such as topic modeling, translation, and word analogy~\cite{bakarov2018survey}.
Random walk node embeddings too have been applied to a number of supervised and unsupervised learning problems such as link prediction, node classification and community detection and have demonstrated state-of-the-art performance~\cite{perozzi2014deepwalk,yang2016revisiting,grover2016node2vec,ding2017node}.

Unfortunately, despite excellent empirical performance in a number of supervised learning problems, random walk embeddings remain poorly understood. This is in stark contrast to the well-known spectral embeddings whose properties for the \textit{unsupervised} learning problem of community detection have been extensively studied and characterized under a variety of statistical network models, specifically the Stochastic Block Model (SBM) and its generalizations~\cite{rohe2011spectral,sussman2012consistent,qin2013regularized,Athreya2017-hn,chaudhuri2012spectral,cape2019spectral}. Attempts of theoretical analysis so far have focused on building connections between random walk embeddings algorithms and matrix factorization \cite{qiu2018network}. The properties of the resulting embedding vectors, however, still remain unexplored.

{\bf Contributions:} This paper proposes a framework for 
random walk based node-embedding algorithms for graphs. This is based on learning node embeddings by optimizing objective functions involving skip-bigram statistics computed from random walks on a graph. This framework subsumes several existing algorithms as special cases and introduces extensions and techniques that simplify theoretical analysis.
We establish ergodic limits of the proposed node-embeddings.
We analyze Grammian re-parameterized convex relaxations and characterize the solution for the expected graph of a two-community SBM and the unconstrained solution for any graph. We prove that the solution of the expected graph of a two-community SBM has rank at most $2$.
We develop algorithms for computing solutions to our proposed embedding objectives
for general graphs and conduct numerical experiments to understand the geometric structure  of embedding vectors (community clustering and separation properties) for SBM random graphs. We also empirically study the concentration properties of node embeddings for SBM random graphs in the linear and logarithmic scaling regimes. We find empirically that the distribution of embeddings are Gaussian-like, centered at the node embeddings of the expected graph within each community, and that they concentrate in the linear degree scaling regime as the number of nodes increases.

{\bf Paper organization:} 
Section~\ref{Sec:related_background} overviews recent work on random walk embeddings, sets up basic notation, and provides background on SBMs.
Section~\ref{Sec:alt_formulation} describes our proposed theoretical framework, results on ergodic limits (\Cref{Subsec:ErgoVEC}), various relaxations (\Cref{Subsec:PMI}), and the characterization of the solution for the expected graph of a two-community SBM  (\Cref{Subsec:expected_obj}).
\Cref{Sec:exp_setups} describes the setting of all our experiments in full detail. The geometric and concentration properties of the distribution of embedding vectors of our proposed algorithms under 2-community SBM are presented and discussed in \Cref{Sec:experiments_sbm}. Concluding remarks appear in Section~\ref{Sec:concl}.

\noindent{\bf Notation:} In this work we consider graphs that are undirected and simple with a node set $\mathcal{V} = [n] := \{1, 2,\dotsc,n\}$ and an edge set $\mathcal{E} \subset \{\{i,j\}: i,j\in V, i\neq j\}$. The edges may be possibly weighted. We denote such a graph by $\mathcal{G} = (\mathcal{V}, \mathcal{E})$ and its adjacency matrix by $A\in \{0,1\}^{n\times n}$, where $A_{ij}=0$ if, and only if, $\{i, j\} \in \mathcal{E}$. 

We denote the set of all real numbers by $\mathbb{R}$, the set of all natural numbers by $\mathbb{N}$, the set of all $n \times n$ real symmetric matrices by $\mathbb{S}^n$, the set of all real, symmetric, and positive semidefinite matrices by $\mathbb{S}^n_{+}$, and the natural logistic-loss function by $\sigma(t) \defeq \ln(1+e^{-t}), t \in \mathbb{R}$. Matrix transpose is denoted by $^\top$.

%

\section{Background and related work}
\label{Sec:related_background}

In this section we overview recent work on random walk node embeddings with a focus on the unsupervised algorithm VEC. We also summarize key aspects of the Stochastic Block Model (SBM) used in our experiments.

\subsection{Random walk node embedding algorithms} \label{Subsec:node_emb_algs}
A random walk node embedding algorithm typically consists of three steps: 
1) Generating multiple random walks over the graph via Markov chains 
with the set of nodes as the state space, 
specified probability transition matrices 
at each step, 
and specified initial distributions. 
2) Computing various statistics
from the sample paths of the random walks.
3) Generating embeddings by optimizing a function that only involves the computed statistics and node embedding variables of the input graph.

Among the random walk node embedding algorithms,
\cite{perozzi2014deepwalk,grover2016node2vec,tang2015line} make use of node embeddings within the context of \textit{supervised} learning problems such as node attribute prediction and link prediction and accordingly design probability transition matrices that depend on the supervised labels. 
In contrast, the VEC algorithm \cite{ding2017node} focuses on the \textit{unsupervised} community detection problem \cite{girvan-newman2002pnas}. 
The unsupervised setting of \cite{ding2017node} is ideal for studying random walk node embeddings that capture pure network connectivity properties unsullied by node labels. We therefore select VEC as our prototypical algorithm for analysis and introduce it in detail in the next subsection.

While our focus is on unsupervised setting, the general Markov-Chain based framework we develop can be used to analyze the supervised setting as well through transition matrices that are label dependent. 

In addition to the node embedding algorithms discussed above, the use of a random walks and their steady-state-distributions for graph clustering has been studied in \cite{spielman2004nearly} and \cite{andersen2006local}. Subsequent work \cite{lambiotte2014random} further proposed to exploit multi-step transition probabilities between nodes for clustering.

In terms of theoretical results, \cite{meng2020analysis} have analyzed the stationary distribution of second-order random walks in \cite{grover2016node2vec} for specific types of networks. 
We provide a complete characterization of the ergodic limits for general random walk node embedding objectives in Section~\ref{Sec:alt_formulation}.
For the task of community detection, \cite{zhang2021consistency} have provided large-sample error bounds for consistent community recovery from the perspective of matrix factorization. Their setting is a special, unconstrained case of our general problem stated in \Cref{Defi:GramErgoPMI} of \Cref{Subsec:PMI}.

\subsection{VEC: unsupervised random walk node embedding}
\label{Subsec:vec}
VEC learns a low-dimensional vector representation for each node of a graph such that the local neighborhood structures of the graph are encoded within the Euclidean geometry of node vectors. Specifically, the inner product between the embedding vectors of node pairs encode their propensity to appear nearby in random walks on the graph. 

VEC generates $r$ random walks on $\mathcal{G}$ of fixed length $\ell$ starting from each node. We let $\{X_s^{(m,p)}\}_{s=1}^{\ell}, p = 1, \ldots, r$, denote the $p$-th random walk starting from node $m$. All random walks follow the ``natural'' transition matrix $W$ where the next node is chosen from the immediate neighbors of the current node with probability proportional to the edge weight between them.

VEC learns node embedding using the negative-sampling framework of noise-contrastive estmation~\cite{mikolov2013distributed}.
The statistics used for learning node embeddings are based on two multisets of node pairs that are computed from the sample paths of the random walks as follows.
The positive multiset $\mathcal{D}_+$ consists of all node pairs $(X^{(m,p)}_{s},X^{(m,p)}_{s'})$, including repetitions, that occur within $w$ steps of each other, i.e., $|s-s'| \leq w$, in all the generated sample paths. Such node pairs are called $w$-skip bigrams in Natural Language Processing with words viewed as nodes and sentences as sample paths of random walks.
The algorithm parameter $w$ controls the size of the local neighborhood of a node in the given graph. 
The negative multiset $\mathcal{D}_-$ is constructed as follows. For each node pair $(i,j)$ in $\mathcal{D}_+$, we append $k$ node pairs $(i, j_1), \ldots, (i, j_k)$ to $\mathcal{D}_-$, where the $k$ nodes $j_1,\ldots,j_k$ are drawn in an IID manner from {\it all} the nodes according to the empirical unigram node distribution computed from all the sample paths.
Let  $n_{ij}^{+}$ and $n_{ij}^{-}$ denote the number of $(i,j)$ pairs, counting repetitions, in $\mathcal{D}_{+}$ and $\mathcal{D}_{-}$ respectively. 

VEC finds the embedding vector $\bm{u}_i \in \mathbb{R}^d$ for each node $i$ by solving the following minimization problem:
\begin{definition}[VEC optimization problem] \label{Defi:VEC}
\begin{eqnarray}
\underset{\{\mathbf{u}_{i} \in \mathbb{R}^d, i \in \mathcal{V} \}}{\arg\min}
\sum_{(i,j)\in \mathcal{V}^2 }\left[n_{ij}^+ \;\sigma(\mathbf{u}_{i}^{\top}\mathbf{u}_{j})
  + n_{ij}^{-}\; \sigma( -\mathbf{u}_{i}^{\top}\mathbf{u}_{j} )\right]
\label{Eq:VEC}
\end{eqnarray}
\end{definition}

One approach to solve Eq.~\eqref{Eq:VEC} is via stochastic gradient descent (SGD) \cite{bottou2010large,bottou-98x}. This approach is followed in \cite{mikolov2013distributed} and implemented in Python \texttt{gensim} package. Besides its conceptual simplicity, SGD  can be parallelized and nicely scaled to large datasets~\cite{recht2011hogwild}. The per-iteration computational complexity of the SGD algorithm used to solve Eq.~\eqref{Eq:VEC} is $O(d)$, i.e., linear in the emebdding dimension. The number of iterations is $O(r\ell wk)$.

\subsection{Stochastic Block Model}
\label{Subsec:SBM}
The Stochastic Block Model (SBM) \cite{white1976social, SBMoriginal, boppana1987eigenvalues} is a canonical generative probabilistic model for random graphs that reflects block (community) structures among the nodes wherein nodes within the same block have the same tendencies for connecting to all other nodes. 
Free of node or edge labels, it serves as a clean platform for generating graphs to empirically study and compare the properties of various node embedding algorithms and conduct a theoretical analysis. For example, SBM has helped in understanding the behavior of spectral embeddings \cite{Athreya2017-hn}.

For any given $K \in \mathbb{N}$, a $K$-block SBM is parameterized by the latent block membership labels $y_1, \dotsc, y_n \in [K]$,
and the edge probability matrix, a symmetric matrix $B\in [0,1]^{K \times K}$.
The latent labels $\{y_i\}$ partition the nodes into communities indexed by each $k\in [K]$. 
We note that there are versions of SBM in which the $y_i$'s are treated as random. This, however, poses minor additional difficulties. To ease the subsequent discussion, unless noted otherwise, the $y_i$'s will always be viewed as fixed deterministic unknowns throughout this work.
For a node in block $k_1$ and a \textit{different} node in block $k_2$ (where $k_2$ may equal $k_1$), the probability that an edge is present between the two nodes is $B_{k_1k_2}$, and all edges appear independently. We use this model for generating graphs in all our experiments.

The goal of any community detection algorithm is to learn the latent communities of nodes purely from the graph structure. Thus community detection is an unsupervised learning problem which aims to uncover the underlying block structure.
A series of work \cite{Abbe2016-bw, abbe2015, abbe2016nips, decelle2011KS, BP14,Mossel2015-na} characterizes the information-theoretic limits of community detection in SBMs in different degree-scaling regimes. Some of our experiments are designed to operate with respect to these information-theoretic limits.
%

\section{Analytical framework and results}
\label{Sec:alt_formulation}

There are three distinct challenges which complicate the analysis of VEC embedding vectors and their relationship to the latent graph community structure.
First, the objective function Eq.~(\ref{Eq:VEC}) is nonlinear due to the logistic loss function.
Second, even though the function $\sigma(t)$ is strictly convex, the overall objective is not convex with respect to the node embedding vectors. 
Finally, the objective function is itself random, partly due to intrinsic randomness in network connectivity, but also due to algorithmic randomness from the random walks and the Stochastic Gradient Descent algorithm.  

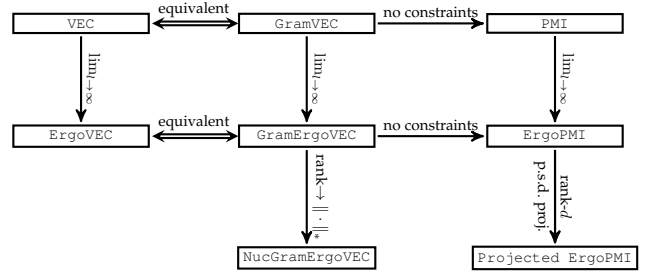
\begin{figure}[!hbt]
\centering 

\begin{tikzpicture}[scale=0.6,->,>=stealth', shorten >=1pt,auto,node distance=5cm,
  thick, column sep=1cm, every node/.style={scale=0.6},
  main node/.style={font=\sffamily, draw=none, fill=gray}, 
  p1 node/.style={font=\ttfamily,rectangle, draw=black, fill=none, minimum width = 30mm}, 
  p2 node/.style={font=\ttfamily,rectangle, draw=black, fill=none},
  edge_in/.style={<-,shorten <=1pt, thick},
  edge_out/.style={->,shorten >=1pt, thick},
  edge_eqv/.style={<->, double, >=stealth}]

  \node[p1 node] (Vec) {VEC};
  \node[p1 node] (GramVec) [right of=Vec] {GramVEC}
     edge [edge_eqv] node [above]{equivalent} (Vec);
  \node[p1 node] (PMI) [node distance = 5.5 cm, right of=GramVec] {PMI}
     edge [edge_in] node [above, sloped] {no constraints} (GramVec);
  \node[p1 node] (ErgoVec) [node distance = 2.5cm, below of=Vec] {ErgoVEC}
  	 edge [edge_in] node [above, sloped]{$\lim_{l\to\infty}$} (Vec);
  \node[p1 node] (GramErgoVec) [node distance = 2.5cm, below of=GramVec] {GramErgoVEC}
  	 edge [edge_in] node [above, sloped]{$\lim_{l\to\infty}$} (GramVec)
  	 edge [edge_eqv] node [above]{equivalent} (ErgoVec);
  \node[p1 node] (ErgoPMI) [node distance = 5.5 cm, right of=GramErgoVec] {ErgoPMI}
  	 edge [edge_in] node [above, sloped] {no constraints} (GramErgoVec)
  	 edge [edge_in] node [above, sloped]{$\lim_{l\to\infty}$} (PMI);
  \node[p1 node] (NucGramErgoVec) [node distance = 2.7 cm, below of=GramErgoVec] {NucGramErgoVEC}
  	 edge [edge_in] node [above, sloped] {rank$\to\|\cdot \|_*$}(GramErgoVec);
  \node[p1 node] (EigenErgoPMI) [node distance = 2.7 cm, below of=ErgoPMI] {Projected ErgoPMI}
  	 edge [edge_in] node [above, sloped]{rank-$d$} (ErgoPMI)
  	 edge [edge_in] node [below, sloped]{p.s.d. proj.} (ErgoPMI);


 \end{tikzpicture}
\caption{Relationships between analysis strategies. }
\label{fig:alg_cartoon}
\end{figure}
To tackle these challenges, in this section we introduce and develop techniques, generalized formulations, and their extensions which are more amenable to theoretical analysis. We leverage three distinct strategies whose inter-relationships are succinctly depicted in Fig.~\ref{fig:alg_cartoon}. These are:
\begin{itemize}
    \item[(1)] \textit{{\ErgoVEC}: Ergodic limits of random walks} ($\lim_{\ell\rightarrow \infty}$). We begin by noting that the sampled coefficients, $n_{ij}^+$'s and $n_{ij}^-$'s in Eq.~(\ref{Eq:VEC}), inherit the randomness of the random walks and depend on a number of algorithm parameters that are described in Sec.~\ref{Subsec:vec}. Previous empirical results \cite{ding2017node} demonstrate that the parameters such as the number of random walks $r$ and their length $\ell$ do not substantially impact performance. Motivated by this observation, as a first step, in Sec.~\ref{Subsec:ErgoVEC}, we eliminate algorithmic randomness by taking the ergodic limits ($\ell \rightarrow \infty$) of the coefficients. This gives rise to a more principled formulation, which we call {\it \ErgoVEC}, that removes dependence on sampled random walks and parameters $r$ and $\ell$.
    \item[(2)] \textit{{\GramErgoVEC} and PMI: reparameterize, unconstrain and project}. Like VEC, {\ErgoVEC} is a nonconvex optimization problem since the objective is a noconvex function of the embedding vectors $\mathbf{u}_i$. We leverage a re-parametrization trick which is similar in spirit to that used in \cite{levy2014neural} to arrive at an equivalent problem, named {\GramErgoVEC}, that has a convex objective function with respect to new matrix variables and additional constraints. 
    {\GramErgoVEC} has a convex objective, but is still a nonconvex optimization problem due to the rank constraint. In order to gain insight into the structure of the solution, we characterize the solution to {\GramErgoVEC} without any constraints and then project the unconstrained solution onto the constraint set. It turns out that the solution to the unconstrained {\GramErgoVEC} objective is directly related to the so-called Pointwise Mutual Information (PMI) matrix \cite{church1990word}. We study {\GramErgoVEC} and PMI in Sec.~\ref{Subsec:PMI}
    \item[(3)] \textit{{\NucGramErgoVEC}: reparameterize and convexify}. 
    Another strategy to convexify {\GramErgoVEC} is to replace the non-convex rank constraint by a convex nuclear norm constraint. We term the resulting optimization problem {\it \NucGramErgoVEC} and study its properties in the later part of Sec.~\ref{Subsec:PMI}
    \end{itemize}

In the rest of this section, we will formally study and establish important theoretical properties of these alternative formulations and their inter-relationships.

\subsection{Ergodic limits} \label{Subsec:ErgoVEC}

As described in Sec.~\ref{Subsec:vec}, $n_{ij}^+$ and $n_{ij}^-$ are the number of the $(i,j)$ node pairs in the positive and negative multisets, $\mathcal{D}_+$ and $\mathcal{D}_-$, respectively. These depend on $5$ algorithm parameters: $r$ (number of random walks per node), $\ell$ (length of each walk), $w$ (context window size) and $k$ (number of negative $w$-skip bigrams per positive $w$-skip bigram). Specifically, $n_{ij}^+$, as a $w$-skip bigram count over $r$ IID sets of $n$ random walks, increases proportionally with $r$ and $n$ and roughly proportionately with $\ell$, for large $\ell$, since
the number of segments of $w$ consecutive steps in a length-$\ell$ walk equals $(l-w+1)$.
They also increase as $w$ increases however their distribution can change substantially with $w$.
As for the negative multiset, note that $\vert\mathcal{D}_-\vert = k\vert\mathcal{D}_+\vert $, so $n_{ij}^-$ increases proportional to $k,r$, and $\ell$. Among these parameters, the results in \cite{ding2017node} show that $r$ and $\ell$ have little effect on the final performance of VEC, while $w$ plays a more important role.

Besides their dependence on the algorithm parameters, $n_{ij}^+$'s and $n_{ij}^-$'s inherit the randomness intrinsic to the random walks. Additionally, the number of negative $(i,j)$ pairs $n_{ij}^-$ also inherit randomness from the categorical sampling of appended nodes $j_1,\dots,j_k$ in the negative pairs. 
In order to gain an algorithmic-randomness-free understanding of network properties captured by $n_{ij}^+$ and $n_{ij}^-$, we study their {\it ergodic limits}.

\begin{definition}[Ergodic limits of $n_{ij}^+$ and $n_{ij}^-$] \label{Defi:Ergolim}
	{\noindent}Let $n_{ij}^+$ and $n_{ij}^-$ be defined as above. The (normalized) ergodic limits of $n_{ij}^+$ and $n_{ij}^-$ are defined as
	\begin{align}
		\label{Eq:Def_Ergolim+}
		\bar{n}_{ij}^+:={}&\frac{1}{rn}\lim\limits_{\ell\rightarrow\infty}\frac{n_{ij}^+}{\ell},\\
		\label{Eq:Def_Ergolim-}
		\bar{n}_{ij}^-:={}&\frac{1}{rn}\lim\limits_{\ell\rightarrow\infty}\frac{n_{ij}^-}{\ell},
	\end{align}
 whenever these limits exist in the almost sure sense.
\end{definition}

The Ergodic limits in \Cref{Defi:Ergolim} provide, for a given graph, a deterministic version of $n_{ij}^+$ and $n_{ij}^-$, normalized by the cumulative length of all random walks. We note that letting $\ell$ go to infinity may seem like incorporating global information about the entire graph instead of the more useful local connectivity patterns, but this is not the case. Regardless of the value of $\ell$, $\mathcal{D}_+$ only contains pairs of nodes which appear within $w$ steps from each other. Therefore, the positive pairs sampled still reflect local information. 

In VEC we launch $r$ random walks starting deterministically from each node which yields a total of $rn$ random walks. Dividing the $w$-skip bigram counts by $rn$ averages them across all random walks. The averaged counts can be loosely viewed as arising from a single random walk with a uniform initial distribution over nodes, i.e., with probability $1/n$ for each node. If the Markov Chain underlying the random walk is ergodic, as $\ell$ tends to infinity, the $n_{ij}^+$'s and $n_{ij}^-$'s, suitably normalized, will converge to their respective expected values under the sampling distribution of random walks.
This intuition is formalized in Theorem~\ref{Thm:Ergolim} below. The theorem encompasses disconnected graphs that consist of several connected components that are often encountered in practice. In such cases, the random walks can be launched within and confined to each connected component. The theorem also covers the case where edges in the graph have real-valued (non-binary) nonnegative weights. The theorem provides explict closed-form expressions for the ergodic limits $\bar{n}_{ij}^+$ and $\bar{n}_{ij}^-$.

\begin{restatable}[Ergodic limits of $n_{ij}^+$ and $n_{ij}^-$]{thm}{ThmErgolim} \label{Thm:Ergolim}
	Let $\mathcal{G}$ be a weighted graph with connected components $\{\mathcal{G}_t\}_{t = 1} ^ m$, where for each $t$, $\mathcal{G}_t$ has $n_t$ nodes and a nonnegative weighted adjacency matrix $A_t$. 
	Let the VEC algorithm be executed on  $\mathcal{G}$ with parameters $w$ and $k$ and transition matrix $W_t \defeq{}D_t^{-1}A_t$ in component $t$, where $D_t$ is a diagonal matrix with $i$th diagonal element $D_{t,ii} = \sum_j A_{t,ij} =: d_i$, i.e., the degree of node $i$.
	Then the ergodic limits $\bar{n}_{ij}^+$'s and $\bar{n}_{ij}^-$'s in \Cref{Defi:Ergolim} exist and are given by
	\begin{align}
		\label{Eq:Thm_Ergolim+}
		\bar{n}_{ij}^+={}& 
		\begin{cases}
			\pi_i\sum_{v=1}^{w}(W_t^v)_{ij}, & \text{if } i, j\in \mathcal{G}_t;\\
			0 & \text{otherwise}.
		\end{cases}, \\
		\label{Eq:Thm_Ergolim-}
		\bar{n}_{ij}^-={}& kw\pi_i\pi_j,
	\end{align}
	where $\mathbf{\pi}$ is a stationary distribution of the random walk with $\pi_i = \frac{n_t}{\sum_t{n_t}}\frac{D_{t,ii}}{\sum_i{D_{t,ii}}}$ for each $t$ and all $i \in \mathcal{G}_t$.
\end{restatable}
The proof of \Cref{Thm:Ergolim} is based on convergence results for irreducible Markov chains and is presented in Appendix~\ref{app:ErgoLim}. The key ideas are as follows. For the positive pairs we expand the state-space of the Markov chain and show that it is irreducible. This implies that the long term average of distributions converges to the stationary distribution. For the negative pairs the major obstacle is to deal with the second-layer of randomness conditioned on the positive samples. We overcome this difficulty by applying McDiarmid's inequality conditionally to establish almost complete convergence.

\Cref{Thm:Ergolim} states that the ergodic limits $\bar{n}_{ij}^+$ and $\bar{n}_{ij}^-$ can be evaluated directly without having to actually launch any random walks. The additional randomness from the random walks and dependence on the algorithm parameters $r$ and $\ell$ are removed. As a result, the coefficients, in the form of ergodic limits, are deterministic functions of the graph adjacency matrix and two algorithm parameters $w$ and $k$. 
Replacing the coefficients in Eq.~\eqref{Eq:VEC} with their limiting values (scaled down by the factor $\nicefrac{1}{(rn\ell)}$) yields the following optimization problem that we name \ErgoVEC:
\begin{definition}[{\ErgoVEC} optimization problem] \label{Defi:ErgoVEC}
\begin{eqnarray}
\underset{\{\mathbf{u}_{i} \in \mathbb{R}^d, i \in \mathcal{V} \}}{\arg\min} 
\sum_{(i,j) \in \mathcal{V}^2 }\left[\bar{n}_{ij}^+ \; \sigma( +\mathbf{u}_{i}^{\top}\mathbf{u}_{j})
  + \bar{n}_{ij}^{-} \; \sigma( -\mathbf{u}_{i}^{\top}\mathbf{u}_{j})\right]
\label{Eq:ErgoVEC}
\end{eqnarray}
\end{definition}
A practical approach to compute the embedding vectors of {\ErgoVEC} can be described as follows. Given a graph and algorithm parameters $w$ and $k$, first calculate $\bar{n}_{ij}^+$'s and $\bar{n}_{ij}^-$'s using \Cref{Thm:Ergolim}. Then use them to solve the {\ErgoVEC} optimization problem in \Cref{Eq:ErgoVEC} via stochastic gradient descent to find embedding vectors $\mathbf{u}_i$'s. 
A neural-network implementation is described in \Cref{Subsec:implementation} and Appendix~\ref{appendix_exp_details}.

{\ErgoVEC} calculates the coefficients of the optimization objective in a more principled way compared to VEC and completely bypasses the random walk sampling process. The $r$ and $\ell$ algorithm parameters of VEC are not needed at all in {\ErgoVEC}. 
However, when the graph is dense or $w$ is large, evaluating $\bar{n}_{ij}^+$ from Eq.~\eqref{Eq:Thm_Ergolim+} can be computationally very expensive. In these cases, $n_{ij}^+$ computed from random walks could serve as an approximation. Thus VEC can may be viewed as a practical approximation to the more principled {\ErgoVEC}. 

{\noindent \bf Relationship to modularity maximization.} When the graph is connected and we set $w=1$, Eq.~\eqref{Eq:ErgoVEC} reduces to
\begin{eqnarray}
\label{Eq:relaxed-mod}
\underset{\{\mathbf{u}_{i}, i \in \mathcal{V} \}}{\arg\min} 
\sum_{(i,j) \in \mathcal{V}^2}\left[A_{ij} \; \sigma(+\mathbf{u}_{i}^{\top}\mathbf{u}_{j})
  +k\frac{d_id_j}{\sum_k d_k} \; \sigma(-\mathbf{u}_{i}^{\top}\mathbf{u}_{j})\right]\!,\!\!\!\!
\end{eqnarray}
where
$d_i$ denotes the degree of node $i$. If instead we set $\sigma(t)\defeq t$, and constrain the embedding vectors so that for all $i,j$, $\mathbf{u}_{i}^{\top}\mathbf{u}_{j} \in \{0, 1\}$, 
then the minimization becomes equivalent to the  modularity maximization problem \cite{newman2004finding} for two communities defined by
\begin{eqnarray}
\underset{\{ y_i \in \{0,1\}, i \in \mathcal{V} \}}{\arg\max} 
\sum_{(i,j) \in \mathcal{V}^2 }\left(A_{ij}-\frac{d_id_j}{2|\mathcal{E}|}\right)
1(y_i = y_j) \nonumber
\end{eqnarray}
where $y_i$ denotes the community assignment for node $i$, 
$1(\cdot)$ is the indicator function,
and $|\mathcal{E}|$ is the number of edges.
This is often relaxed to \eqref{Eq:relaxed-mod} and solved via spectral approaches followed by clustering~\cite{Newman2013-ex}.

\subsection{Walk-distance weighting and large  \texorpdfstring{$r$}{r} asymptotics} 
\label{Subsec:3.2}

{\noindent \bf Walk-distance-weighted count statistics:} 
In VEC, $n_{ij}^+$ is the count of all instances where node $i$ appears within $w$ steps of node $j$ in all the random walks. Instances where nodes $i$ and $j$ appear exactly $1$ step from each other and instances where they appear in exactly $w$ steps from each other, both contribute a count of $1$ to the value of $n_{ij}^+$. A nuanced alternative must account for the number of steps between appearances of nodes.  

As a general approach to construct such a statistic, we propose associating a walk-distance weight $\alpha_v$ to the counts of instances of node pairs that occur \textit{exactly} $v$ steps from each other. With this modification, the walk-distance-weighted positive-pair counts will become $n_{ij}^+ := \sum_{v=1}^\infty \alpha_v n_{ij}^+(v)$, where  $n_{ij}^+(v)$ is the count of instances where node $i$ appears \textit{exactly} $v$ steps from node $j$ in all the random walks.
The original count statistic for positive pairs can be recovered as a special case of our proposed general framework by choosing $\alpha_v = 1$ for all $v \leq w$ and $\alpha_v = 0$ for all $v > w$.
Choosing a nonnegative decreasing sequence of walk-distance weights $\alpha_v$ can be viewed as providing a ``soft cutoff'' for the bigram counts when compared to the ``hard cutoff'' of the original counts. 

To compute walk-distance weighted counts for negative-pairs, we propose the following modification to the original negative sampling process. For each positive pair of nodes that occur exactly $v$ steps apart, we append $k$ node pairs drawn in an IID manner exactly as in the original sampling process. However, these $k$ negative node pairs will now contribute the value $\alpha_v$ to the walk-distance weighted negative-pair counts as opposed to the value of $1$ previously.

{\noindent \bf Large $r$ asymptotics:}  The effect of increasing $\ell$ is similar to that of increasing $r$. In a random walk on a graph, the choice of the next node depends only on the current node. From this point of view, we may loosely visualize a long random walk as being formed by joining many shorter segments which are nearly independent random walks. In this sense, an infinitely long random walk is similar to an infinite sequence of short random walks with each starting node chosen from the stationary distribution of the Markov chain.
Thus, in addition to the large $\ell$ asymptotics characterized in \Cref{Thm:Ergolim}, we can also study other types of asymptotics such as $r\rightarrow\infty$ or, more generally,  $\ell$ and $r$ both going to infinity together in some manner. 

The counterpart of \Cref{Thm:Ergolim} for the proposed walk-distance-weighted counts is the following general result which is proved in Appendix~\ref{app:ErgoLim_alpha}.
\begin{restatable}[Limits of walk-distance weighted counts]{thm}{ThmWtdErgodic}
	\label{Thm:Ergolim_full}
Let $\mathcal{G}$ be a weighted connected graph with $n$ nodes and $W$ be the probability transition matrix of the natural random walk on $\mathcal{G}$ with stationary distribution $\bm{\pi}$. Let the VEC algorithm be executed with $\mathcal{G}$ as input, walk-distance weights $\{\alpha_v\}_{v=1}^{\infty}$, and negative sampling rate $k$. If $\{\alpha_v\}_{v=1}^{\infty}$ is absolutely convergent, i.e., $\sum_{v=1}^{\infty}\vert\alpha_v\vert < \infty$, the following limits of $\bar{n}_{ij}^+$'s and $\bar{n}_{ij}^-$'s exist in the almost sure sense:
	\begin{enumerate}
	    \item 	When $r$ is fixed and $\ell\rightarrow\infty$ (ergodic limits):
    	    \begin{align} 
        		\label{Eq:Thm3_Ergolim_l+}
				\frac{1}{rn}\lim\limits_{\ell\rightarrow\infty}\frac{n_{ij}^+}{\ell}={}& 
						\pi_i\sum_{v=1}^{\infty}\alpha_v(W^v)_{ij},\\
        		\label{Eq:Thm3_Ergolim_l-}
        		\frac{1}{rn}\lim\limits_{\ell\rightarrow\infty}\frac{n_{ij}^-}{\ell}={}& k\pi_i\pi_j\sum_{v=1}^{\infty}\alpha_v.
        	\end{align}
        \item   When $\ell$ is fixed and $r\rightarrow\infty$:
            \begin{align} 
        		\label{Eq:Thm3_Ergolim_r+}
        		\frac{1}{\ell n}\lim\limits_{r\rightarrow\infty}\frac{n_{ij}^+}{r}
        		={}&\frac{1}{\ell n}\sum_{m=1}^n \sum_{v=1}^{\infty} \alpha_v(W^v)_{ij}\sum_{s=1}^{\ell-v}(W)^{s-1}_{mi},\\
        		\label{Eq:Thm3_Ergolim_r-}
        		\frac{1}{\ell n}\lim\limits_{r\rightarrow\infty}\frac{n_{ij}^-}{r}={}& \frac{k\pi^{(\ell)}_j}{\ell n}\sum_{m=1}^n \sum_{v=1}^{\infty} \alpha_v\sum_{s=1}^{\ell-v}(W)^{s-1}_{mi} 
        	\end{align}
        	where $\pi^{(\ell)}_j = \frac{1}{\ell}\sum_{u=1}^{\ell} \frac{1}{n}\bm{1}_n^\top W^{u-1}\bm{e}_j$.\footnote{We follow the convention that when the upper limit of a summation is smaller than its lower limit, the sum is $0$.}
        \item   Double limits:
            \begin{align} 
        		\label{Eq:Thm3_Ergolim_lr+}
        		\frac{1}{n}\lim\limits_{r\rightarrow\infty}\lim\limits_{\ell\rightarrow\infty}\frac{n_{ij}^+}{r\ell}={}\frac{1}{n}\lim\limits_{\ell\rightarrow\infty}\lim\limits_{r\rightarrow\infty}\frac{n_{ij}^+}{r\ell}={}& \pi_i\sum_{v=1}^{\infty}\alpha_v(W^v)_{ij},\\
        		\label{Eq:Thm3_Ergolim_lr-}
				\frac{1}{n}\lim\limits_{r\rightarrow\infty}\lim\limits_{\ell\rightarrow\infty}\frac{n_{ij}^-}{r\ell}={}\frac{1}{n}\lim\limits_{\ell\rightarrow\infty}\lim\limits_{r\rightarrow\infty}\frac{n_{ij}^-}{r\ell}={}& k\pi_i\pi_j\sum_{v=1}^{\infty}\alpha_v.
        	\end{align}
	\end{enumerate}
\end{restatable}
\Cref{Thm:Ergolim_full} is stated for a connected graph, but it holds for each connected component of a disconnected graph.
The main changes are that the expressions $\pi_i\sum_{v=1}^{w} (W^v)_{ij}$ and $k\pi_i \pi_j w$ in \Cref{Lemma:Ergolim} change to $\pi_i\sum_{v=1}^{\infty} \alpha_v (W^v)_{ij}$ and $k\pi_i \pi_j (\sum_{v=1}^{\infty} \alpha_v$) respectively
and the large-$r$ asymptotic limits of the normalized count statistics are also characterized.
To establish these results, we need to assume that the walk-distance weight series is absolutely convergent.
Walk-distance weighting makes it possible to realize different nonlinear functions of the transition matrix as the ergodic limit of count statistics, not just polynomial functions. For instance, if we choose $\alpha_v = \nicefrac{1}{v!}$ for all $v$, then  for all $i,j$, $\bar{n}_{ij}^{+} = \pi_i (\exp\{W\})_{ij}$, where $\exp\{W\}$ denotes matrix exponential, and $\bar{n}_{ij}^{-} = k \pi_i \pi_j e$.
We note that the large-$\ell$ asymptotic limits are independent of $r$ but the large-$r$ asymptotic limits depend on $\ell$. Yet, the double limits where $r$ is sent to infinity first before $\ell$ equal the corresponding large-$\ell$ limits.

\subsection{Reparameterized relaxations and their properties} \label{Subsec:PMI}

In this subsection, we study matrix re-parameterizations and convex relaxations of the VEC and {\ErgoVEC} optimization problems. We follow \cite{levy2014neural} and begin by defining the $n\times n$ matrix $X$ to be the Gram matrix of the node embedding vectors, i.e., for all $i,j$, $X_{ij} := \mathbf{u}_i^{\top}\mathbf{u}_j$. Let $N^+$ and $N^-$ denote the $n \times n$ matrices of the positive-pair and negative-pair counts respectively, i.e., for all $i,j$,  $[N^+]_{ij} = n_{ij}^+$ and $[N^-]_{ij} := n_{ij}^-$. If
\begin{align*} 
f(N^+, & N^-, X) \defeq{} \!\!\!\!
\sum_{(i,j) \in \mathcal{V}^2 }n_{ij}^{+} \sigma(+X_{ij}) + 
\sum_{(i,j) \in \mathcal{V}^2 }n_{ij}^- \sigma(-X_{ij}),
\label{Eq:GramVEC}
\end{align*}
then the VEC optimization problem (Eq.~\eqref{Eq:VEC}) reduces to the following equivalent optimization problem in the matrix variable $X$ named {\GramVEC}:
\begin{definition}[{\GramVEC} optimization objective] \label{Defi:GramVEC}
\begin{equation} 
\argmin_{X \in \mathbb{S}^n_{+},\  \mathrm{rank}(X)\leq d } 
f(N^+, N^-, X).
\label{Eq:GramVEC}
\end{equation}
\end{definition}
In Eq.~\eqref{Eq:GramVEC}, the constraint $X \in \mathbb{S}^n_{+}$ arises from the fact that the Gram matrix of embedding vectors is real, symmetric, and positive semi-definite.  The rank constraint comes from the fact that $\mathbf{u}_i\in\mathbb{R}^d$. 
The equivalence of the VEC (Eq.~\eqref{Eq:VEC}) and {\GramVEC} (Eq.~\eqref{Eq:GramVEC}) optimization problems can be seen as follows. For any set of feasible $\mathbf{u}_i$'s in \eqref{Eq:VEC}, setting $X_{ij} = \mathbf{u}_i^{\top}\mathbf{u}_j$ for all $i,j$, makes $X$ a rank-$d$ matrix in $\mathbb{S}^{n}_{+}$ and yields the same cost as in \eqref{Eq:GramVEC}.
In the other direction, for any feasible choice of $X$ in \eqref{Eq:GramVEC}, 
let $X = V^{\top}_d \Sigma_d V_d$
denote its rank-$d$ reduced SVD with diagonal $\Sigma_d \in \mathbb{S}^{d}_{+}$ and define $U = [\mathbf{u}_1,\ldots,\mathbf{u}_n] := \sqrt{\Sigma_d} V_d$. Then, $X = U^{\top} U$ and for all $i,j$, $X_{ij}=\mathbf{u}_i^{\top}\mathbf{u}_j$ and $\mathbf{u}_i \in \mathbb{R}^d$, and we  obtain the same cost in \eqref{Eq:VEC}.
The choices for the $\mathbf{u}_i$'s are not unique since 
$X = U^{\top} F^{\top} F U$ 
for any real orthonormal matrix $F$. The $\mathbf{u}_i$'s are unique only up to a real orthonormal transformation, just as in \eqref{Eq:VEC}.

The same re-parameterization can also be applied to {\ErgoVEC} (Eq.~\eqref{Eq:ErgoVEC}). Let $\bar{N}^+$ and $\bar{N}^-$ be $n\times n$ matrices such that for all $i,j$, $[\bar{N}^+]_{ij} := \bar{n}_{ij}^+$ and $[\bar{N}^-]_{ij} := \bar{n}_{ij}^-$. Then the {\GramErgoVEC} optimization problem is defined as follows.
\begin{definition}[{\GramErgoVEC} optimization problem]
	\label{Defi:GramErgoVEC}\hfill
	
\begin{equation} 
\argmin_{X \in \mathbb{S}^n_{+},\ \mathrm{rank}(X)\leq d} 
	f(\bar{N}^+, \bar{N}^-, X).
\label{Eq:GramErgoVEC}
\end{equation}
\end{definition}

Although Eq.~\eqref{Eq:GramVEC} and Eq.~\eqref{Eq:GramErgoVEC} are equivalent to Eq.~\eqref{Eq:VEC} and Eq.~\eqref{Eq:ErgoVEC}, respectively, they are more convenient to work with and analyze. The matrix re-parameterization transfers the non-convexity from the objective function to the constraint set which makes it possible to relax or convexify the problem as we do next. 

Relaxing all constraints on $X$ in {\GramErgoVEC} leads to the following optimization problem named {\GramErgoPMI} (relaxing constraints in {\GramVEC} similarly will yield a corresponding optimization problem {\GramPMI}):
\begin{definition}[{\GramErgoPMI} optimization problem] \label{Defi:GramErgoPMI}\hfill	
\begin{equation} 
\argmin_{X \in \mathbb{R}^{n\times n}}f(\bar{N}^+, \bar{N}^-, X).
\label{Eq:PMI}
\end{equation}
\end{definition}
In general, {\GramErgoPMI} is not equivalent to {\GramErgoVEC}. The relaxation enlarges the feasible set and the optimal solution may not satisfy the constraints in Eq.\eqref{Eq:GramErgoVEC}. Nonetheless,  Eq.~\eqref{Eq:PMI} admits a closed-form solution:
\begin{prop}
\label{Prop:PMI_solution}
Let $X^*$ be the solution to Eq.~\eqref{Eq:PMI}. Then, $X^*$ is unique, symmetric, and given by 	%
\begin{equation} \label{Eq:PMI_solution}
    X_{ij}^* = X_{ji}^* =
	  \begin{cases}
        \ln\left(\frac{\bar{n}_{ij}^+}{\bar{n}_{ij}^-}\right) & \text{if } \bar{n}_{ij}^+\neq 0;\\
				-\infty & \text{if } \bar{n}_{ij}^+= 0.
      \end{cases}	
\end{equation}
Let $p_{\ell}(i, j)$ denote the probability that a randomly sampled pair from the positive set $\mathcal{D}_{+}$ equals $(i,j)$ and let $p_{\ell 1}(i)$ and $p_{\ell 2}(j)$ denote, respectively, the first- and second-component marginal probabilities of $i$ and $j$ that are consistent with the joint pmf $p_{\ell}(i,j)$.\footnote{That is, $p_{\ell 1}(i) := \sum_{j\in\mathcal{V}} p_{\ell}(i,j)$ and $p_{\ell 2}(j) := \sum_{i\in\mathcal{V}} p_{\ell}(i,j)$. Note that $p_{\ell}(i,j)$ may not be symmetric when $\ell$ is finite.}
Let $\text{PMI}_{\ell}(i,j) := \ln\left(\frac{p_{\ell}(i,j)}{p_{\ell 1}(i)p_{\ell 2}(j)}\right)$ denote the {\textit {\textbf{Pointwise Mutual Information (PMI)}}} of $(i,j)$ \cite{church1990word}. 
Then for all $i,j$,
	  \begin{equation}
			X_{ij}^* = \lim_{\ell\rightarrow\infty}\text{PMI}_{\ell}(i,j) - \ln{k}.
	  \end{equation} 
\end{prop}
The proof of \Cref{Prop:PMI_solution} is presented in Appendix~\ref{app:PMI_solution}. 
Although $X^*$ is symmetric, there is no guarantee that $X^*$ will satisfy the positive semi-definiteness constraint of {\GramErgoVEC}, let alone the rank constraint.
Without positive semi-definiteness,
the square root of $X^*$'s nonzero singular values 
will be imaginary and there will not exist any real-valued embedding vectors, even in $\mathbb{R}^n$, whose Gram matrix equals $X^*$. 
A practical solution then is to compute the $\ell_2$ projection of $X^*$ into the rank-$d$ real positive semi-definite cone of $n \times n$ matrices  and factorize the projected matrix $\text{Proj}(X^*,d)$ to get embeddings. Still, there is no guarantee that the projected matrix $\text{Proj}(X^*,d)$ or the embeddings $\mathbf{u}_i$'s obtained from it will be a solution to {\GramErgoVEC}. We compare the {\GramErgoVEC} and {\GramErgoPMI} embedding vectors experimentally in \Cref{Sec:experiments_sbm}.

An alternative approach to deal with the non-convexity of the {\GramErgoVEC} is to replace the non-convex rank constraint with a nuclear norm constraint which is convex. The nuclear norm $\|X\|_{*}$ of a matrix $X$ is defined as the sum of its singular values. Its relationship with the rank of a matrix has been extensively studied in the literature. For example, nuclear norm level sets have been shown to be the convex-envelope of rank level sets\cite{fazel2004rank}. 
A bounded nuclear norm constraint has been used as a proxy for a bounded rank constraint in a number of problems such as low rank matrix completion \cite{dong2018low}, tensor robust PCA \cite{lu2019tensor}, and compressed sensing \cite{recht2010guaranteed}.
For some problems there exists an exact equivalence between these constraints but the conditions under which this occurs varies from problem to problem. Relaxing the rank constraint of {\GramErgoVEC} via a bound on the nuclear norm leads to the following optimization problem that we term {\NucGramErgoVEC} (we can similarly define {\NucGramVEC}): 
\begin{definition}[{\NucGramErgoVEC} optimization problem] \label{Defi:NucGramErgoVEC}\hfill
    \begin{equation} \label{Eq:NucGramErgoVEC}
        \argmin_{X \in \mathbb{S}^n_{+},\ \|X\|_{*} \leq \nu_n} f(\bar{N}^+, \bar{N}^-, X)
    \end{equation}
\end{definition}

A larger $\nu_n$ implies a looser nuclear norm constraint. When $\nu_n$ goes to $\infty$, the solution to {\NucGramErgoVEC} will approach the $\ell_2$ projection of the {\GramErgoPMI} solution onto the real positive semi-definite cone of $n \times n$ matrices. However, when $\nu_n$ goes to $0$, the solution to {\NucGramErgoVEC} will reduce to the all zeros matrix which has rank $0$. In general, we should expect smaller values of $\nu_n$ to yield solutions with approximately lower rank. This is corroborated by our experiments in \Cref{Subsec:exp_geometry}. Thus the rank of the solution matrix, and consequently the dimension of the embedding vectors, can be controlled by the value of $\nu_n$. 
Note that we allow the nuclear norm threshold $\nu_n$ to depend on $n$. In \Cref{Subsec:expected_obj} we show that the nuclear norms of the solutions to {\GramErgoPMI} and {\GramErgoVEC} for idealized graphs that have community structure scale with $n$ as $\nu_n = \Theta(n)$.

There are numerous algorithms available to numerically compute a solution to the {\NucGramErgoVEC} optimization problem. We modified and implemented Hazan's algorithm~\cite{hazan2008sparse} to generate all our experimental results. 
    
\subsection{Embeddings of expected SBM graphs} \label{Subsec:expected_obj}

As an important step toward analyzing concentration properties of embeddings, in this section we study the embedding solutions of different optimization problems focusing on the expected graph of a two-community SBM. Such graphs have an ideal community structure: all edges between nodes belonging to any specified pair of communities have the same edge weight. Our main result is summarized in the following theorem:
\begin{restatable}[Embeddings of an expected SBM graph]{thm}{ThmExpectedSolution} \label{Thm:expected_solution}
Let $\mathcal{G}$ be an SBM graph with $n=2m$, $m\geq2$, nodes and two balanced communities. For all $i\in \mathcal{V}$, let $y_i \in \{0,1\}$ denote the community label of node $i$. Let $a$ and $b$ be the edge forming probabilities for within- and cross-community edges, respectively, with $a > \frac{m}{m-1}b$. Let $\mathbf{E}[\mathcal{G}]$ denote the expected graph and $\bar{n}_{ij}^+$'s and $\bar{n}_{ij}^-$'s the ergodic limits for $\mathbf{E}[\mathcal{G}]$ as in \Cref{Defi:Ergolim}, with $k \geq 1$ and $w \geq 1$. Let 
\begin{align}
X^*(\mathcal{H}) := \argmin_{X \in \mathcal{H}} f(\bar{N}^+, \bar{N}^-, X),
\label{Eq:Thm2_optimization}
\end{align}
where $\mathcal{H} \subset \mathbb{R}^{n\times n}$. Let $\mathcal{E}_0 := \left\{(i,j) \in \mathcal{E}: y_i = y_j \right\}$ and  $\mathcal{E}_1 := \mathcal{E} \backslash \mathcal{E}_0$.
Then:
\begin{enumerate}
	\item \label{Thm2:Coeffs}
	Structure of ergodic limits: The values of $\bar{n}_{ij}^+$ and $\bar{n}_{ij}^-$ depend only on the community  membership of $(i,j)$, i.e.,
		\begin{align*}
 			\bar{n}_{ij}^+ &= \begin{cases}
 								\alpha_1, &\text{if $(i,j)\in\mathcal{E}_0$} \\
 								\alpha_2, &\text{if $(i,j)\in\mathcal{E}_1$} \\
 								\alpha_3, &\text{if $i=j$}
 						 	  \end{cases} \\
 			\bar{n}_{ij}^- &= \beta,\quad \forall (i,j),
 		\end{align*}
 	where $\beta = \frac{kw}{n^2}$ and $\alpha_i=C_i/n^2+o(1/n^2)$ for $i=1,2,3$, where $C_i$'s are functions of only $a$, $b$ and $w$. 
	\item \label{Thm2:Unconstrained} {\GramErgoPMI} solution: Let $\mathcal{H} = \mathbb{R}^{n\times n}$. Then $X^*(\mathcal{H})$ has full rank with the same structure as $\bar{n}_{ij}^+$ with
		\begin{align*}
 			X^*(\mathbb{R}^{n\times n}) &= 
 			            \begin{cases}
 							\ln\left(\frac{\alpha_1}{\beta}\right), &\text{if $(i,j)\in\mathcal{E}_0$} \\
 							\ln\left(\frac{\alpha_2}{\beta}\right), &\text{if $(i,j)\in\mathcal{E}_1$} \\
 							\ln\left(\frac{\alpha_3}{\beta}\right), &\text{if $i=j$}
 						\end{cases}
 		\end{align*}
 	\item \label{Thm2:PSD_constrained} {\NucGramErgoVEC} solution for $\nu_n = \infty$: Let $\mathcal{H} = \mathbb{S}^n_{+}$ and $\nu_1 := \ln\left(\frac{\bar{\alpha}_{13}+\beta}{\alpha_2+\beta}\right)$ where $\bar{\alpha}_{13} \defeq \frac{m-1}{m}\alpha_1 + \frac{1}{m}\alpha_3$.
 	Then $X^*(\mathcal{H})$ has rank $1$.  Moreover, if $\nu_n = \nu_1 n$, then 
		\begin{align*}
 			X^*(\mathbb{S}^n_{+}) &= 
 			    \begin{cases}
 				    \phantom{-}\nu_1, &\text{if $(i,j)\in\mathcal{E}_0$ or $i=j$} \\
					-\nu_1, &\text{if $(i,j)\in\mathcal{E}_1$}.
 				\end{cases}
 		\end{align*}
 	\item The nuclear norms of $X^*(\mathbb{R}^{n\times n}) $ and $X^*(\mathbb{S}^n_{+})$ scale with $n$ as $\Theta(n)$. 
\end{enumerate}
\end{restatable}
The full proof of \Cref{Thm:expected_solution} is presented in Appendix \ref{app:proof_expected_solution}, but the key ideas are as follows. Part $1)$ holds as a result of \Cref{Thm:Ergolim} which characterizes the ergodic limits of normalized bigram counts in terms of the random walk transition matrix. Since the transition matrix of the expected graph has block-wise constant entries, this property is carried forward to the normalized ergodic counts. Part $2)$ follows directly from part $1)$ and \Cref{Prop:PMI_solution}. Part $3)$ is the major piece of this theorem and is proved via an intricate analysis of the structure of the solution. 

The $\Theta(n)$ scaling of nuclear norms of $X^*(\mathbb{R}^{n\times n})$ and $X^*(\mathbb{S}^n_{+})$ arises from the block structure and the fact that all entries are of constant order.
\Cref{Thm:expected_solution} also shows that $\bar{N}^{+}$ and $X^*(\mathbb{R}^{n\times n})$, the solution  to {\GramErgoPMI}, have the structure of a rank $2$ matrix minus a scalar multiple of the identity matrix making them full rank. In contrast, $X^*(\mathbb{S}^n_{+})$, the solution to {\NucGramErgoVEC}, has rank $1$ due to the positive semi-definite constraint.

Part $3)$ of \Cref{Thm:expected_solution} may seem surprising on first glance because we are getting a rank $1$ solution without any rank or nuclear norm constraints. 
The surprise dissipates when we note that the solution is for an expected graph which has an ideal community structure. Such a result would not hold true for a random graph realization.
Nonetheless, we can directly obtain the {\GramErgoVEC} and {\NucGramErgoVEC} solutions for the expected graph from part $3)$ of \Cref{Thm:expected_solution}:
\begin{restatable}{corollary}{Thm2_corollary} \label{Cor1:Thm_sol_expectation_under_rank_nuc}
  Under the assumptions of \Cref{Thm:expected_solution},
  \begin{enumerate}
    \item {\GramErgoVEC} solution for any positive rank: Let $\mathcal{H} = \{X \in \mathbb{S}^n_{+}: \text{rank}(X) \leq d \}$ with $d \geq 1$. Then $X^*(\mathcal{H})$ has rank $1$ and equals the {\GramErgoVEC} solution for $\nu_n = \infty$ characterized in part 3) of \Cref{Thm:expected_solution}.
	\item {\NucGramErgoVEC} solution for all $\nu_n \geq \nu_1 n$: 
	Let $\mathcal{H} = \{X \in \mathbb{S}^n_{+}: \|X\|_* \leq \nu_n \}$. If $\nu_n \geq \nu_1 n$, then $X^*(\mathcal{H})$ has rank $1$ and equals the {\GramErgoVEC} solution for $\nu_n = \infty$ characterized in part 3) of \Cref{Thm:expected_solution}.
  \end{enumerate}
\end{restatable}

When given the expected SBM graph as input, {\GramErgoVEC} and {\NucGramErgoVEC} will return a Gram matrix of rank $1$ or $2$ which when factorized will provide two distinct embedding vectors, each representing one community in the original graph. In short, the algorithms will give embeddings that are perfectly separated across communities and perfectly concentrated within communities.

In part $2)$ of \Cref{Cor1:Thm_sol_expectation_under_rank_nuc}, with $\nu_n \geq \nu_1 n$, the nuclear norm constraint becomes inactive. Suppose that  $\nu_n = \nu_0 n$. If $\nu_0 \leq \nu_1$, we conjecture that the solution will scale proportionally with $\nu_0$:
\begin{conj}
\label{conj:Thm_sol_expectation_small_nu}
Under the assumptions of \Cref{Thm:expected_solution}, let $\mathcal{H} = \{X \in \mathbb{S}^n_{+}: \|X\|_*\leq \nu_0 n \}$. If $\nu_0 < \nu_1$, then
		\begin{align*}
 			X^*(\mathcal{H}) &= \begin{cases}
 							\phantom{-}\nu_0, &\text{if $(i,j)\in\mathcal{E}_0$ or $i=j$} \\
							-\nu_0, &\text{if $(i,j)\in\mathcal{E}_1$}.
 						 								\end{cases} \\
 		\end{align*}     
\end{conj}
Conjecture~\ref{conj:Thm_sol_expectation_small_nu} makes an assertion about the solution to the {\NucGramErgoVEC} optimization problem when the nuclear norm constraint is active. For a suitable nuclear norm constraint, we conjecture that the solution will be a scaling of the solution in part 3) of \Cref{Thm:expected_solution}. If this conjecture holds, then, we can conclude that the solution to {\NucGramErgoVEC} is always rank $1$ with perfect separation of the communities regardless of the sparsity level of the graph. This would provide a solid starting point for analyzing the concentration properties of solutions as $n$ increases to $\infty$.

%

\section{Experimental setup} \label{Sec:exp_setups}

In the following sections, we compare the node embeddings from different algorithms qualitatively and quantitatively through extensive experiments. This section details our experimental setup, including the generation of random graphs, details of implementation and parameter choices for algorithms, and evaluation metrics for embedding vectors. Section~\ref{Subsec:exp_geometry} explores the geometric properties of embedding vectors for a fixed graph size. Specifically, we study how the nuclear norm linear scaling factor $\nu_0$ of the nuclear norm limit $\nu_n = \nu_0 n$ impacts the embedding geometry in {\NucGramErgoVEC}. In Section~\ref{Subsec:exp_asymptotic} we investigate how embedding vectors change as the graph size $n$ increases and whether they tend to concentrate.

\subsection{SBM graph generation}
For simplicity, we focus on assortative, equal-sized, planted-partition SBM graphs with 2 communities. We generate random graphs with $n = 100, 200, 500, 1000$ nodes. 
We consider two scaling regimes for the within-community edge forming probability $p$ and the cross-community edge forming probability $q$:
1) \textit{Linear regime:} Here, $p$ and $q$ are held constant, with values $p = 0.6$ and $q = 0.06$, for all graph sizes. As a result, the expected node degree scales linearly with $n$.
2) \textit{Logarithmic regime:} Here, $p$ and $q$ diminish with increasing graph size $n$ as a multiple of $(\ln{n})/n$, specifically as $p = 9 (\ln{n}) / n$ and $q = 2 (\ln{n}) / n$. The expected node degree then increases proportionally with $\ln{n}$. Our choices of scaling factors $\tilde{p} = 9$ and $\tilde{q} = 2$ in the logarithmic regime ensure that the information-theoretic threshold for exact community recovery for two communities, given by $\sqrt{\tilde{p}} - \sqrt{\tilde{q}} > \sqrt{2}$~ \cite{abbe2015}, is slightly surpassed.

To ensure graph connectivity and improve community detection performance, we follow the prescription  in \cite{joseph2016impact} and apply $\varepsilon$-smoothing to all generated graphs. 
For a given $\varepsilon\geq 0$, the $\varepsilon$-smoothing of $\mathcal{G}$, denoted by $\mathcal{G}^{\varepsilon}$, is the weighted complete graph with adjacency matrix $A^{\varepsilon}$, where for all $i,j$, $A^{\varepsilon}_{ij} := A_{ij} + \varepsilon$.
In addition to analytical convenience, graph-smoothing also improves the performance of spectral clustering in practice \cite{joseph2016impact}.
The ergodic limit of the coefficients under $\varepsilon$-smoothed graphs can be computed by 
\Cref{Thm:Ergolim}
with a modified probability transition matrix $W_\varepsilon$ that corresponds to the new graph.

The optimal choice of $\varepsilon$ for various performance measures such as signal-to-noise ratio, community detection accuracy, etc., varies across different random graph realizations and embedding algorithms.
Since our focus is on embedding algorithms, we apply $\varepsilon$-smoothing to all random graphs that we generate with the fixed choice $\varepsilon  = \nicefrac{1}{10n}$ instead of optimizing $\varepsilon$ for each algorithm and each performance measure.
This is a relatively small value of $\varepsilon$ as it changes the degrees from the original graph by at most $\nicefrac{1}{10}$ whereas the expected degrees in the linear and logarithmic regimes scale as $n$ and $\ln(n)$, respectively.

\subsection{Algorithm parameter choices and implementation}
\label{Subsec:implementation}
We implement and compare  {\ErgoVEC}, {\GramErgoPMI}, and {\NucGramErgoVEC}, that were proposed in Section~\ref{Sec:alt_formulation}, with VEC and Spectral Clustering (SC). SC serves as a classical benchmark due to its widespread usage.

{\bf \noindent Parameter choices}. For ease of visualization, we compute and plot 2-dimensional embedding vectors for all algorithms, i.e. $d=2$. 

For all algorithms other than SC, we set the window size to $w = 8$, and the negative sampling rate to $k = 5$.
Settings that are specific to SC, VEC, and {\NucGramErgoVEC} are as follows:  

\begin{enumerate}
	\item {\bf SC:} We use the first two eigenvectors of the symmetrically normalized Laplacian, i.e., $D^{-1/2} A D^{-1/2}$, \cite{von2007tutorial}.
	Since the unit-norm eigenvectors are in $\mathbb{R}^n$, their components, and therefore also the node embedding vectors, scale as $O(1/\sqrt{n})$. In order to simultaneously visualize and compare embedding vectors across different values of $n$, we scale them by $\sqrt{n}$.
	\item {\bf VEC:} We launch $r = 10$ walks starting from each node, each of length $\ell = 100$.
	\item {\bf NucGramErgoVEC:} In light of the linear scaling of the nuclear norm of the gram matrix of node embedding vectors for an expected SBM graph, (\textit{cf.}~\Cref{Thm:expected_solution} and \Cref{Cor1:Thm_sol_expectation_under_rank_nuc}), we set $\nu_n = \nu_0 n$ and sweep $\nu_0$ over the range $0.018$ through $0.216$, in steps of $0.018$, in order to illustrate changes in the geometric structure of embedding vectors (\textit{cf.} Fig.~\ref{fig:exp_geometry_scatter_nuc} and Fig.~\ref{fig:exp_metrics_vs_nn}).
\end{enumerate}

{\bf \noindent Implementation of VEC and {\ErgoVEC}.} Both VEC and {\ErgoVEC} have non-convex objectives for which there is no optimization procedure available which guarantees convergence to a global minimizer. A practical way forward is to use Stochastic Gradient Descent (SGD). We can consider two distinct approaches for implementing SGD in VEC or {\ErgoVEC} 1) Map them to an equivalent Word2Vec problem by identifying nodes as words and random walks as sentences and then obtain word embeddings using a Word2Vec package such as Gensim~\cite{rehurek_lrec}. 
2) Reformulate each optimization problem as the training of a neural network with an appropriate architecture and cost and then train the neural network using a neural network package such as Keras~\cite{chollet2015keras}. 
Since the Gensim package cannot handle non-integer coefficients that arise in {\ErgoVEC}, we use Keras to implement {\ErgoVEC} and VEC.
Details of our neural network implementation are presented in Appendix~\ref{appendix_exp_details}.
In \cref{Subsec:exp_asymptotic} we also compare the Gensim and Keras implementations of VEC.

Since both algorithms involve optimizing non-convex objectives, convergence is not guaranteed. We assess the convergence of the objective function value and the solution by evaluating the change in the objective function value and the embeddings after each epoch. To quantify the change in embeddings, we first perform a Procrustes alignment of the embedding solutions from successive epochs and then compute the Frobenius norm of the difference between the aligned sets of embeddings. 
We found that the change in the objective function value diminishes as the number of epochs increases, but the change in the corresponding embeddings retains a small fluctuation after diminishing initially
(\textit{cf.}~Appendix~\ref{appendix_exp_details}).
This suggests that although the objective function value converges, the embeddings may be oscillating around a local optimizer. 
Note that the Keras implementation, which implements SGD, is not designed to minimize changes in the solution (arguments), but rather in the objective function.

{\bf \noindent Implementation of NucGramErgoVEC}. We use Hazan's algorithm~\cite{hazan2008sparse} (suitably modified to handle inequality constraints) to solve the {\NucGramErgoVEC} optimization problem. The algorithm is iterative and is guaranteed to converge to the global minimum. We also empirically confirmed the convergence of both the objective function value and the solution matrix using the approach used for VEC and {\ErgoVEC} described above.

We note, however, that even though Hazan's algorithm is guaranteed to converge to a global minimum,  its convergence rate is slow. In order to improve convergence speed, we initialize with the {\GramErgoPMI} solution suitably scaled to fit the nuclear norm limit. We also terminate the algorithm after $1000$ iterations which is adequate for all our experiments.

\subsection{Visualization and performance evaluation}

\textbf{Representation and alignment of embeddings:} The absolute positions and orientations of embedding vectors may vary across algorithms, graph realizations, and graph sizes. Even for a given graph and algorithm the embedding vectors are not unique due to invariance of the objective function under orthogonal transformations.
However clustering and separation properties of embeddings only depend on the \textit{relative} positions and orientations of embedding vectors. Thus, in order to visualize and simultaneously compare different embeddings qualitatively and quantitatively, we first represent the embedding vectors using their SVD coordinates and then align them with Procrustes analysis.
Specifically, let $U$ be an $n \times d$ matrix whose $i$-th row $\bm{u}_i^\top$ is the embedding of node $i$. Let $U = \tilde{U}\Sigma\tilde{V}^\top$ be the SVD decomposition of $U$. Then, the SVD coordinates of the embedding vectors are given by $U\tilde{V}$. 
To align two sets of embedding points $U_1$ and $U_2$, we do Procrustes analysis, which finds the orthogonal matrix $P$ that minimizes $\Vert U_1 - U_2 P \Vert^2_{\text{F}}$. The aligned points are given by $U_1$ and $U_2P$.

{\bf\noindent Quantifying community separation}. 
We quantify the separation of nodes belonging to different communities using a signal-to-noise ratio (SNR) measured along the line joining the embedding centroids of the two communities.  Specifically, for embeddings of nodes in community $i$ ($i = 1,2$), let $\hat{\bm{\mu}}_i$ denote their empirical mean and $\hat{K}_i$ their empirical covariance matrix. Then we define SNR-1D as follows: 
\begin{align*}
	\text{SNR-1D} &\defeq \frac{\Vert \hat{\bm{\mu}}_1 - \hat{\bm{\mu}}_2\Vert_2^2}{\frac{1}{2} ( \hat{\eta}_1^2 + \hat{\eta}_2^2)}
\end{align*}
where $\hat{\eta}^2_i := (\hat{\bm{\mu}}_1-\hat{\bm{\mu}}_2)^{\top} \hat{K}_i(\hat{\bm{\mu}}_1 - \hat{\bm{\mu}}_2)$ is the empirical variance of the embeddings of nodes in community $i$ when projected onto the line joining the embedding centroids of the two communities.
%

%

\section{Node embedding geometry of SBM graphs}
\label{Sec:experiments_sbm}

In this section, we present and compare embeddings for SBM graphs produced by different algorithms and how they are positioned relative to the embeddings of the expected graph (indicated by black crosses in all our plots). Section~\ref{Subsec:exp_geometry} focuses on the comparing the geometry of embeddings across different algorithms and parameter choices whereas  Section~\ref{Subsec:exp_asymptotic} focuses on the large graph asymptotic behavior of embeddings. 

\subsection{Geometry of embeddings}
\label{Subsec:exp_geometry}

The geometry of node embedding vectors from SC, VEC, {\GramErgoPMI} and {\ErgoVEC} are shown in Fig.~\ref{fig:exp_geometry_scatter_other}. We plot the 2D 95\% confidence ellipse for each embedding cluster (red-colored elliptical curves) based on a maximum likelihood Gaussian fit to the data. 
For SC, {\ErgoVEC}, and {\GramErgoPMI}, the embeddings of the expected SBM graph are two distinct points (characterized in \Cref{Thm:expected_solution} and \Cref{Cor1:Thm_sol_expectation_under_rank_nuc}) which are marked as black crosses in Fig.~\ref{fig:exp_geometry_scatter_other}.
Since the VEC objective is based on empirical skip bigram counts from random walks, for a finite random walk length $\ell$, the embedding vectors of the expected graph will not collapse to two just distinct points, but will be distributed 
around the embedding vectors of {\ErgoVEC} which are indicated as black crossses in the VEC subplot of Fig.~\ref{fig:exp_geometry_scatter_other}.

\begin{figure}[ht]
\captionsetup[subfigure]{labelformat=empty}
\centering
    \begin{subfigure}{0.5\linewidth}
		\includegraphics[trim = 70pt 20pt 85pt 35pt, clip, width=\linewidth]{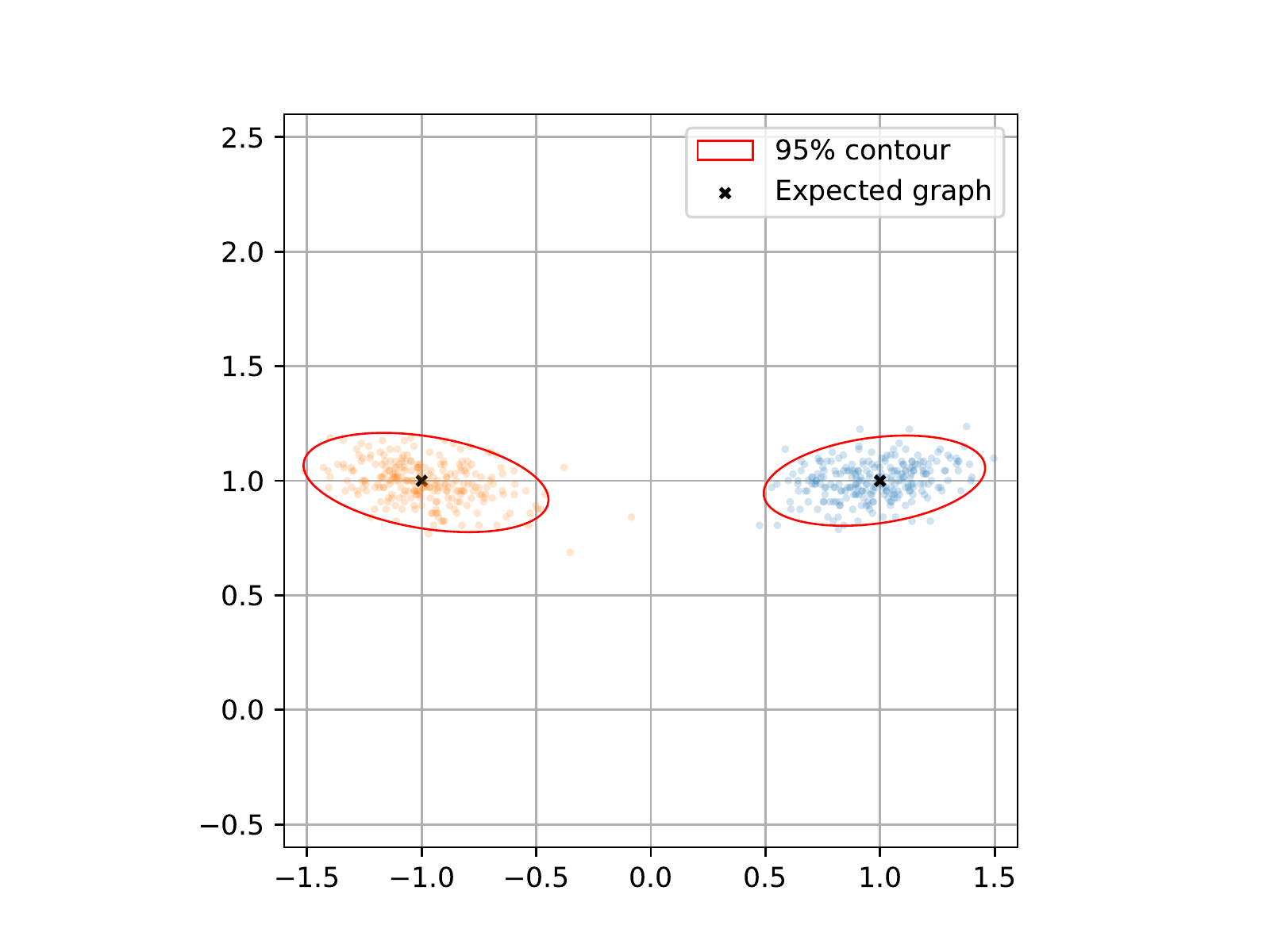}
        \caption{\small ~~~~~SC}
    \end{subfigure}%
    \begin{subfigure}{0.5\linewidth}
		\includegraphics[trim = 70pt 20pt 85pt 35pt, clip, width=\linewidth]{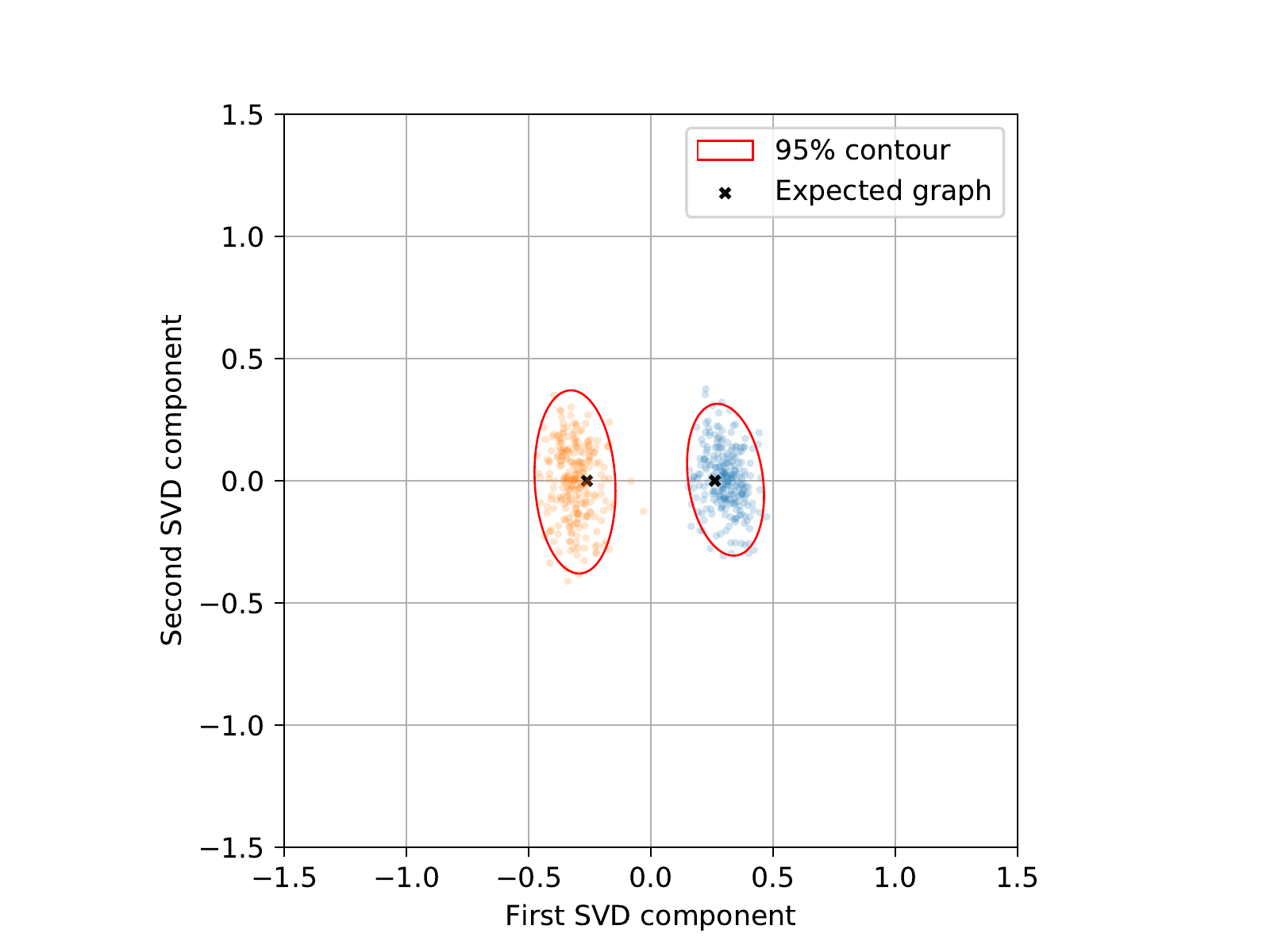}
        \caption{\small ~~~~~VEC}
    \end{subfigure}
    \begin{subfigure}{0.5\linewidth}
		\includegraphics[trim = 70pt 20pt 85pt 35pt, clip, width=\linewidth]{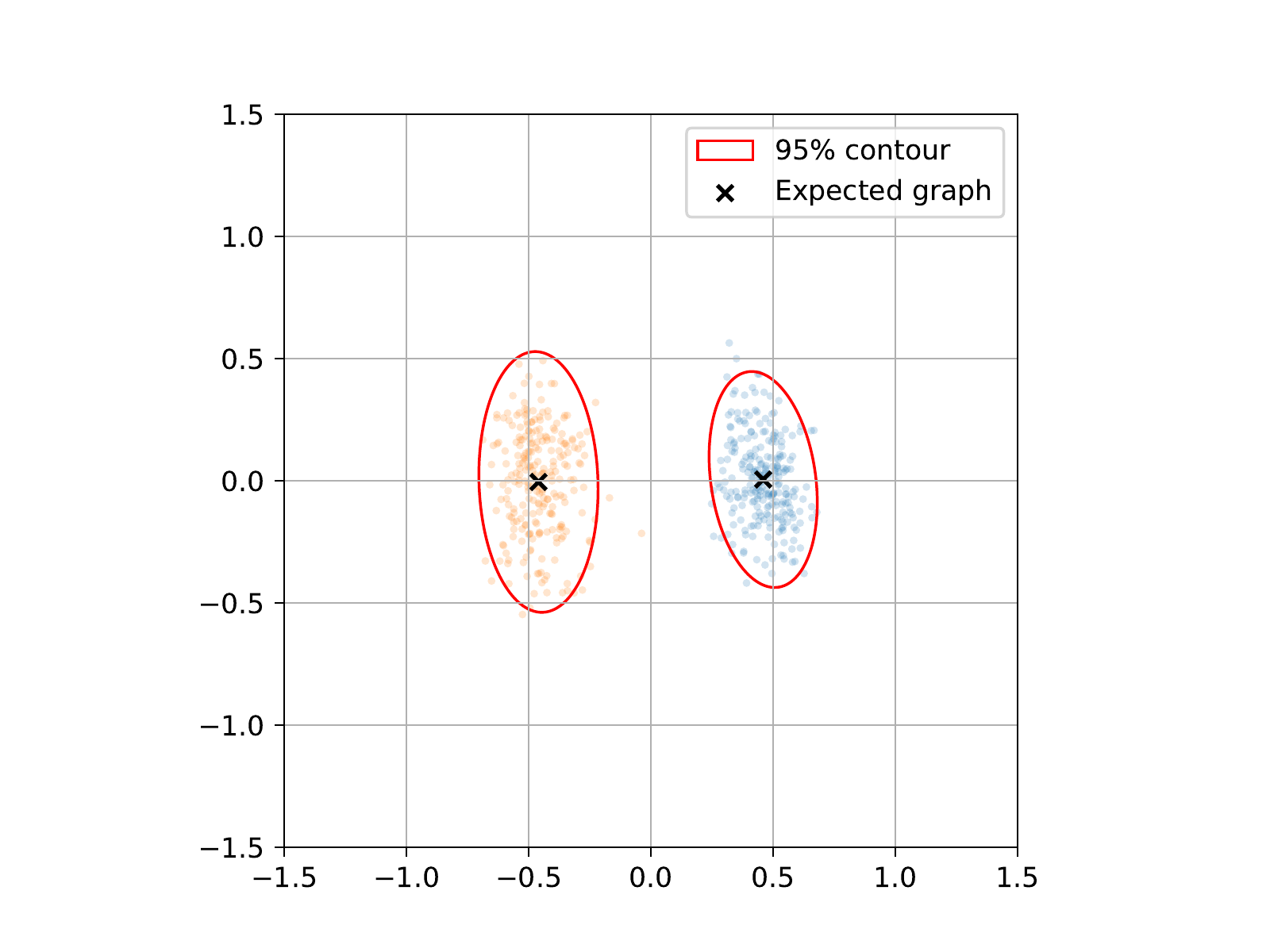}
        \caption{\small ~~~~~ErgoPMI}
    \end{subfigure}%
    \begin{subfigure}{0.5\linewidth}
		\includegraphics[trim = 70pt 20pt 85pt 35pt, clip, width=\linewidth]{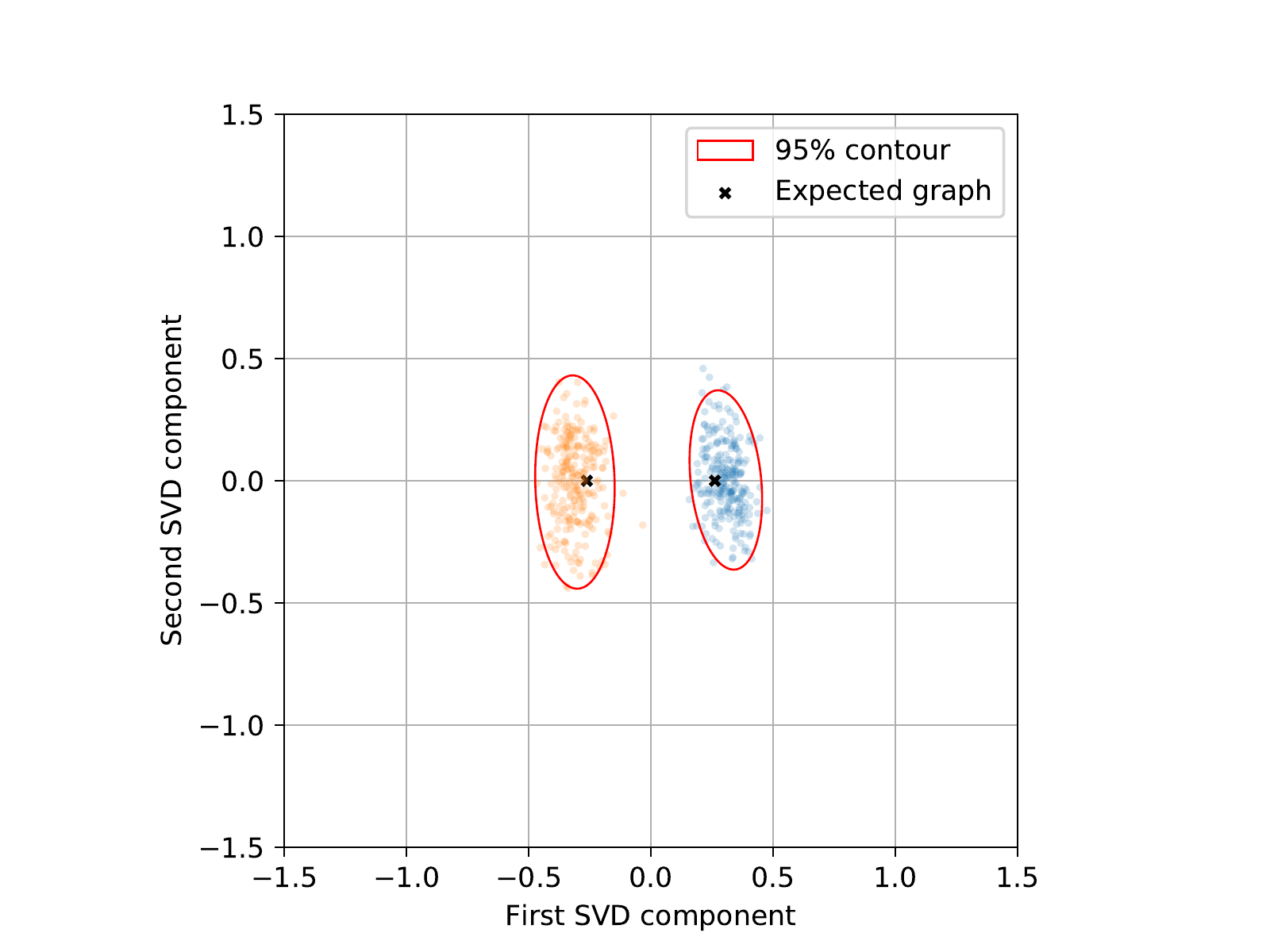}
        \caption{\small ~~~~~ErgoVEC}
    \end{subfigure}
\vglue -1ex
    \caption{\small 2D-visualization of embeddings for SC, VEC, ErgoPMI and ErgoVEC. All four algorithms receive the same graph input with $n=500$ nodes generated using within community edge probabilities $p = 9 \ln{n} / n$ and across community edge probabilities $q = 2 \ln{n} / n$.}
    \label{fig:exp_geometry_scatter_other}
\end{figure}

In Fig.~\ref{fig:exp_geometry_scatter_other}, we observe that in all four algorithms, the node embeddings in each cluster have an elliptical distribution around the cluster centroid and they can be perfectly separated linearly by the bisector of the line joining the two cluster centroids. However, the major axes of the SC embedding ellipses are nearly parallel to their inter-centroid line whereas the major axes of embedding ellipses in the other three algorithms are nearly perpendicular to their respective inter-centroid lines.

We also notice that the embedding ellipses of VEC and {\ErgoVEC} in Fig.~\ref{fig:exp_geometry_scatter_other} look very similar. This is to be expected as the {\ErgoVEC} objective is exactly the large $\ell$ ergodic limit of the VEC objective introduced in Section~\ref{Subsec:ErgoVEC}. To empirically confirm that the {\ErgoVEC} embeddings converge to the VEC embeddings in the large $\ell$ limit, in Fig.~\ref{fig:exp_walk_length_metric} we plot the distance between VEC and {\ErgoVEC} embeddings for increasing values of $\ell$ and different graph sizes.

\begin{figure}[!ht]
\centering
        \centering
		\includegraphics[trim = 25pt 20pt 45pt 40pt, clip, width=\linewidth]{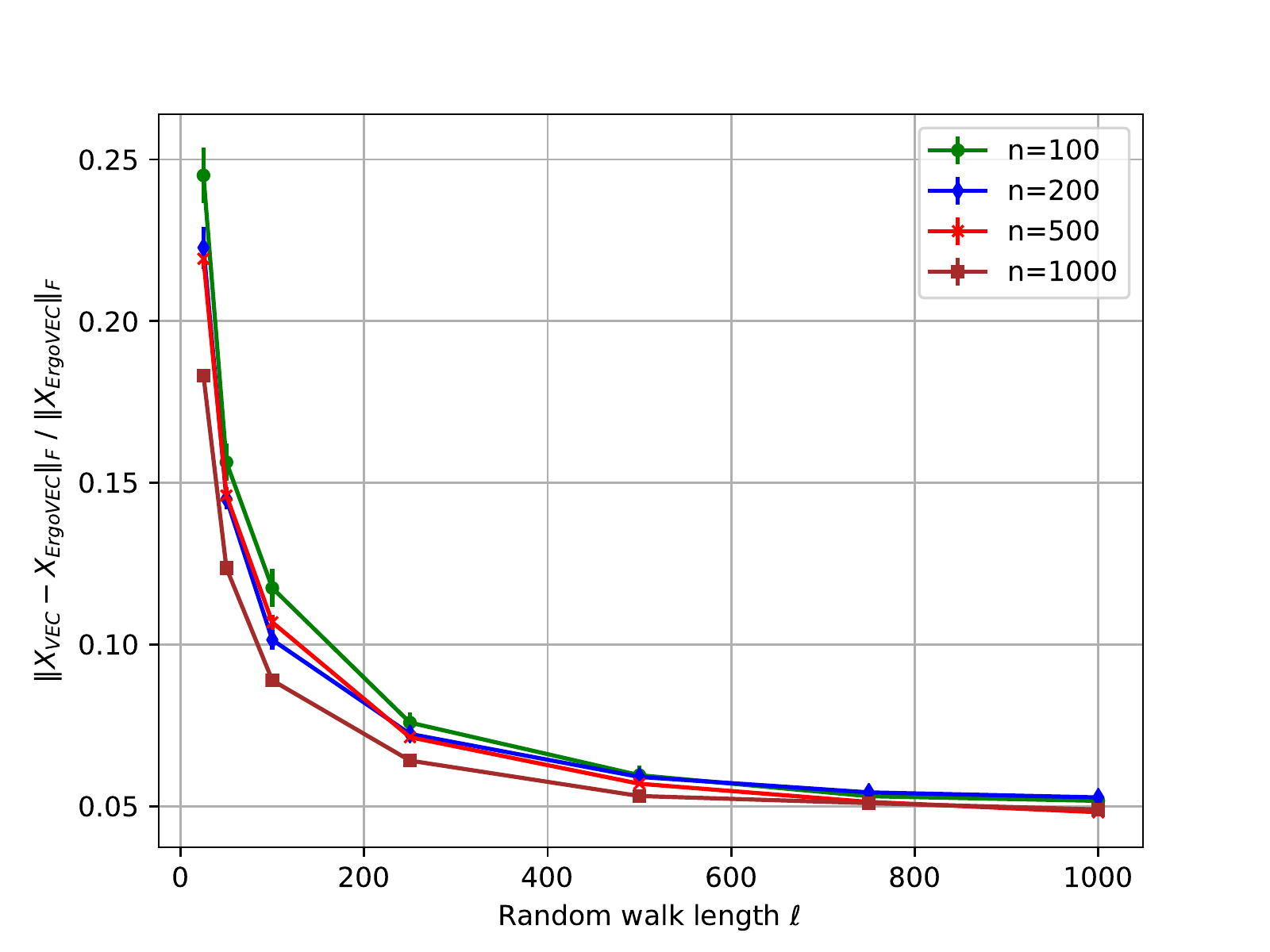}
\vglue -1ex
\caption{\small Convergence of VEC embeddings: Frobenius norm distance between the Gram matrices of VEC and {\ErgoVEC} versus random walk length $\ell$. 
}
\label{fig:exp_walk_length_metric}
\end{figure}

In order to measure the distance between embeddings up to any orthogonal transformation, we use the normalized Frobenius norm distance between the Gram matrices of the embeddings. 
For each $n$, graphs are generated using within-community edge probabilities $p = 9 \ln{n} / n$ and across-community edge probabilities $q = 2 \ln{n} / n$. The plot depicts the mean distance and associated error bar averaged over $5$ independent graph realizations.
Observe that for all graph sizes $n=100, 200, 500, 1000$, as the length of the random walk increases, the distance between VEC and {\ErgoVEC} Gram matrices shrinks. However, due to the non-convexity of the VEC and and lack of global convergence guarantees for SGD methods used to optimize VEC and {\ErgoVEC} objectives ({\it cf.} Section~\ref{Subsec:implementation}), the distance seems to be strictly bounded away from zero even at $\ell = 1000$.
However, the positive and negative $w$-skip bigram counts and the objective function of VEC do converge to their respective {\ErgoVEC} counterparts as $\ell$ increases to infinity.

We now discuss the embedding geometry of {\NucGramErgoVEC}. We separated this discussion from the previous four algorithms because although the embeddings of {\NucGramErgoVEC} are also elliptically distributed and separate well into two clusters, the specific shape depends on the nuclear norm linear scaling factor $\nu_0$ as shown in Fig.~\ref{fig:exp_geometry_scatter_nuc}. 
\begin{figure}[ht]
\captionsetup[subfigure]{labelformat=empty}
\centering
   	\begin{subfigure}{0.5\linewidth}
		\includegraphics[trim = 70pt 20pt 85pt 35pt, clip, width=\linewidth]{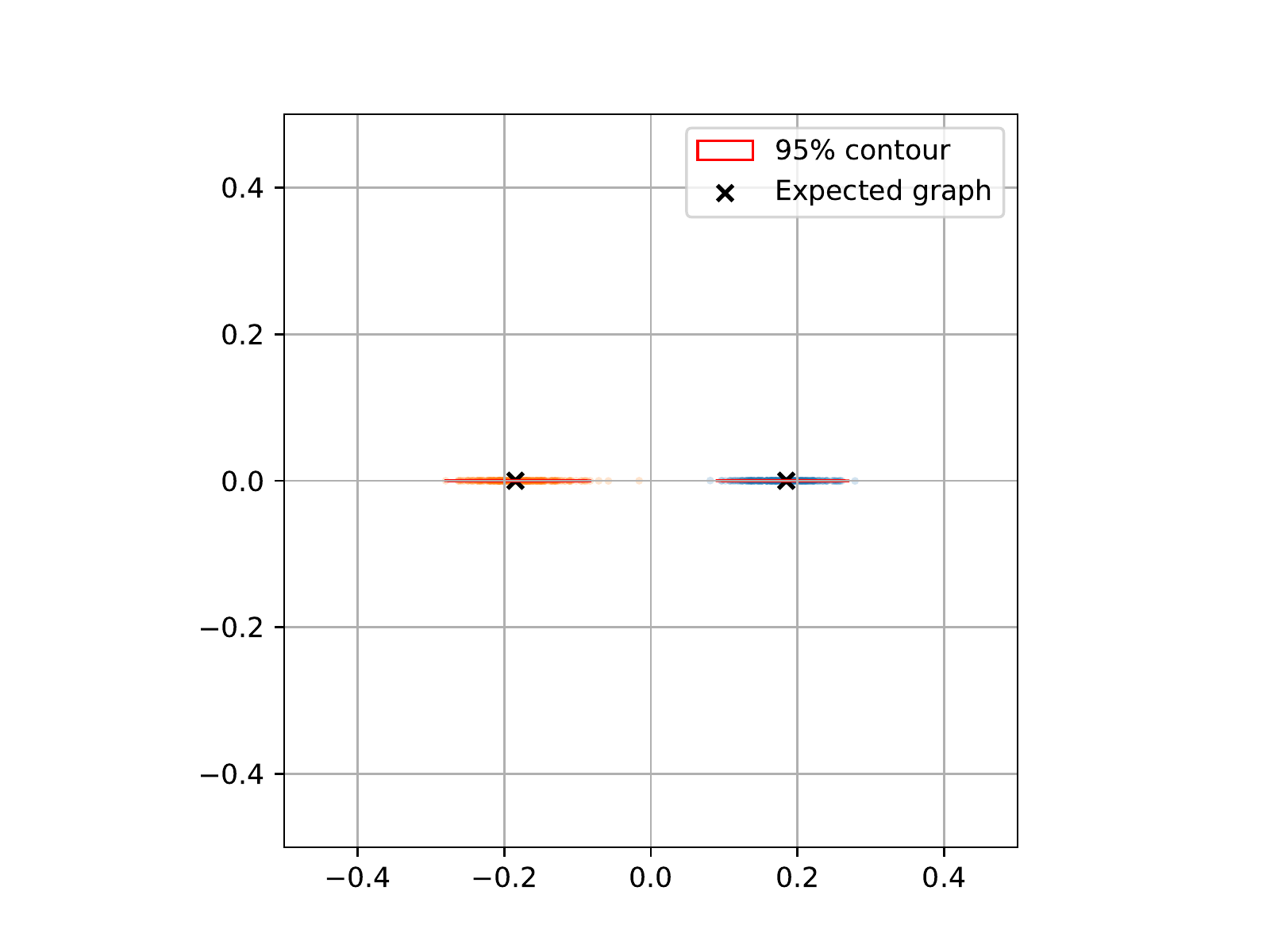}
        \caption{\small ~~~~~$\nu_0 = 0.036$}
    \end{subfigure}%
    \begin{subfigure}{0.5\linewidth}
		\includegraphics[trim = 70pt 20pt 85pt 35pt, clip, width=\linewidth]{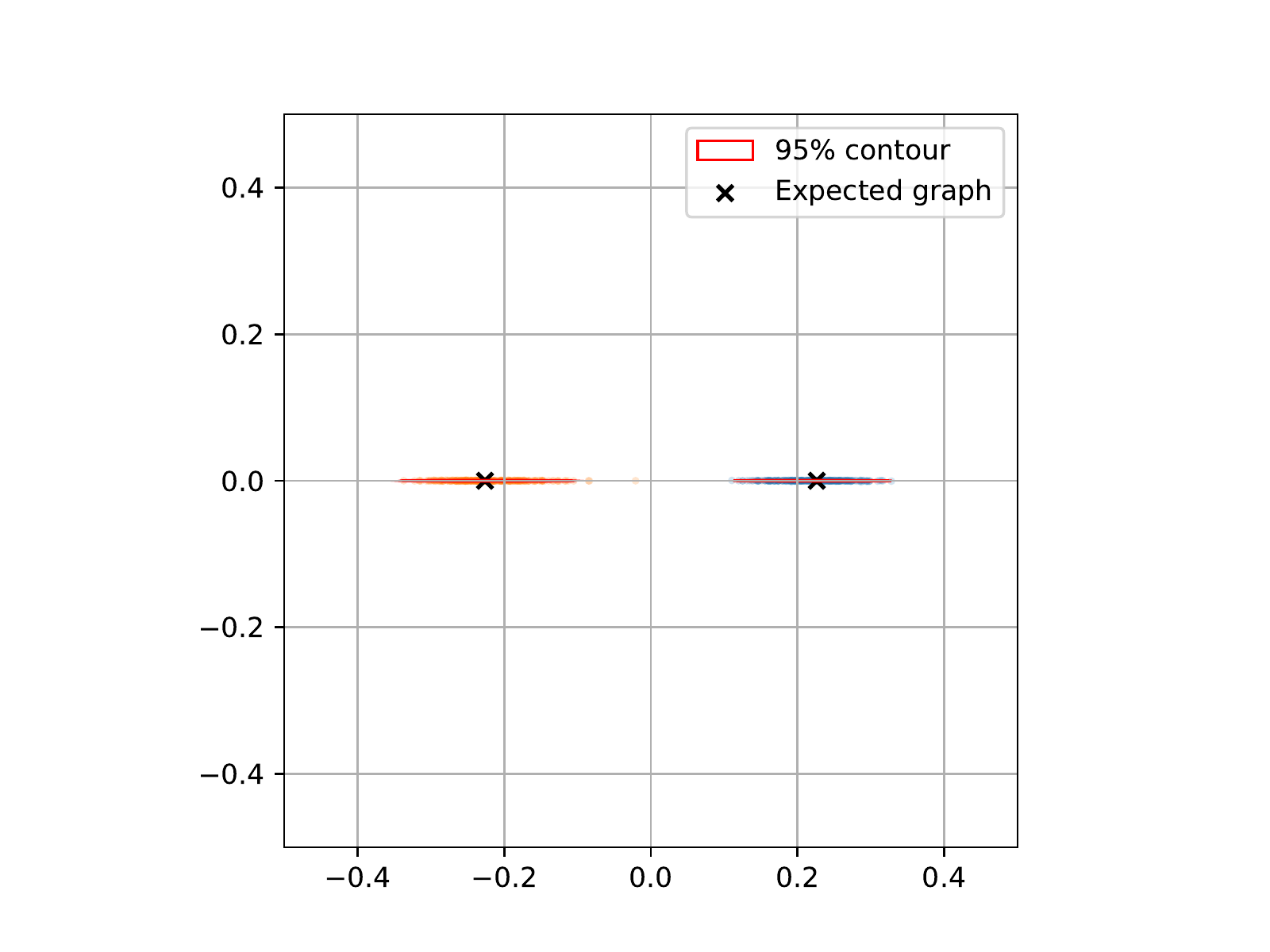}
        \caption{\small ~~~~~$\nu_0 = 0.054$}
    \end{subfigure}
	\begin{subfigure}{0.5\linewidth}
		\includegraphics[trim = 70pt 20pt 85pt 35pt, clip, width=\linewidth]{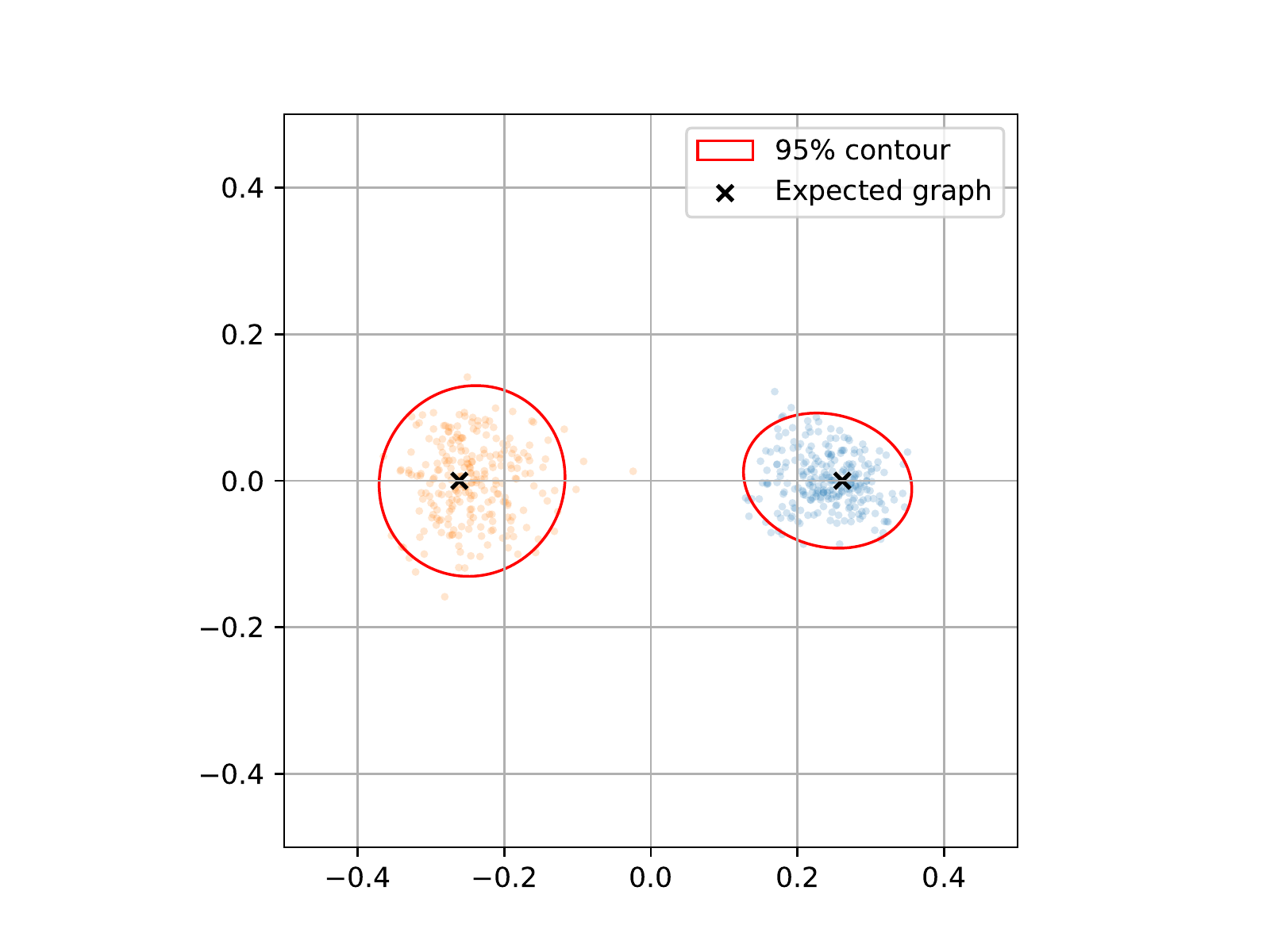}
        \caption{\small ~~~~~$\nu_0 = 0.072$}
    \end{subfigure}%
    \begin{subfigure}{0.5\linewidth}
		\includegraphics[trim = 70pt 20pt 85pt 35pt, clip, width=\linewidth]{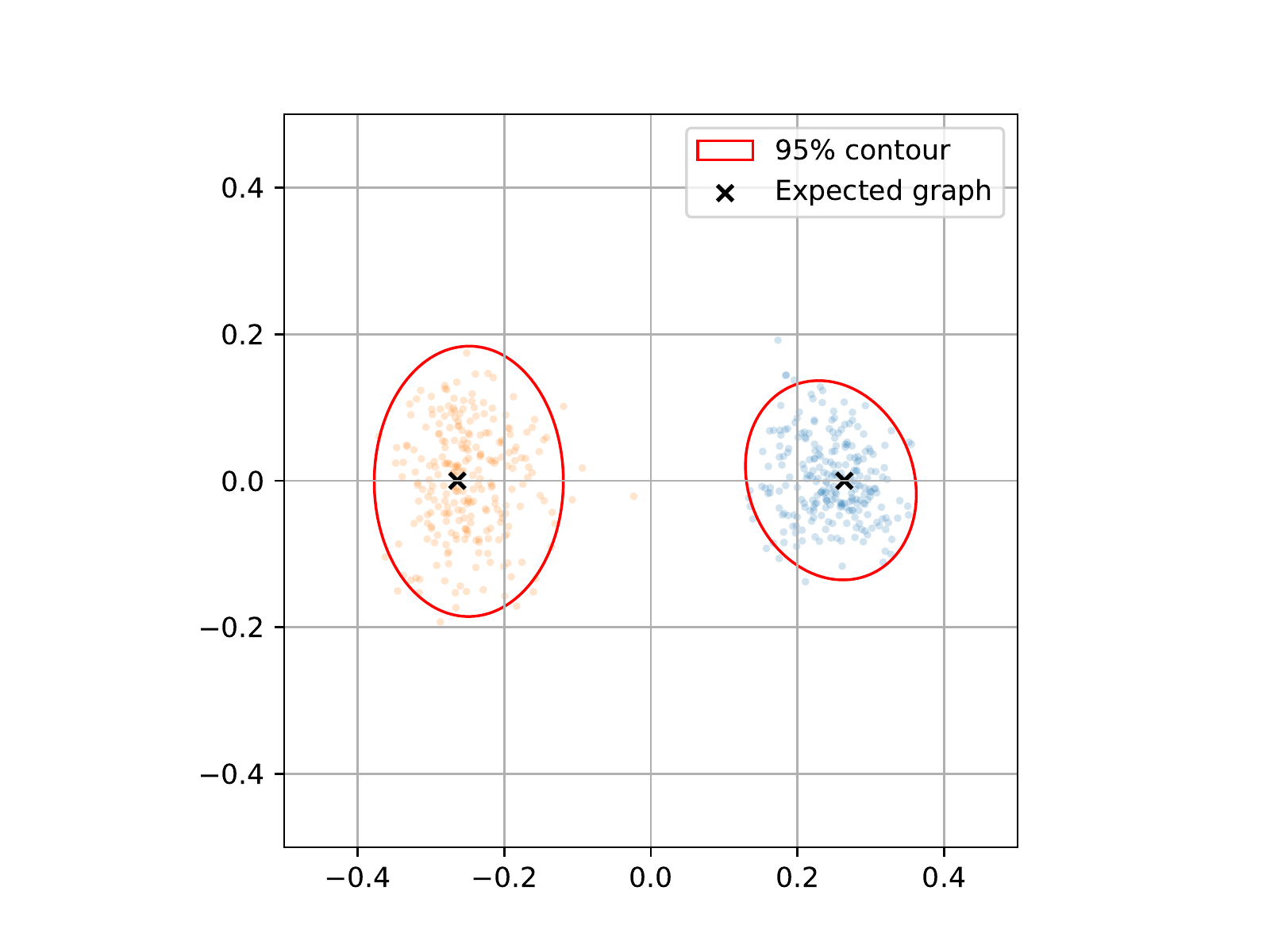}
        \caption{\small ~~~~~$\nu_0 = 0.108$}
    \end{subfigure}
\vglue -1ex    
    \caption{\small 2D-visualization of NucGramErgoVEC embeddings for different nuclear norm linear scaling factors. The input graph is the same as the one used in Fig.~\ref{fig:exp_geometry_scatter_other}.}
    \label{fig:exp_geometry_scatter_nuc}
\end{figure}

When $\nu_0$ is very small, the embeddings are one dimensional ({\it cf.} Fig.~\ref{fig:exp_geometry_scatter_nuc}(a)). As $\nu_0$ increases slightly, the embeddings remains one dimensional but spread out within each cluster and the cluster centroids move apart ({\it cf.} Fig.~\ref{fig:exp_geometry_scatter_nuc}(b)). This continues until $\nu_0$ reaches a threshold. When $\nu_0$ increases beyond the threshold, the embeddings stop extending in the first dimension and start to spread in the second dimension ({\it cf.} Fig.~\ref{fig:exp_geometry_scatter_nuc}(c)(d)).

In order to obtain a more quantitative understanding of how $\nu_0$ influences the embedding geometry, we plot the 1D-SNR of embeddings and their variance in the second dimension for a range of values of $\nu_0$ in Fig.~\ref{fig:exp_metrics_vs_nn}.
\begin{figure}
\centering
    \begin{subfigure}{0.5\linewidth}
		\includegraphics[trim = 10pt 10pt 45pt 40pt, clip, width=\linewidth]{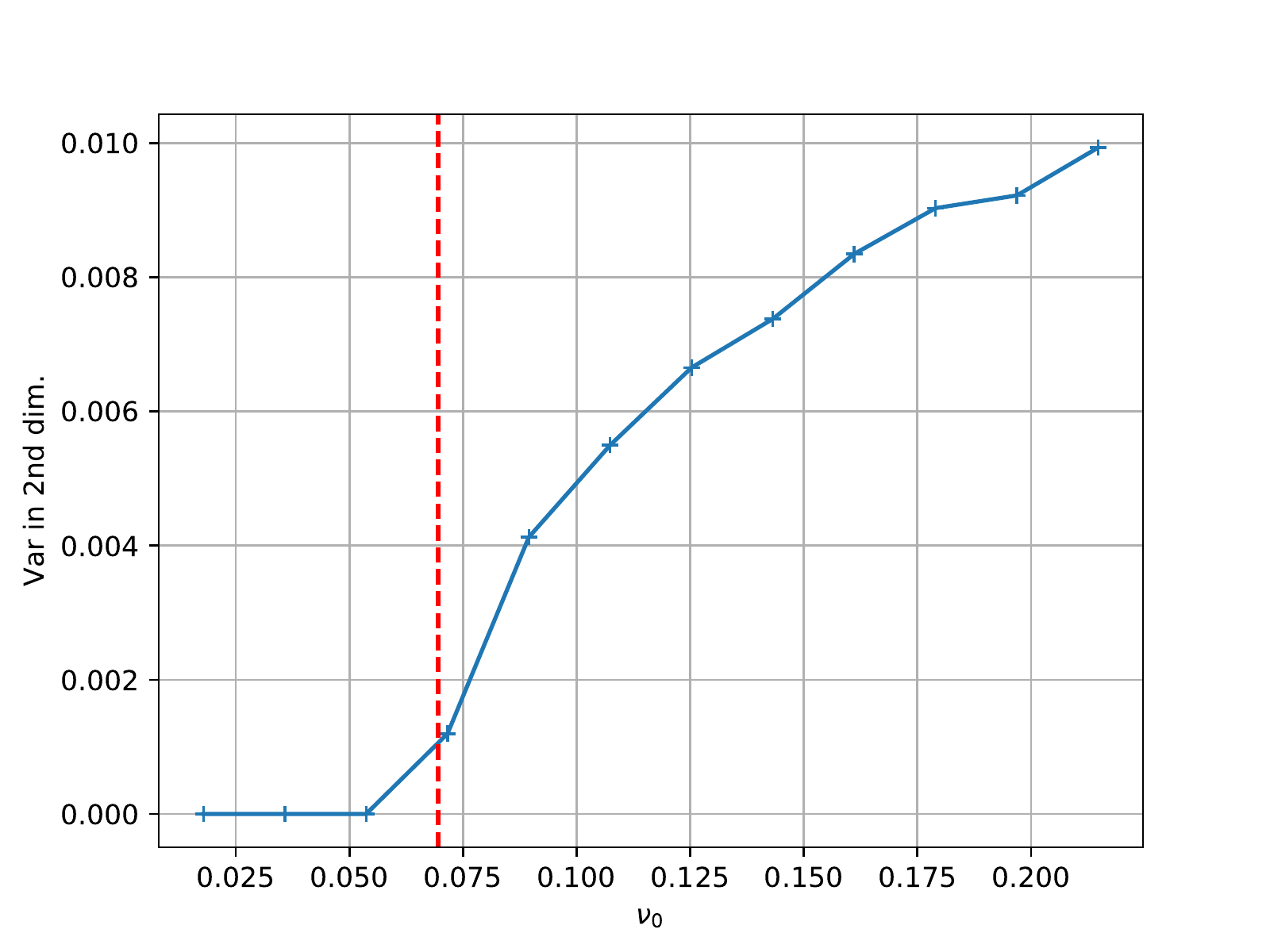}
        \caption{Variance in 2nd dim.}
        \label{fig:exp_second_dim_variance_vs_nn_log}
    \end{subfigure}%
        \begin{subfigure}{0.5\linewidth}
        \centering
		\includegraphics[trim = 10pt 10pt 45pt 40pt, clip, width=\linewidth]{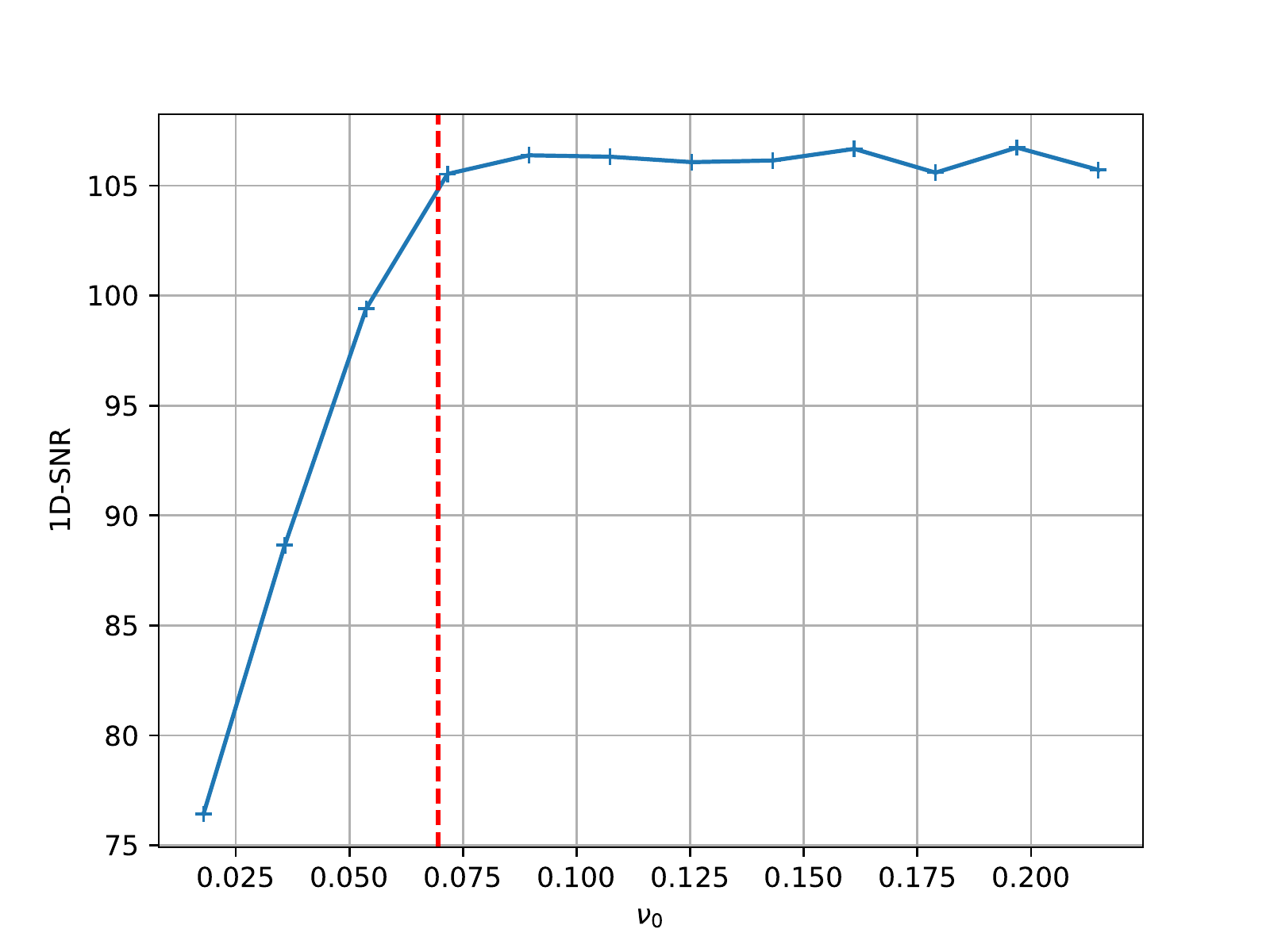}
        \caption{1D-SNR}
        \label{fig:exp_snr_vs_nn_log}
    \end{subfigure}
\vglue -1ex    
    \caption{\small The change in (a) variance in second dimension and (b) 1D-SNR of NucGramErgoVEC embeddings as the nuclear norm linear scaling factor $\nu_0$ increases. The input graph has $n=200$ nodes generated with within community edge probabilities $p = 9 \ln{n} / n$ and across community edge probabilities $q = 2 \ln{n} / n$.}
    \label{fig:exp_metrics_vs_nn}
\vglue -2ex
\end{figure}

Fig~\ref{fig:exp_metrics_vs_nn}(a) shows how the variance of embeddings in the second dimension changes as the nuclear norm linear scaling factor $\nu_0$ increases. When $\nu_0$ is very small, the variance in the second dimension is $0$, suggesting that embeddings are $1$ dimensional. As $\nu_0$ increases, the variance in the second dimension remains zero until $\nu_0$ crosses a threshold that lies somewhere between $\nu_0 =0.054$ and $\nu_0 = 0.072$ and then thereafter the variance increases monotonically. The exact value of $\nu_0$ where the second dimension variance emerges 
depends on the input graph in general and specifically on the edge forming probability. 

Fig~\ref{fig:exp_metrics_vs_nn}(b) shows how $\nu_0$ affects 1D-SNR of the embeddings. Here, we see a clear increase of 1D-SNR as $\nu_0$ increase from $0.018$ to $0.072$. A relative maximum level is reached when the $\nu_0$ is around the transition point where the second dimension variance emerges. Beyond the transition point, the 1D-SNR holds steady around the maximum level. These properties are consistent with our observations for Fig.~\ref{fig:exp_geometry_scatter_nuc}.

\subsection{Concentration of embeddings}
\label{Subsec:exp_asymptotic}

After understanding how the geometry of embeddings of a single graph differs across embedding algorithms and changes with $\nu_0$, in this section, we explore how the embeddings change as  $n$, the number of nodes, increases. To focus on the effect of increasing the number of nodes, throughout this section, we fix the scaling factors of edge forming probabilities within and across communities in each set of experiments. In addition, we omit the results of VEC because of their similarity to {\ErgoVEC} ({\it cf.} Fig.~\ref{fig:exp_walk_length_metric}). 
As we will see, the asymptotic behavior of embeddings largely depends on the edge forming probability.

To gain a qualitative perspective, we first plot the embeddings and their $95\%$ Gaussian contours for graph sizes $n=100, 500, 1000$ for each algorithm. Fig.~\ref{fig:exp_concentration_linear} shows the embedding contours in the linear degree scaling regime. We can see that all the contours shrink as $n$ increases. This suggests that empirically, the embeddings of all the four algorithms concentrate to their centroids. 
In the logarithmic degree scaling regime, as shown in Fig.~\ref{fig:exp_concentration_logarithm}, the Gaussian contours for different graph sizes mostly overlap on top of each other, suggesting a convergence in distribution as opposed to a concentration that we observed in the linear regime.
\begin{figure}[ht]
\captionsetup[subfigure]{labelformat=empty}
\centering
	\begin{subfigure}{0.5\linewidth}
        \centering
		\includegraphics[trim = 35pt 65pt 35pt 85pt, clip, width=\linewidth]{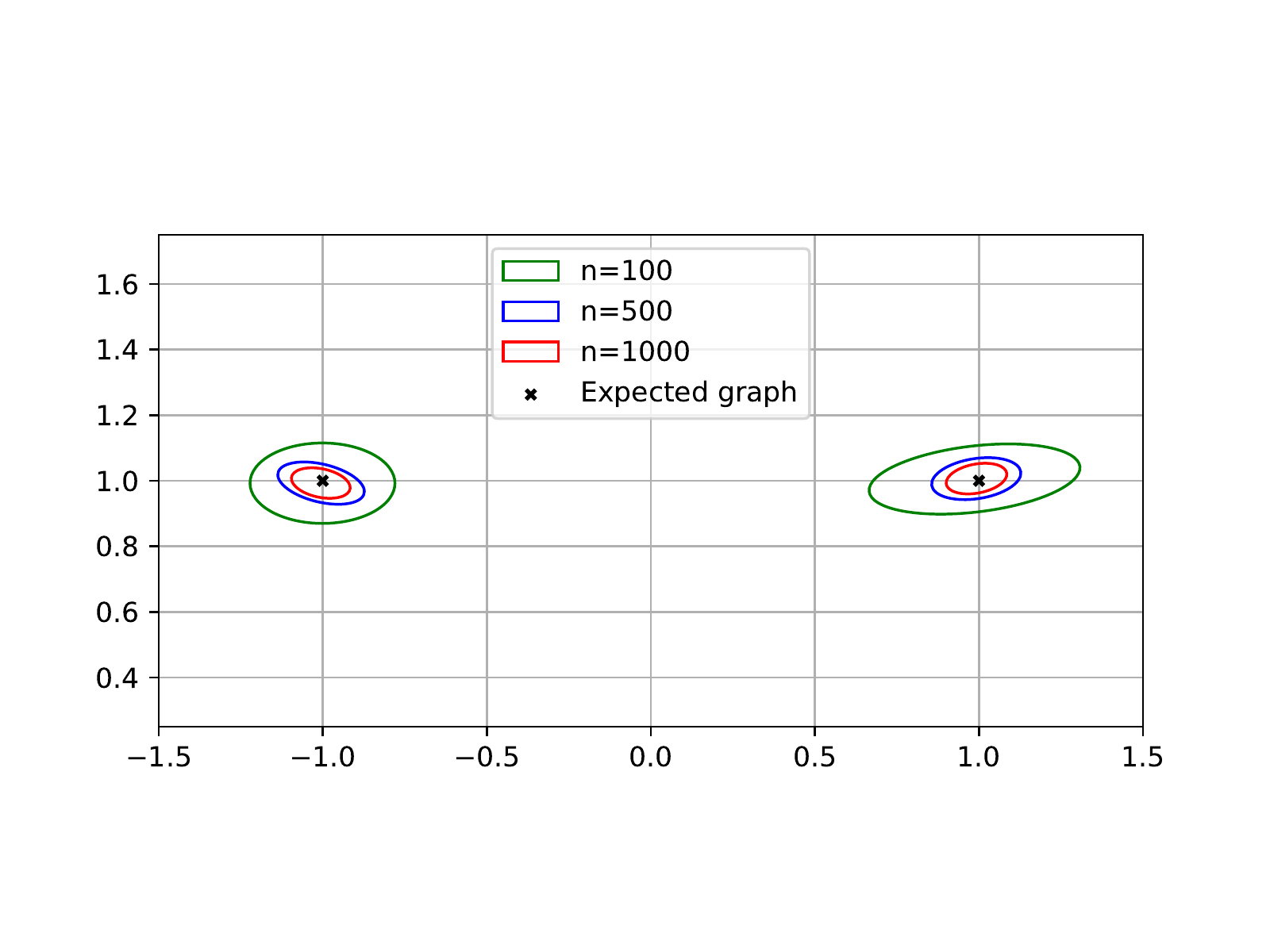}
        \caption{~~SC}
    \end{subfigure}%
    \begin{subfigure}{0.5\linewidth}
        \centering
		\includegraphics[trim = 35pt 65pt 35pt 85pt, clip, width=\linewidth]{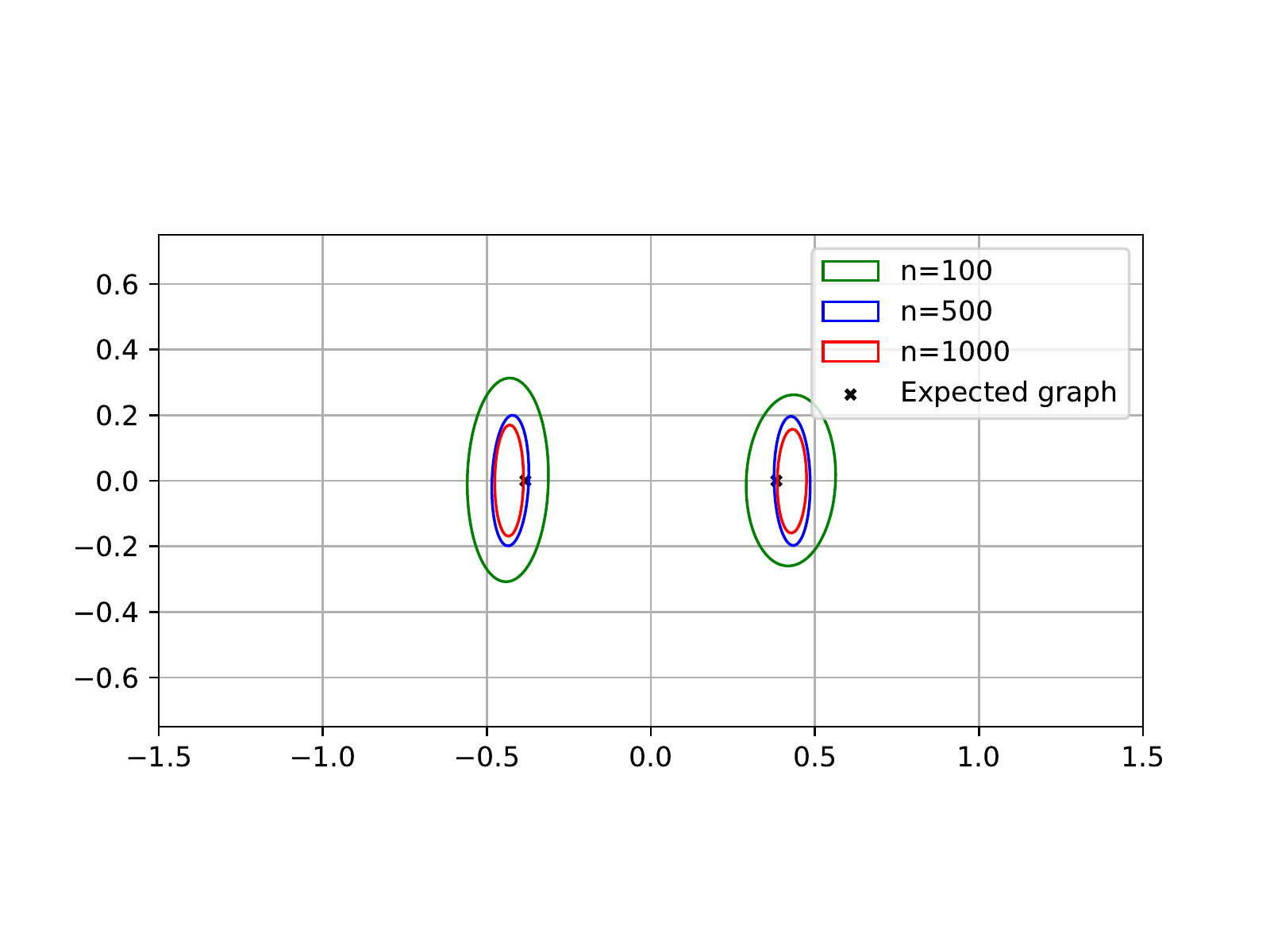}
        \caption{~~{\ErgoVEC}}
    \end{subfigure}
	\begin{subfigure}{0.5\linewidth}
        \centering
		\includegraphics[trim = 35pt 65pt 35pt 85pt, clip, width=\linewidth]{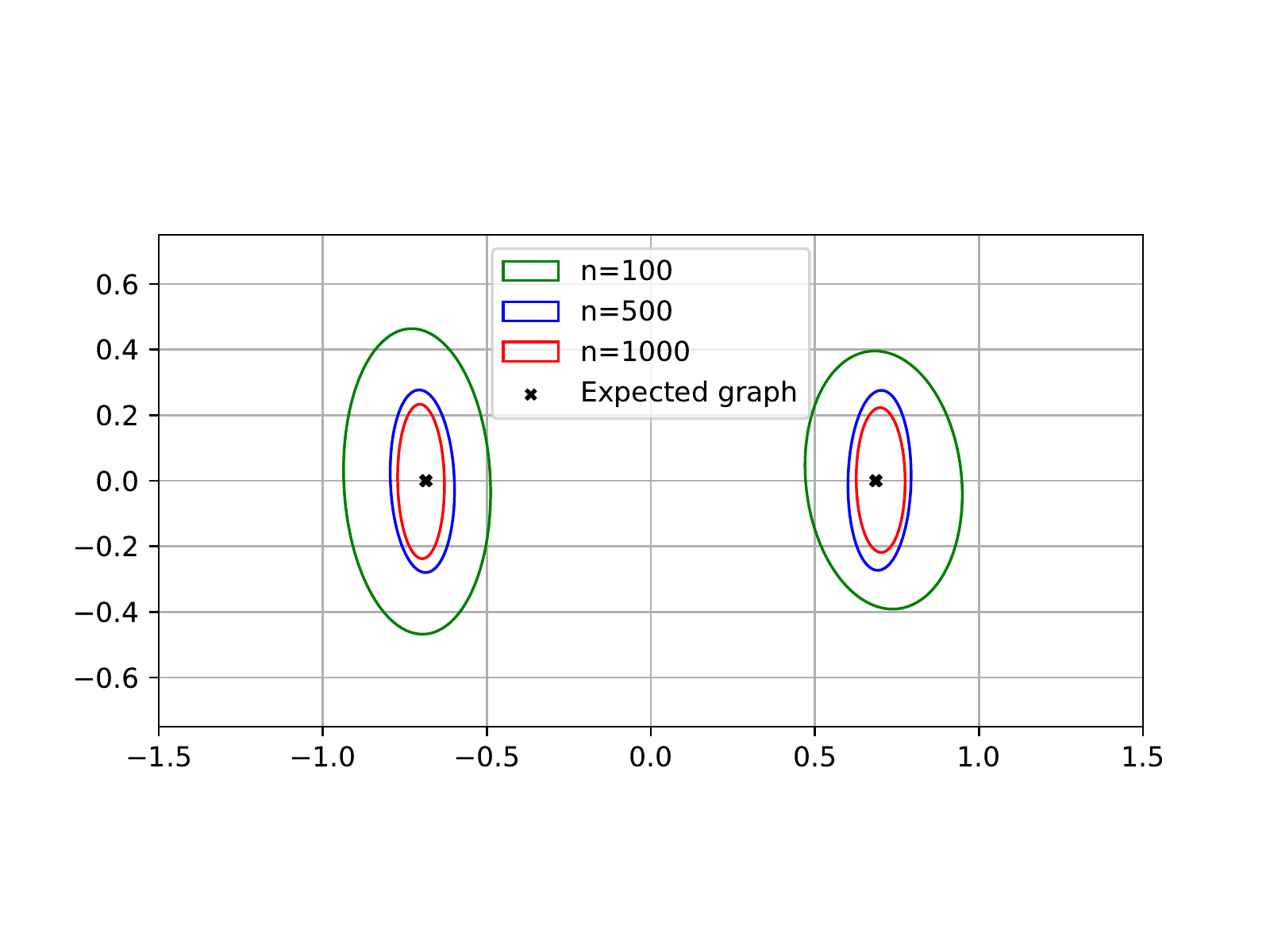}
        \caption{~~ErgoPMI \\ $\quad $}
    \end{subfigure}%
    \begin{subfigure}{0.5\linewidth}
        \centering
		\includegraphics[trim = 35pt 65pt 35pt 85pt, clip, width=\linewidth]{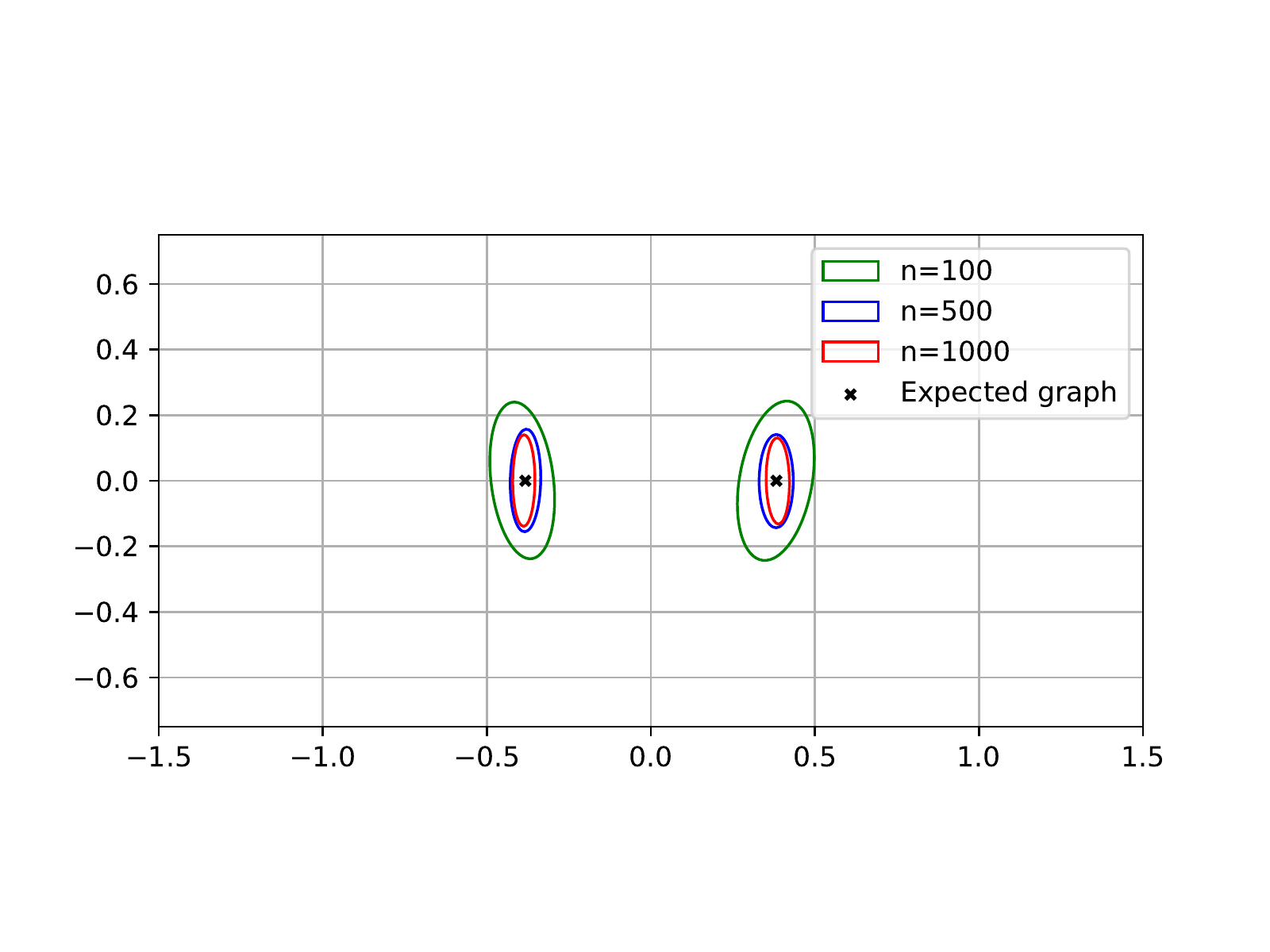}
        \caption{~~NucGramErgoVEC \\ $\nu_0 = 0.108$ }
    \end{subfigure}
\vglue -1ex
	\caption{\small 95\% Gaussian contours of 2D-embeddings from four algorithms in the linear regime. All algorithms receive the same sets of graphs with $n=100, 500$ and $1000$ nodes generated using within-community edge probabilities $p = 0.6$ and across community edge probabilities $q = 0.06$.}
	\label{fig:exp_concentration_linear}
\end{figure}
\begin{figure}[ht]
\captionsetup[subfigure]{labelformat=empty}
\centering
	\begin{subfigure}{0.495\linewidth}
        \centering
		\includegraphics[trim = 75pt 20pt 83pt 35pt, clip, width=\linewidth]{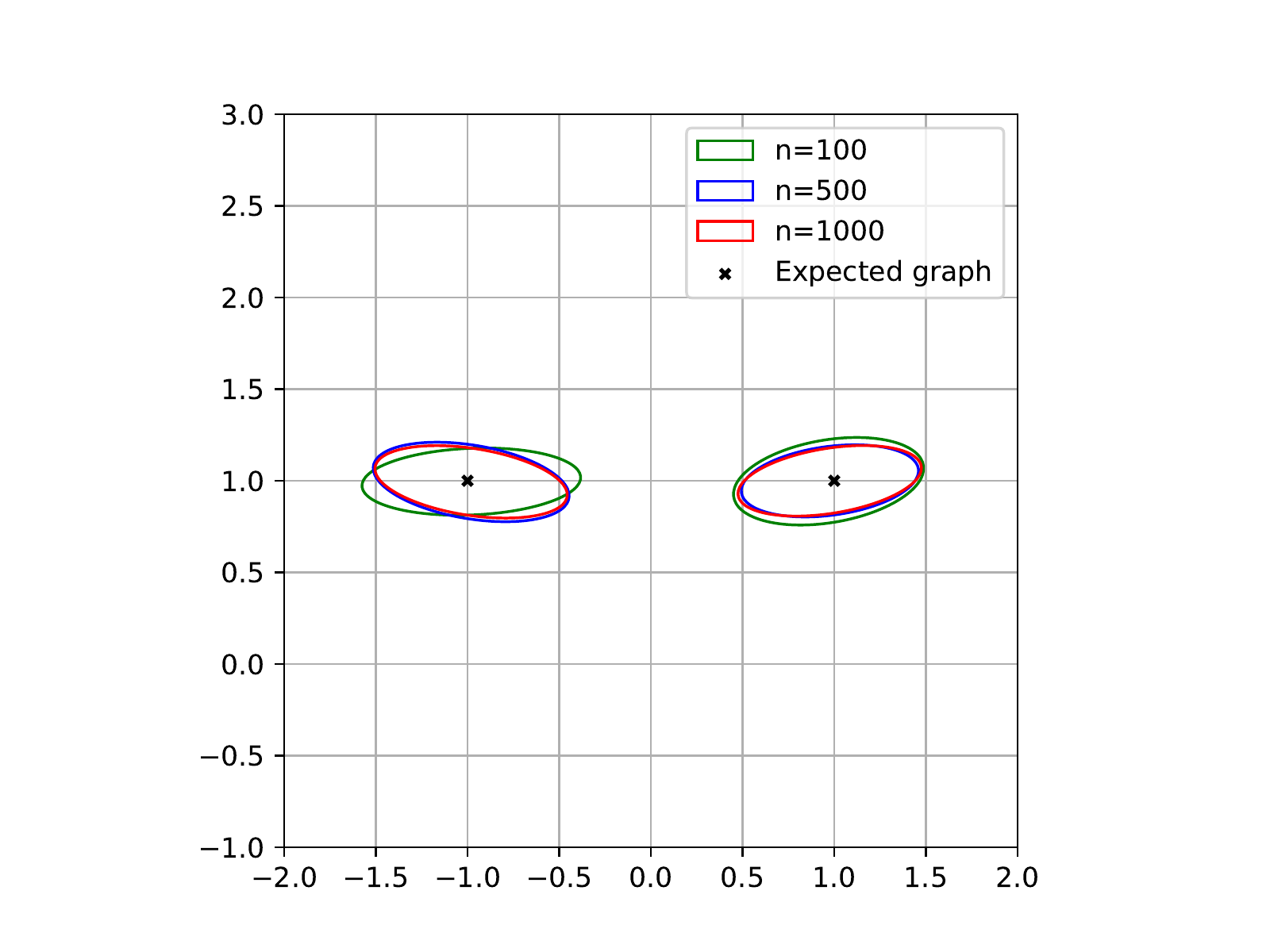}
        \caption{~~~~SC}
    \end{subfigure}
    \begin{subfigure}{0.495\linewidth}
        \centering
		\includegraphics[trim = 75pt 20pt 83pt 35pt, clip, width=\linewidth]{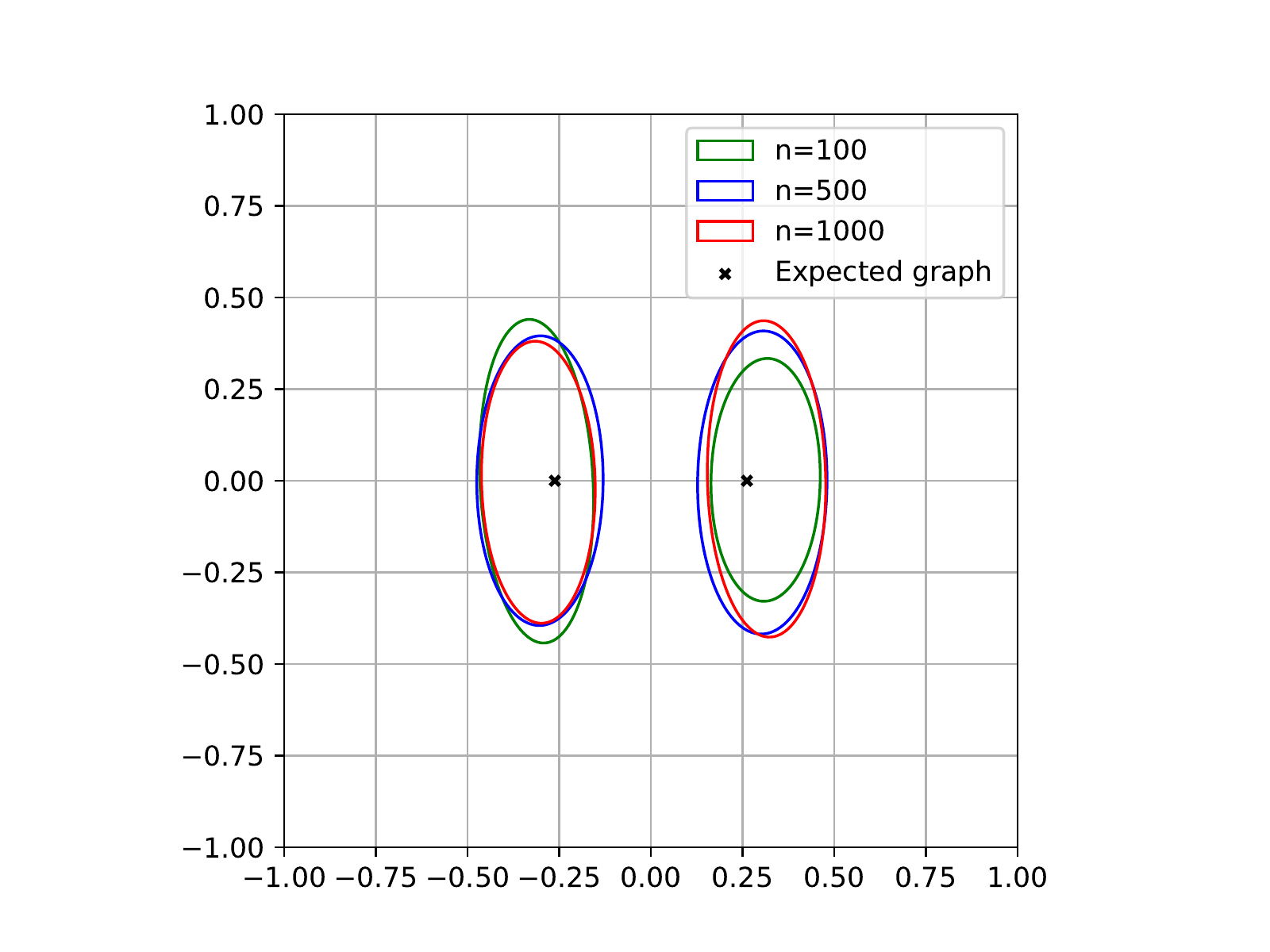}
        \caption{~~~~{\ErgoVEC}}
    \end{subfigure}
	\begin{subfigure}{0.495\linewidth}
        \centering
		\includegraphics[trim = 75pt 20pt 83pt 35pt, clip, width=\linewidth]{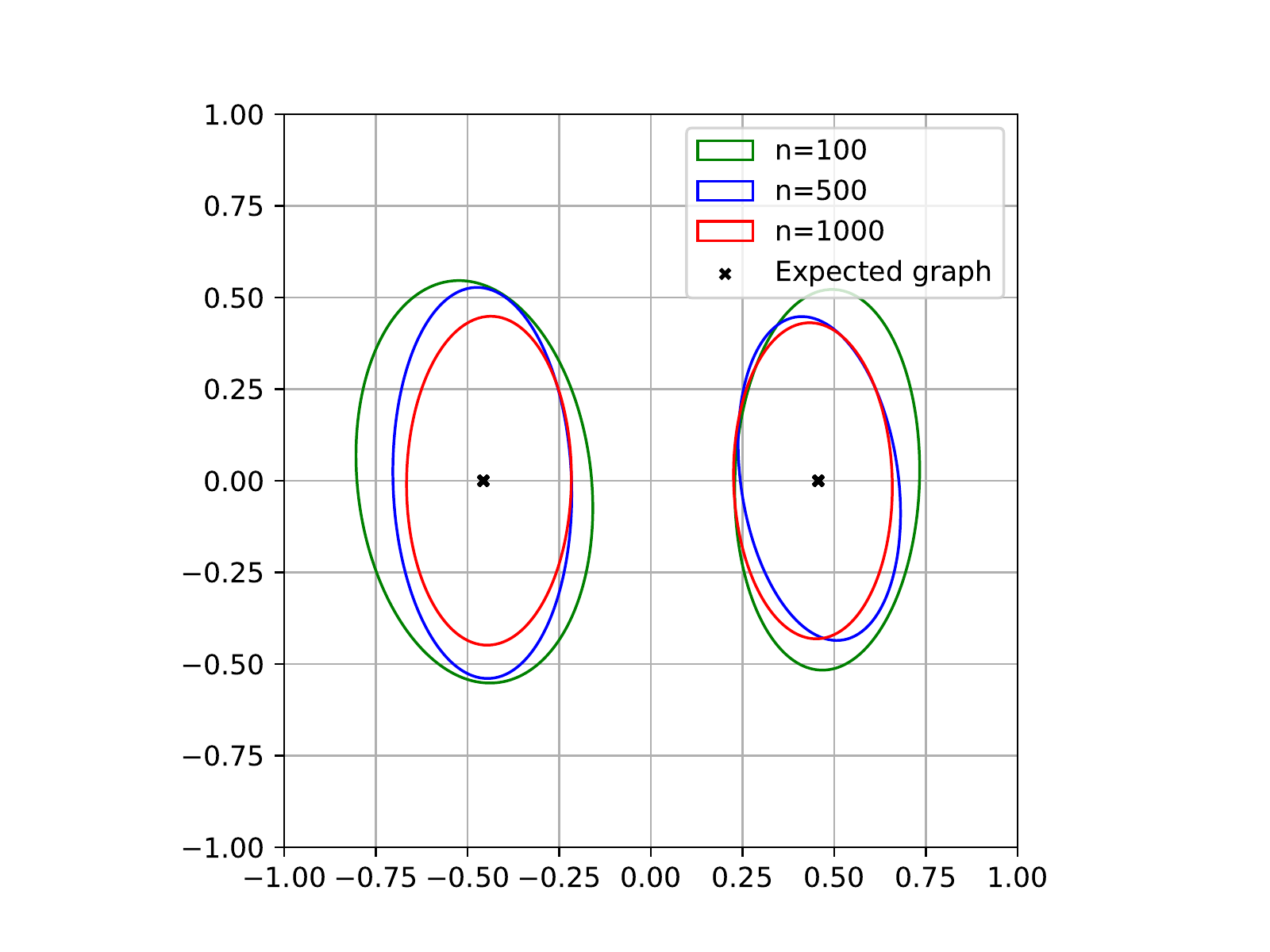}
        \caption{~~~~ErgoPMI \\ $\quad$}
    \end{subfigure}
    \begin{subfigure}{0.495\linewidth}
        \centering
		\includegraphics[trim = 75pt 20pt 83pt 35pt, clip, width=\linewidth]{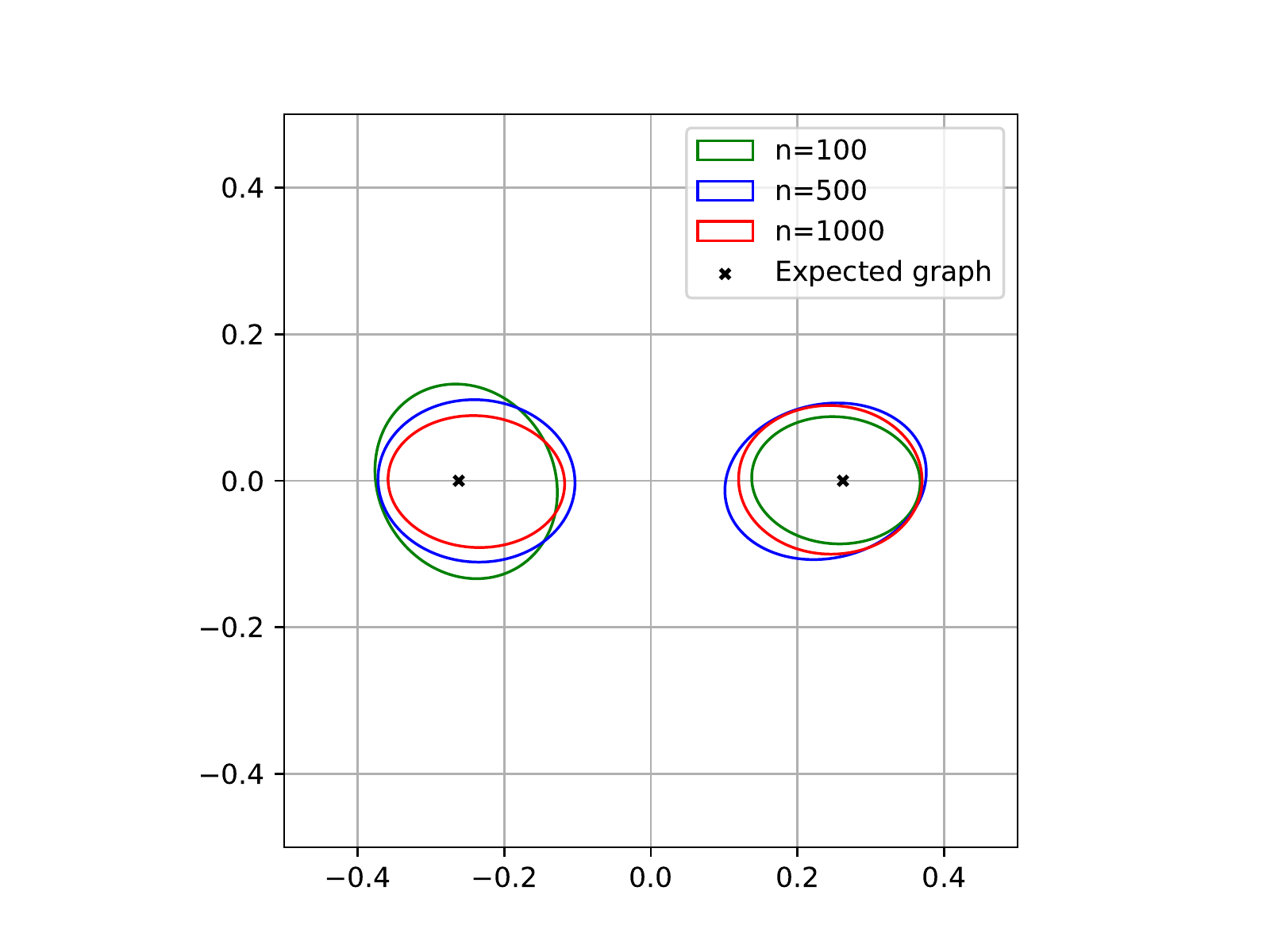}
        \caption{~~NucGramErgoVEC \\ $\nu_0 = 0.108$ }
    \end{subfigure}
\vglue -1ex
	\caption{\small 95\% Gaussian contours of 2D-embeddings from four algorithms in the logarithmic regime. All algorithms receive the same sets of graphs with $n=100, 500$ and $1000$ nodes generated using within-community edge probabilities $p = 9 \ln{n} / n$ and across community edge probabilities $q = 2 \ln{n} / n$.}
	\label{fig:exp_concentration_logarithm}
\end{figure}

We turn to quantitative metrics to gain a more nuanced understanding. 
In Fig.~\ref{fig:exp_snr_comparison}, we plot the 1D-SNR of embeddings for increasing values of $n$ for each algorithm. Note that a higher 1D-SNR indicates either a smaller within group variance along the line that joins the cluster centroids or a greater distance between cluster centroids. 
Fig.~\ref{fig:exp_snr_comparison}(a) shows results for the linear degree scaling regime, where we see that 1D-SNR increases as $n$ increases, with NucGramErgoVEC leading, followed by SC, ErgoVEC and ErgoPMI. The VEC embeddings obtained from both implementations (Keras and Gensim) reside at the bottom. 
In the logarithmic degree scaling regime, as shown in Fig.~\ref{fig:exp_snr_comparison}(b), the 1D-SNR is relatively steady across different $n$ as opposed to a clear a growth trend observed in the linear degree scaling regime. This is consistent with our observations for Fig.~\ref{fig:exp_concentration_linear} and Fig.~\ref{fig:exp_concentration_logarithm} that the embeddings concentrate in linear regime but converges to a fixed distribution in the logarithmic regime. While VEC embeddings still under perform, ErgoPMI and ErgoVEC surpasses SC and catches NucGramErgoVEC's lead.

\begin{figure}[h]
\centering
	\begin{subfigure}{\linewidth}
        \centering
		\includegraphics[trim = 25pt 20pt 45pt 40pt, clip, width=\linewidth]{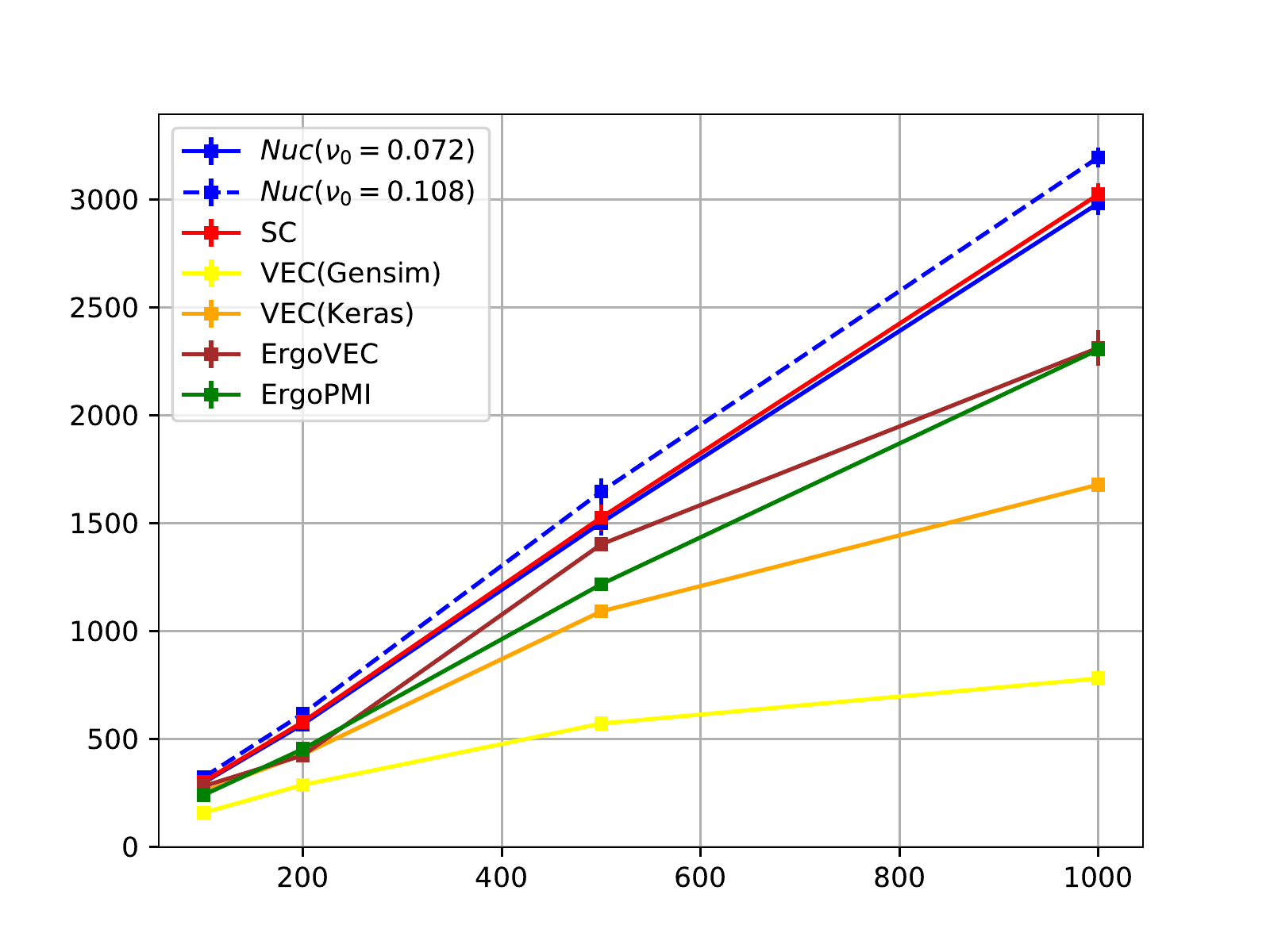}
        \caption{Linear regime}
    \end{subfigure}
\vglue 2ex
    \begin{subfigure}{\linewidth}
        \centering
		\includegraphics[trim = 25pt 20pt 45pt 40pt, clip, width=\linewidth]{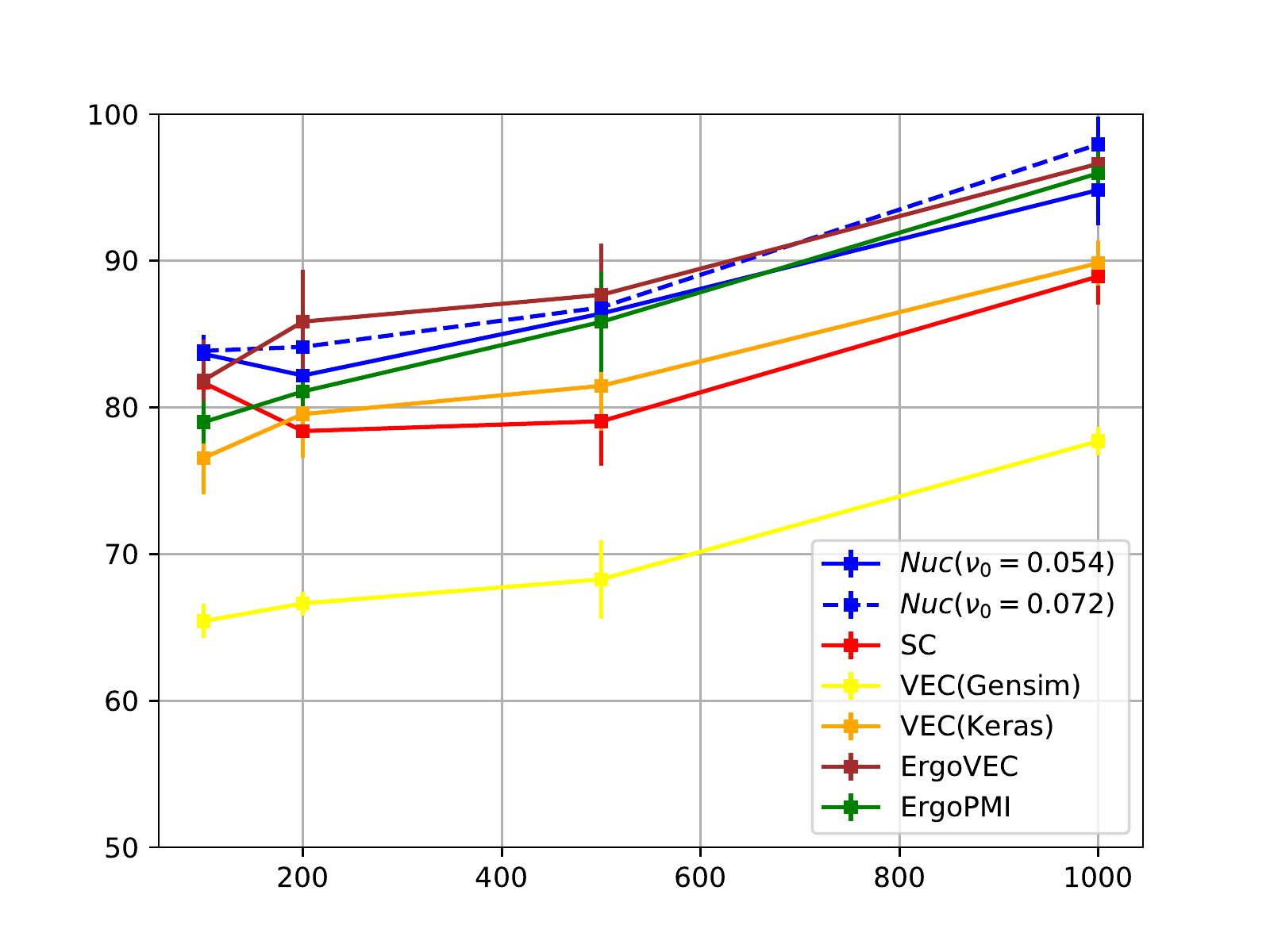}
        \caption{Logarithmic regime}
    \end{subfigure}
	\caption{\small 1D-SNR versus graph size in the linear and logarithmic scaling regimes.}
	\label{fig:exp_snr_comparison}
\end{figure}

\section{Concluding remarks}
\label{Sec:concl}

In this paper, we proposed a novel framework consisting of ergodic limits of random walks and a Grammian re-parameterization of the embedding objective to analyze a large class of random walk based node-embedding algorithms. In particular, we derived a closed-form expression for the ergodic limit of the random walk node embedding objective and proved that under the positive semi-definite constraint, the Gram matrix of optimum embedding vectors for two-community expected SBM graphs has either rank $1$ or rank $2$. 
In addition, through an empirical study we demonstrated that the embeddings based on ergodic limits, while forming better clusters, in terms of 1D-SNR, compared to the original random walk embeddings, concentrate to the embeddings of the expected graph in the linear degree scaling regime and seem to converge to a fixed distribution in the logarithmic regime.

Computational costs can vary substantially across different algorithms. For example, the Gram matrix of the optimal embedding vectors in {\ErgoPMI} has a simple closed-form solution, whereas 
the better performing {\NucGramErgoVEC} requires a computationally expensive iterative optimization procedure to compute the optimal Grammian.  This suggests a possible trade-off between computational cost and accuracy of algorithms. Although not the focus of this paper, understanding these trade-offs would benefit the end users of these methods.

The results of this paper can be further improved and extended on both theoretical and practical fronts. 
For simplicity we focused on SBM graphs with two balanced communities. Our theoretical and experimental results can be potentially extended to more complex graph models that have a community structure.
On the theoretical side, although we have shown perfect separation of the embeddings of the expected graph, there is no theoretical guarantee that the embeddings of SBM random graphs will concentrate to those of the expected graph. Further analysis of random walk embedding algorithms, especially the concentration properties of their solutions in various degree scaling regimes, would bring us more insight and understanding. 
On the practical side, the convergence of our Keras implementations for VEC and ErgoVEC depend highly on tuning parameters and may not converge very well, and the Hazan's algorithm for {\NucGramErgoVEC} suffers from slow convergence. Developing more scalable implementations of algorithms with faster and more stable convergence can bring these generalized formulations into large-scale real-world problems and also guide the theoretical analysis endeavor.

\appendices

\section{}
\label{appendix}

\subsection{Proof of \texorpdfstring{\Cref{Thm:Ergolim}}{Theorem 4.1}} 
\label{app:ErgoLim}

Natural random walks will remain within the connected components in which they start. Since only pairs of nodes within the same connected component will occur in any random walk, we can analyze each connected component separately. Within each connected component $\mathcal{G}_t$, the random walk has transition matrix $W_t = D_t^{-1} A_t$. The proof of the theorem will follow immediately from the following lemma which focuses on connected graphs.
\begin{lemma}
	\label{Lemma:Ergolim} 
	Let $W$ be the probability transition matrix of an irreducible Markov chain on the (finite) node space of $\mathcal{G}$. Let the VEC algorithm be executed on $\mathcal{G}$ with random walk transition matrix $W$ and parameters $w$ and $k$. Then for all $i,j$, the ergodic limits $\bar{n}_{ij}^+$ and $\bar{n}_{ij}^-$ in \Cref{Defi:Ergolim} exist and are given by
	\begin{align*}
		\bar{n}_{ij}^+={}& \pi_i\sum_{v=1}^{w}(W^v)_{ij}, \\
		\bar{n}_{ij}^-={}& kw\pi_i\pi_j,
	\end{align*}
	where $\bm{\pi}$ is the unique stationary distribution of the random walk. Moreover, the Ergodic limiting coefficients are symmetric, i.e.,
	\begin{align*}
		\bar{n}_{ij}^+={}& \bar{n}_{ji}^+, \\
		\bar{n}_{ij}^-={}& \bar{n}_{ji}^-.
	\end{align*}
\end{lemma}
The proof of \Cref{Lemma:Ergolim} 
is based on convergence results for irreducible Markov chains.

First, we prove the result for positive pairs.
Let $\{X_s^{(m,p)}\}_{s=1}^{\infty}$ be the $p$-th random walk starting from node $m$ following the transition matrix $W$. We examine the first $\ell$ steps in each random walk. Since the $n_{ij}^+$'s consist of positive pairs extracted from all $rn$ random walks ($r$ walks from each of the $n$ nodes), we have
\begin{equation}
\label{Eq:Lemma1_proof_pos_pair_count}
    \frac{n_{ij}^+}{nr\ell} ={} \frac{1}{nrl}\sum_{m=1}^n\sum_{p=1}^r\sum_{v=1}^w\sum_{s=1}^{\ell-v} \mathbf{1}_{\{X_s^{(m,p)} = i, X_{s+v}^{(m,p)} = j\}}.
\end{equation}
Letting $\ell$ go to infinity on both sides, we have
\begin{align*}
\MoveEqLeft
   \frac{1}{rn}\lim_{\ell\rightarrow\infty}\frac{n_{ij}^+}{\ell} \\
   ={} &\frac{1}{nr}\sum_{m=1}^n\sum_{p=1}^r\sum_{v=1}^w \lim_{\ell\rightarrow\infty}\frac{1}{\ell}\sum_{s=1}^{\ell-v} \mathbf{1}_{\{X_s^{(m,p)} = i, X_{s+v}^{(m,p)} = j\}}.
\end{align*}
The key step is to compute $\lim\limits_{\ell\rightarrow\infty}\frac{1}{\ell}\sum_{s=1}^{\ell-v} \mathbf{1}_{\{X_s^{(m,p)} = i, X_{s+v}^{(m,p)} = j\}}$. To begin, we define a new Markov Chain $\{Y_s^{(m,p)}\}_{s=1}^{\infty}$, where $Y_s^{(m,p)}\! = (X_s^{(m,p)}\!, X_{s+1}^{(m,p)}\!, \ldots, X_{s+v}^{(m,p)})$. The state space of $\{Y_s^{(m,p)}\}_{s=1}^{\infty}$ is the set of all length-$(v+1)$ walks under $W$, i.e., $S_v = \{(i_1,i_2,\ldots,i_{v+1})\ \big|\ i_{k+1}\text{ is accessible from }i_{k}, k=1,\ldots,v\}\subset{[n]^{v+1}}$. 

We claim that $\{Y_s^{(m,p)}\}_{s=1}^{\infty}$ is a positive recurrent Markov Chain. To see this, we first note that the state space is finite as $\vert S_v\vert\leq\vert[n]^{v+1}\vert = n^{v+1}<\infty$. Then, we notice that $\forall \bm{a}, \bm{b}\in S_v$, since $W$ is irreducible, $b_1$ is reachable from $a_{v+1}$ in $X_s^{(m,p)}$. Therefore, $\bm{b}$ is also reachable from $\bm{a}$, which shows that $Y_s^{(m,p)}$ is irreducible. An irreducible Markov chain on a finite state space must be positive recurrent. 

Applying standard results from renewal theory, specifically~\cite[Proposition 3.3.1, p.102]{ross1996stochastic} to Markov chain $\{X_s^{(m,p)}\}_{s=1}^{\infty}$ and~\cite[Theorem 3.3.4, p.107]{ross1996stochastic} to Markov chain $\{Y_s^{(m,p)}\}_{s=1}^{\infty}$, we get
\begin{align}
	&\lim_{\ell\rightarrow\infty}\frac{\sum_{s=1}^{\ell}\mathbf{1}_{\{X_s^{(m,p)} = i\}}}{\ell}=\pi_{i}, \qquad \text{a.s.},
	\notag\\
	&\lim_{\ell\rightarrow\infty}\frac{\mathbb{E}\sum_{s=1}^{\ell}\mathbf{1}_{\{X_s^{(m,p)} = i\}}}{\ell}=\pi_{i}, \label{Eq:Lemma1_proof_+X_e} \\
	&\lim_{\ell\rightarrow\infty}\frac{\sum_{s=1}^{\ell}\mathbf{1}_{\{Y_s^{(m,p)} = \bm{a}\}}}{\ell}=\eta_{\bm{a}}, \qquad \text{a.s.},
	\label{Eq:Lemma1_proof_+Y_as} \\
	&\lim_{\ell\rightarrow\infty}\frac{\mathbb{E}\sum_{s=1}^{\ell}\mathbf{1}_{\{Y_s^{(m,p)} = \bm{a}\}}}{\ell}=\eta_{\bm{a}},  \notag
\end{align}
where $\bm{\pi}$ and $\bm{\eta}$ are the stationary distributions of $X_s^{(m,p)}$ and $Y_s^{(m,p)}$, respectively, and do not depend on $m,p$ because of the positive recurrence and irreducibility of the Markov chains. Note that the relationship between state counts of $X_s^{(m,p)}$ and $Y_s^{(m,p)}$ is given by
\begin{equation*}
	\mathbf{1}_{\{X_s^{(m,p)} = i, X_{s+v}^{(m,p)} = j\}} = \sum_{\bm{a}:\ a_1 = i, a_{v+1} = j}\mathbf{1}_{\{Y_s^{(m,p)} = \bm{a}\}}.
\end{equation*}
And thus, 
\begin{align*}
\MoveEqLeft
      \lim_{\ell\rightarrow\infty}\frac{1}{\ell}\sum_{s=1}^{\ell} \mathbf{1}_{\{X_s^{(m,p)} = i, X_{s+v}^{(m,p)} = j\}} \\
    ={}&\lim_{\ell\rightarrow\infty}\frac{1}{\ell}\sum_{\bm{a}:\ a_1 = i, a_{v+1} = j}\mathbf{1}_{\{Y_s^{(m,p)} = \bm{a}\}} \\
    \stackrel{\mathclap{\mathrm{\eqref{Eq:Lemma1_proof_+Y_as}}}}{=}{}& \sum_{\bm{a}:\ a_1 = i, a_{v+1} = j}\eta_{\bm{a}}\\
	={}&\sum_{\bm{a}:\ a_1 = i, a_{v+1} = j} \lim_{s \to \infty}\mathbb{P}[Y_s^{(m,p)} = \mathbf{a}] \\
	={}&\lim_{s \to \infty} \mathbb{P}(X_{s}^{(m,p)}=i,X_{s+v}^{(m,p)}=j) \\
	={}&\lim_{\ell\rightarrow\infty}\frac{\sum_{s=1}^{\ell}\mathbb{P}(X_{s}^{(m,p)}=i,X_{s+v}^{(m,p)}=j)}{\ell}\\
	={}&(W)^v_{ij}\lim_{\ell\rightarrow\infty}\frac{\sum_{s=1}^{\ell}\mathbb{P}(X_{s}^{(m,p)}=i)}{\ell}\\
	={}&(W)^v_{ij}\lim_{\ell\rightarrow\infty}\frac{\mathbb{E}\sum_{s=1}^{\ell}\mathbf{1}_{\{X_s^{(m,p)} = i\}}}{\ell}\\
	\stackrel{\mathclap{\mathrm{\eqref{Eq:Lemma1_proof_+X_e}}}}{=}{}& \ (W^v)_{ij}\pi_i.
\end{align*}
Therefore, we have
\begin{equation*}
\lim_{\ell\rightarrow\infty}\frac{1}{\ell}\sum_{s=1}^{\ell} \mathbf{1}_{\{X_s^{(m,p)} = i, X_{s+v}^{(m,p)} = j\}} = (W^v)_{ij}\pi_i,
\end{equation*}
and 
\begin{equation*}
   \frac{1}{rn}\lim_{\ell\rightarrow\infty}\frac{n_{ij}^+}{\ell}
   =\frac{1}{nr}\sum_{m=1}^n\sum_{p=1}^r\sum_{v=1}^w(W^v)_{ij}\pi_i
   =\pi_i\sum_{v=1}^w(W^v)_{ij}.
\end{equation*}

We now analyize the ergodic limits of negative pairs. We first count the number of $(i,j)$ pairs in the negative multi-set.

Let $Z_c^{(m,p,v,s)}$ be the second node in $c$-th negative pair generated from the positive pair $(X_s^{(m,p)}, X_{s+v}^{(m,p)})$ (for each positive pair we generate $k$ negative pairs). Since all negative pairs are generated in an i.i.d. manner, for all $m, p, v, s, c$, $Z_c^{(m,p,v,s)}, c = 1, \ldots, k$ are i.i.d. random variables with a distribution specified by the unigram node frequencies computed from the collection of random walks $X = \bigcup_{m=1}^n \bigcup_{p=1}^r \{X^{(m,p)}\}$. As in \Cref{Eq:Lemma1_proof_pos_pair_count}, the counts of negative pairs is given by
\begin{equation}
\label{Eq:Lemma1_proof_neg_pair_count}
    \frac{n_{ij}^-}{nr\ell} ={} \frac{1}{nr\ell}\sum_{m=1}^n\sum_{p=1}^r\sum_{v=1}^w\sum_{s=1}^{\ell-v} \mathbf{1}_{\{X_s^{(m,p)} = i\}}\sum_{c=1}^{k}\mathbf{1}_{\{Z_{c}^{(m,p,s,v)} = j\}}.
\end{equation}
Letting $\ell$ go to infinity on both sides, we get
\begin{align}
     \MoveEqLeft \frac{1}{nr}\lim_{\ell\rightarrow\infty}\frac{n_{ij}^-}{\ell} \notag
     \\
  ={}& \!\!\!\sum_{v=1}^w\sum_{c=1}^{k}\lim_{\ell\rightarrow\infty}\frac{1}{nr\ell}\sum_{m=1}^n\sum_{p=1}^r\sum_{s=1}^{\ell-v} \mathbf{1}_{\{X_s^{(m,p)} = i\}}\mathbf{1}_{\{Z_{c}^{(m,p,s,v)} = j\}}. 
    \label{Eq:Lemma1_proof_neg_lim}
\end{align}
The remainder of the proof will focus on calculating the right hand side. For this purpose, we introduce the following proposition:

{\noindent \bf Notation}: $[n]\defeq \{1,\ldots,n\}$ and $X_{[n]}\defeq X_1, \ldots, X_n$.
\begin{prp}\label{Prop:product_as_convergence}
Let $\{X_\ell, \ell \in \mathbb{N}\}$ be a sequence of random variables with $X_\ell \in [n]$ for every $\ell$.
Let $\varphi: [n]\rightarrow [0,1]$. For every $L \in \mathbb{N}$, let $\widehat{q}_L: [n]^L \rightarrow [0,1]$ and
$V^{(L)}_{[L]} \in\{0,1\}$ be  a sequence of random variables such that 
\[V^{(L)}_{[L]} \ \vert \ X_{[L]} \overset{\text{i.i.d}}{\sim} \text{Ber}\left(\widehat{q}_L(X_{[L]})\right),\]
If for some $p, q \in [0,1]$,
\begin{equation}
\label{Eq:prp_A1_1_1}	\frac{1}{L}\sum_{\ell = 1}^L \varphi(X_{\ell}) \xrightarrow[L \to \infty]{a.s.} p,
\end{equation}
and
\begin{equation}
\label{Eq:prp_A1_1_2}
\widehat{q}_L(X_{[L]}) \xrightarrow[L \to \infty]{a.s.} q
\end{equation}
then
\begin{equation*}
	\frac{1}{L}\sum_{\ell = 1}^L \varphi(X_{\ell})V_{\ell}^{(L)} \xrightarrow[L \to \infty]{a.s.} pq.
\end{equation*}
\end{prp}

\begin{proof}
\begin{align*}
    \frac{1}{L}\sum_{\ell = 1}^L \varphi(X_{\ell})V_{\ell}^{(L)}
	={}
	& \left(  \frac{1}{L}\sum_{\ell = 1}^L \varphi(X_{\ell})V_{\ell}^{(L)} - 
		      \frac{1}{L}\sum_{\ell = 1}^L \varphi(X_{\ell})\widehat{q}_{L} 
      \right)
    \\
    & \quad +  \widehat{q}_{L} \frac{1}{L}\sum_{\ell = 1}^L \varphi(X_{\ell}).
\end{align*}
Due to \Cref{Eq:prp_A1_1_1,Eq:prp_A1_1_2} we immediately have 
\begin{equation*}
	\widehat{q}_{L}\left(\frac{1}{L}\sum_{\ell = 1}^L \varphi(X_{\ell})\right)\xrightarrow[L \to \infty]{a.s.} pq.
\end{equation*}
We will prove that
\begin{equation*}
	\frac{1}{L}\sum_{\ell = 1}^L \varphi(X_{\ell})V_{\ell}^{(L)} - 
			   \frac{1}{L}\sum_{\ell = 1}^L \varphi(X_{\ell})\widehat{q}_{\ell} 
	\xrightarrow[L \to \infty]{a.s.} 0.
\end{equation*}
For any fixed $x_{[L]}\in[n]$, define $g: [0,1]^L \rightarrow [0,1]$ as
\begin{equation*}
g(v_{[L]}) \defeq{} \frac{1}{L}\sum_{\ell = 1}^L \varphi(x_{\ell})v_{\ell}.
\end{equation*}
We can show that $g(\cdot)$ satisfies the so-called coordinate-wise bounded difference property. In fact, for any $i\in[L]$ and any $v_{[L]}, v_i'\in[0,1]$, since $\varphi(x)\in[0,1]$, we have
\begin{align*}
	    &\left\vert g(v_{[L]}) -  g(v_{[L]\setminus \{i\}}, v_i')\right\vert	\\
	={} &\frac{1}{L}\left\vert\varphi(x_i)\right\vert \left\vert(v_i - v_i')\right\vert \\
	\leq{} & \frac{1}{L}.
\end{align*}
Since , $V_{[L]}^{(L)}$ are i.i.d. conditioned on $X_{[L]}$ and $g(\cdot)$ is coordinate-wise bounded, we can apply McDiarmid's inequality~\cite{mcdiarmid1989method} to $V_{[L]}^{(L)}$ and $g(\cdot)$ under the conditional probability measure: $\forall \varepsilon > 0$,
\begin{equation*}
	\mathbb{P}\left[ \left\vert g(V_{[L]}^{(L)}) - \mathbb{E}\left[g(V_{[L]}^{(L)}) \Big\vert X_{[L]} \right] \right\vert \geq \varepsilon \Bigg\vert X_{[L]} \right] 
	\leq{} 2 e^{-2L\varepsilon^2}.
\end{equation*}
Since the right hand side is constant and independent of $X_{[L]}$, the above bound also holds for unconditional probability:
\begin{equation*}
	\mathbb{P}\left[ \left\vert g(V_{[L]}^{(L)}) - \mathbb{E}\left[g(V_{[L]}^{(L)}) \Big\vert X_{[L]} \right] \right\vert \geq \varepsilon \right] 
	\leq{} 2 e^{-2L\varepsilon^2}.
\end{equation*}
Since
\begin{equation*}
	g(V_{[L]}^{(L)})  ={} \frac{1}{L}\sum_{\ell = 1}^L \varphi(X_{\ell})V_{\ell}^{(L)},
\end{equation*}
we have
\begin{align*}
	\mathbb{E}\left[g(V_{[L]}^{(L)}) \Big\vert X_{[L]} \right] & ={} \mathbb{E}\left[\frac{1}{L}\sum_{\ell = 1}^L \varphi(X_{\ell})V_{\ell}^{(L)}\Big\vert X_{[L]} \right] \\
	& ={}\frac{1}{L}\sum_{\ell = 1}^L \varphi(X_{\ell})\mathbb{E}\left[V_{\ell}^{(L)}\Big\vert X_{[L]} \right] \\
	& ={}\frac{1}{L}\sum_{\ell = 1}^L \varphi(X_{\ell})\widehat{q}_{\ell}.
\end{align*}
In other words, we have shown
\begin{equation*}
	\mathbb{P}\left[ \left\vert\frac{1}{L}\sum_{\ell = 1}^L \varphi(X_{\ell})V_{\ell}^{(L)} - 
			   \frac{1}{L}\sum_{\ell = 1}^L \varphi(X_{\ell})\widehat{q}_{\ell} \right\vert \geq \varepsilon\right] 
	\leq{} 2 e^{-2L\varepsilon^2}.
\end{equation*}
Therefore,
\begin{align*}
	& \sum_{L=1}^\infty\mathbb{P}\left[ \left\vert\frac{1}{L}\sum_{\ell = 1}^L \varphi(X_{\ell})V_{\ell}^{(L)} - 
			   \frac{1}{L}\sum_{\ell = 1}^L \varphi(X_{\ell})\widehat{q}_{\ell} \right\vert \geq \varepsilon\right] \\
	\leq{} & \sum_{L=1}^\infty 2 e^{-2L\varepsilon^2} \\
	={} & \frac{2e^{-2\varepsilon^2}}{1-e^{-2\varepsilon^2}} \\
	<{} & \infty
\end{align*}
which proves that $\frac{1}{L}\sum_{\ell = 1}^L \varphi(X_{\ell})V_{\ell}^{(L)} - \frac{1}{L}\sum_{\ell = 1}^L \varphi(X_{\ell})\widehat{q}_{\ell}$ is converges to zero completely. Since complete convergence implies almost sure convergence 
\cite[Theorem~4~(c),p.310]{grimmettprobability}, it follows that
\begin{equation*}
	\frac{1}{L}\sum_{\ell = 1}^L \varphi(X_{\ell})V_{\ell}^{(L)} - 
			   \frac{1}{L}\sum_{\ell = 1}^L \varphi(X_{\ell})\widehat{q}_{\ell} 
	\xrightarrow[L \to \infty]{a.s.} 0,
\end{equation*}
which completes the proof of this proposition.
\end{proof}
\vspace{1ex}

To compute the right hand side of \Cref{Eq:Lemma1_proof_neg_lim}, for each fixed $c, v$, we apply \Cref{Prop:product_as_convergence} in the following way:

{\bf Identifying variables:}  Let $\Gamma \defeq{} \{1,\dotsc,n\} \times \{1,\dotsc, r\} \times \{1,\dotsc,\ell\}$. 
For $\gamma ={} (m,p,s)$, we define
\begin{align*}
                     L &\defeq{} \vert \Gamma \vert ={} nr\ell, \\
              X_\gamma &\defeq{} X_s^{(m,p)} \\
        V_\gamma^{(L)} &\defeq{} \mathbf{1}_{\{Z_{c}^{(m,p,s,v)} = j\}} \\
     \varphi(X_\gamma) &\defeq{} \mathbf{1}_{\{X_\gamma = i\}} \\
\widehat{q}_L(X_{[L]}) &\defeq{} \frac{1}{L}\sum_{\gamma\in\Gamma}\mathbf{1}_{\{X_\gamma = j\}}.
\end{align*}

{\bf Verification of assumptions:}
\begin{align*}
V_{[L]}^{(L)} \ \Big{\vert} \ X_{[L]} &\overset{\text{i.i.d}}{\sim} 
\text{Ber}\left(\widehat{q}_L(X_{L})\right), \\
\frac{1}{L}\sum_{\gamma\in\Gamma}\varphi(X_\gamma) &\xrightarrow[L \to \infty]{a.s.} \pi_i, \\
\widehat{q}_L(X_{[L]}) ={} &\xrightarrow[L \to \infty]{a.s.} \pi_j
\end{align*}

Therefore, by \Cref{Eq:Lemma1_proof_neg_lim}, we have
\begin{equation*}
	\frac{1}{L}\sum_{\gamma\in\Gamma} \varphi(X_\gamma) V_{[L]}^{(L)} \xrightarrow[\ell \to \infty]{a.s.} \pi_i\pi_j.
\end{equation*}
Or equivalently, for each $c, v$, 
\begin{equation*}
	\frac{1}{nr\ell}\sum_{m=1}^{n}\sum_{p=1}^{r}\sum_{s=1}^{\ell} \mathbf{1}_{\{X_s^{(m,p)} = i\}}\mathbf{1}_{\{Z_{c}^{(m,p,s,v)} = j\}} \xrightarrow[\ell \to \infty]{a.s.} \pi_i\pi_j.
\end{equation*}
Dropping a finite number of terms in the summation will not affect the limit as $\ell \rightarrow \infty$. Thus,
\begin{equation*}
	\frac{1}{nr\ell}\sum_{m=1}^{n}\sum_{p=1}^{r}\sum_{s=1}^{\ell - v} \mathbf{1}_{\{X_s^{(m,p)} = i\}}\mathbf{1}_{\{Z_{c}^{(m,p,s,v)} = j\}} \xrightarrow[\ell \to \infty]{a.s.} \pi_i\pi_j.
\end{equation*}
Together with \Cref{Eq:Lemma1_proof_neg_lim}, we have
\begin{align*}
\MoveEqLeft \frac{1}{nr}\lim_{\ell\rightarrow\infty}\frac{n_{ij}^-}{\ell} \notag
     \\
  ={}& \sum_{v=1}^w\sum_{c=1}^{k}\lim_{\ell\rightarrow\infty}\frac{1}{nr\ell}\sum_{m=1}^n\sum_{p=1}^r\sum_{s=1}^{\ell-v} \mathbf{1}_{\{X_s^{(m,p)} = i\}}\mathbf{1}_{\{Z_{c}^{(m,p,s,v)} = j\}} \\
  ={}& \sum_{v=1}^w\sum_{c=1}^{k} \pi_i\pi_j \qquad \text{a.s.} \\
  ={}& kw\pi_i\pi_j. \quad \ \ \ \qquad \text{a.s.}
\end{align*}
This concludes the proof of expressions for $\bar{n}_{ij}^+$ and $\bar{n}_{ij}^-$ in \Cref{Lemma:Ergolim} and also shows that $\bar{n}_{ij}^- = \bar{n}_{ji}^-$ .
In order to show that $\bar{n}_{ij}^+$ is symmetric, it suffices to show that for any $v$, 
\[
\pi_i(W^v)_{ij}=\pi_j(W^v)_{ji}.
\]
Since $\pi_i$ is proportional to the node degree, this is equivalent to showing that 
\[
d_i(W^v)_{ij} = d_j(W^v)_{ji}
\]
We will prove this via induction. For $v=1$, by definition, $d_iW_{ij}={}A_{ij}={}d_jW_{ji}$ (initial case). If  $d_i(W^s)_{ij} = {}d_j(W^s)_{ji}$, for $v=s+1$ (induction hypothesis), then 
\begin{align*}
d_i(W^{s+1})_{ij}=&\sum_{k=1}^nd_i(W^s)_{ik}W_{kj}\\
               =&\sum_{k=1}^nd_k(W^s)_{ki}W_{kj}\\
               =&\sum_{k=1}^nd_kW_{kj}(W^s)_{ki}\\
               =&\sum_{k=1}^nd_jW_{jk}(W^s)_{ki}\\
               =&d_j(W^{s+1})_{ji}.
\end{align*}
which proves the inductive step and concludes the proof of symmetry of $\bar{n}_{ij}^+$.
$\hfill \square$
\vspace{1ex}
 
\subsection{Proof of Theorem~\ref{Thm:Ergolim_full}} 
\label{app:ErgoLim_alpha}

We will follow the same ideas as in the proof of \Cref{Lemma:Ergolim}. With walk-distance weights $\{\alpha_v\}_{v=1}^{\infty}$, the positive pair count \Cref{Eq:Lemma1_proof_pos_pair_count} becomes:
\begin{equation}
\label{Eq:Thm3_proof_pos_pair_count}
    \frac{n_{ij}^+}{nr\ell} ={} \frac{1}{nrl}\sum_{m=1}^n\sum_{p=1}^r\sum_{v=1}^\infty\alpha_v\sum_{s=1}^{\ell-v} \mathbf{1}_{\{X_s^{(m,p)} = i, X_{s+v}^{(m,p)} = j\}}.
\end{equation}
And the negative pair count \Cref{Eq:Lemma1_proof_neg_pair_count} becomes:
\begin{equation}
\label{Eq:Thm3_proof_neg_pair_count}
    \frac{n_{ij}^-}{nr\ell} ={} \! \frac{1}{nr\ell}\!\sum_{m=1}^n\sum_{p=1}^r\sum_{v=1}^\infty\alpha_v\sum_{s=1}^{\ell-v} \mathbf{1}_{\{X_s^{(m,p)} = i\}}\sum_{c=1}^{k}\mathbf{1}_{\{Z_{c}^{(m,p,s,v)} \! = j\}}.
\end{equation}
This provides the starting point for our proof.

{\noindent \bf 1)}
The proof closely parallels the proof of \Cref{Lemma:Ergolim} with minor modifications to account for the walk-distance weighting. We note that the exchange of the limit and the infinite sum is ensured by the dominated convergence theorem.

{\noindent \bf 2)} From \eqref{Eq:Thm3_proof_pos_pair_count}, we have
\begin{align*}
\MoveEqLeft \frac{1}{\ell n}\lim_{r\rightarrow\infty}\frac{n_{ij}^+}{\ell}
        \\
     ={}& \frac{1}{rn}\sum_{m=1}^n
	    \sum_{v=1}^{\infty}\alpha_v
	    \sum_{s=1}^{\ell-v} 
	    \lim_{r\rightarrow\infty}\frac{1}{r}\sum_{p=1}^r
	    \mathbf{1}_{\{X_s^{(m,p)} = i, X_{s+v}^{(m,p)} = j\}} 
	    \\
    ={}&\frac{1}{rn}\sum_{m=1}^n
	    \sum_{v=1}^{\infty}\alpha_v \\
	   & \quad \left(\sum_{s=1}^{\ell-v} 
	    \lim_{r\rightarrow\infty}\frac{1}{r}\sum_{p=1}^r
	    \mathbf{1}_{\{X_s^{(m,p)} = i\}}\mathbf{1}_{\{X_{s+v}^{(m,p)} = j\ |\ X_s^{(m,p)} = i\}}\right).
\end{align*}
Note that
\begin{align*}
%
 & \!\!\! \lim_{r\rightarrow\infty}\frac{1}{r}\sum_{p=1}^r	\mathbf{1}_{\{X_s^{(m,p)} = i\}}\mathbf{1}_{\{X_{s+v}^{(m,p)} = j\ |\ X_s^{(m,p)} = i\}}	    \\
	  &={}\lim_{r\rightarrow\infty} \frac{1}{r}\sum_{p=1}^r
	    \mathbf{1}_{\{X_s^{(m,p)} = i\ |\ X_{1}^{(m,p)} = m\}}\mathbf{1}_{\{X_{s+v}^{(m,p)} = j\ |\ X_s^{(m,p)} = i\}} 
	    \\
	  &={} \mathbb{E}\mathbf{1}_{\{X_s^{(m,p)} = i\ |\ X_{1}^{(m,p)} = m\}}\mathbf{1}_{\{X_{s+v}^{(m,p)} = j\ |\ X_s^{(m,p)} = i\}}
	     \\
	  &={}\mathbb{E}\mathbf{1}_{\{X_s^{(m,p)} = i\ |\ X_{1}^{(m,p)} = m\}}\mathbb{E}\mathbf{1}_{\{X_{s+v}^{(m,p)} = j\ |\ X_s^{(m,p)} = i\}}
	     \\
	  &={} (W^{s-1})_{mi}(W^v)_{ij}
\end{align*}

Therefore, we have
\begin{align*} 
	\frac{1}{\ell n}\lim\limits_{r\rightarrow\infty}\frac{n_{ij}^+}{r}={}& \frac{1}{\ell n}\sum_{k=1}^n \sum_{v=1}^w \alpha_v(W^v)_{ij}\sum_{s=1}^{l-v}(W)^{s-1}_{mi},
\end{align*}
where the exchange of the limit and infinite sum is allowed by the dominated convergence theorem.

For the negative terms, from \eqref{Eq:Thm3_proof_neg_pair_count}, we have
\begin{align*}
\MoveEqLeft \frac{1}{\ell n}\lim_{r\rightarrow\infty}\frac{1}{r}n_{ij}^-
	={} \frac{1}{\ell n}\sum_{m=1}^n\sum_{v=1}^\infty\alpha_v\\
	& \left(\sum_{s=1}^{\ell-v}\sum_{c=1}^{k}
		\lim_{r\rightarrow\infty}\frac{1}{r}\sum_{p=1}^r \mathbf{1}_{\{X_s^{(m,p)} = i\}}\mathbf{1}_{\{Z_{c}^{(m,p,s,v)} = j\}}\right)
\end{align*}
Proceeding as we did in the proof of Lemma~\ref{Lemma:Ergolim},, we apply \Cref{Prop:product_as_convergence} to obtain 
\begin{align*}
	   \MoveEqLeft
	   \lim_{r\rightarrow\infty}\frac{1}{r}\sum_{p=1}^r \mathbf{1}_{\{X_s^{(m,p)} = i\}}\mathbf{1}_{\{Z_{c}^{(m,p,s,v)} = j\}} \\
	={}& (W)^{s-1}_{mi}\left(\frac{1}{\ell}\sum_{u=1}^{\ell} \frac{1}{n}\bm{1}_n^\top W^{u-1}\bm{e}_j\right) \\
	\defeq{}& (W)^{s-1}_{mi}\pi^{(\ell)}_j
\end{align*}
Therefore, 
\begin{equation*}
		\frac{1}{\ell n}\lim\limits_{r\rightarrow\infty}\frac{n_{ij}^-}{r}={}         		 \frac{k\pi^{(\ell)}_j}{\ell n}\sum_{m=1}^n \sum_{v=1}^{\infty} \alpha_v\sum_{s=1}^{\ell-v}(W)^{s-1}_{mi}
\end{equation*}

{\noindent \bf 3)} From 1), we know that $\frac{1}{rn}\lim_{\ell\rightarrow\infty}\frac{n_{ij}^+}{\ell}$ and $\frac{1}{rn}\lim_{\ell\rightarrow\infty}\frac{n_{ij}^-}{\ell}$does not depand on $r$, and therefore the first part of the equality holds.

An irreducible Markov chain on a finite state space with a time-homogeneous transition matrix $W$ has a unique stationary distribution $\boldsymbol{\pi}$. Moreover, for any initial distribution on states $\boldsymbol{\pi}_0$, the Cesaro-average: 
$\bar{\boldsymbol{\pi}}_\ell := \boldsymbol{\pi}_0^\top\,  \frac{1}{\ell} \sum_{s = 1}^{\ell}  W^s $, $\ell = 1, 2, \ldots$,
converges to the unique stationary distribution $\boldsymbol{\pi}$ (even if the Markov chain is not aperiodic). While this is a somewhat well-known result, we were unable to find a reliable reference that explicitly states or proves it. So for completeness we briefly sketch its proof. 
We argue that $\bar{\boldsymbol{\pi}}_\ell$ must converge to $\boldsymbol{\pi}$. If not, there is an $\epsilon > 0$ and a subsequence that lies strictly outside an $\epsilon$-ball around $\boldsymbol{\pi}$. But, the probability simplex in finite-dimensional Euclidean space is compact and has the Bolzano-Weierstrass property: there is a subsequence of the subsequence (a sub-subsequence) which converges. Below we will show that the limit of this sub-subsequence must be $\boldsymbol{\pi}$ which would result in a contradiction (since the subsequence is outside an $\epsilon$ ball around $\boldsymbol{\pi}$). Therefore, $\bar{\boldsymbol{\pi}}_\ell$ must converge to the unique stationary distribution $\boldsymbol{\pi}$. We will now show that any convergent subsequence of $\bar{\boldsymbol{\pi}}_\ell$ (a convergent sub-subsequence is also convergent subsequence) must converge to $\boldsymbol{\pi}$. Let $\bar{\boldsymbol{\pi}}_{\ell_t}$ denote a convergent subsequence and $\bar{\boldsymbol{\pi}}_\infty$ its limit. Then,
\begin{flalign*}
\bar{\boldsymbol{\pi}}_\infty\,W  
&= \lim_{t \rightarrow \infty} \big( \bar{\boldsymbol{\pi}}_{\ell_t}\, W \big) 
= \lim_{t \rightarrow \infty}  \Big( \boldsymbol{\pi}_0^\top\,  \frac{1}{\ell_t} \sum_{s = 1}^{\ell_t}  W^s \cdot W \Big) \\
&= \lim_{t \rightarrow \infty}  \Big( \frac{\ell_t + 1}{\ell_t}\,\bar{\boldsymbol{\pi}}_{\ell_t + 1} - \frac{1}{\ell_t}\boldsymbol{\pi}_0^\top\,W \Big) \\
&= \bar{\boldsymbol{\pi}}_\infty
\end{flalign*}
where in the first equality we made use of the fact that linear maps between finite-dimensional Euclidean spaces are continuous. Thus, $\bar{\boldsymbol{\pi}}_\infty$ is a stationary distribution of $W$ since the above analysis shows that  $\bar{\boldsymbol{\pi}}_\infty\,W = \bar{\boldsymbol{\pi}}_\infty$. Since $W$ has a unique stationary distribution $\boldsymbol{\pi}$, we have $\bar{\boldsymbol{\pi}}_\infty = \boldsymbol{\pi}$. 
We reiterate that aperiodicity is not needed. This is important since it is not guaranteed that the connected component subgraphs of a given graph will be aperiodic.

The following results follow immediately:
	\begin{equation*}
		\lim\limits_{\ell\rightarrow\infty}\frac{1}{\ell}\sum_{s=1}^{l-v}(W)^s_{ki} = \pi_i.
	\end{equation*}
	and
	\begin{equation*}
		\lim\limits_{\ell\rightarrow\infty}\pi^{(\ell)}_j = \lim\limits_{\ell\rightarrow\infty}\frac{1}{\ell}\sum_{u=1}^{\ell} \frac{1}{n}\bm{1}_n^\top W^{u-1}\bm{e}_j = \pi_j.
	\end{equation*}
	
	And therefore,
    \begin{align*} 
		\frac{1}{n}\lim\limits_{\ell\rightarrow\infty}\lim\limits_{r\rightarrow\infty}\frac{n_{ij}^+}{r\ell}={}& \pi_i\sum_{v=1}^{\infty}\alpha_v(W^v)_{ij},\\
		\frac{1}{n}\lim\limits_{\ell\rightarrow\infty}\lim\limits_{r\rightarrow\infty}\frac{n_{ij}^-}{r\ell}={}& k\pi_i\pi_j\sum_{v=1}^{\infty}\alpha_v.
	\end{align*}
$\hfill \square$

{\noindent \bf{Remark}}: The characterization of ergodic limits of walk-distance weighted counts stated in \Cref{Thm:Ergolim_full} are for connected graphs. For disconnected graphs the characterization is similar, but confined to each connected component, as in \Cref{Thm:Ergolim}, and can be proved similarly.

\subsection{Proof of Proposition~\ref{Prop:PMI_solution}}
\label{app:PMI_solution}
\begin{proof}
\Cref{Eq:PMI} is separable with respect to the $X_{ij}$ variables, and for each $X_{ij}$, the problem reduces to the following univariate optimization problem:
\begin{equation} 
\label{Eq:PMI_single}
\argmin_{x\in\mathbb{R}}\ f_{ij}(x)\defeq \bar{n}_{ij}^+\ln\left(1+e^{-x}\right)+\bar{n}_{ij}^-\ln\left(1+e^{+x}\right).
\end{equation}
Since 
\begin{equation*}
	\frac{\text{d}^2f_{ij}}{\text{d}x^2}=\bar{n}_{ij}^+\frac{e^{-x}}{(1+e^{-x})^2}+\bar{n}_{ij}^-\frac{e^{x}}{(1+e^{x})^2}>0,
\end{equation*}
it follows that $f_{ij}$ is a twice-differentiable convex function and therefore attains a global minimum at values of $x$ where the derivative vanishes, i.e,
\begin{equation*}
	\frac{\text{d}f_{ij}}{\text{d}x}=-\bar{n}_{ij}^+\frac{e^{-x}}{1+e^{-x}}+\bar{n}_{ij}^-\frac{e^{x}}{1+e^{x}} = 0,
\end{equation*}
or equivalently
\begin{equation*}
	\bar{n}_{ij}^-e^{2x}+(\bar{n}_{ij}^+-\bar{n}_{ij}^-)e^x-\bar{n}_{ij}^+=0.
\end{equation*}
Note that from \Cref{Eq:Thm_Ergolim-} of \Cref{Thm:Ergolim} we know $\bar{n}_{ij}^- > 0$. Therefore, when $\bar{n}_{ij}^+\neq 0$, we have a unique solution $e^x=\frac{\bar{n}_{ij}^+}{\bar{n}_{ij}^-}$, i.e., $x=\ln\left(\frac{\bar{n}_{ij}^+}{\bar{n}_{ij}^-}\right)$. When $\bar{n}_{ij}^+\neq 0$, $f_{ij}(x)$ is monotonically increasing over the entire real line, and we take $x = -\infty$ as the solution.
Thus, 
\begin{equation*}
X_{ij}^* = \begin{cases}
	\ln\left(\frac{\bar{n}_{ij}^+}{\bar{n}_{ij}^-}\right) & \text{if } \bar{n}_{ij}^+\neq 0;\\
	-\infty & \text{if } \bar{n}_{ij}^+= 0.
\end{cases}	
\end{equation*}
From \Cref{Lemma:Ergolim}, we have $\bar{n}_{ij}^+ = \bar{n}_{ji}^+$ and $\bar{n}_{ij}^- = \bar{n}_{ji}^-$. Therefore we have $X_{ij}^* = X_{ji}^*$.

For $rn$ random walks each of length $\ell$, the total number of node pairs within $w$ steps of each other is $\vert\mathcal{D}_{\ell,+}\vert=rn\left(\ell w-\frac{w(w+1)}{2}\right)$. First note that, when $\bar{n}_{ij}^+ = 0$, $(i,j)$ are in two different connected components, and thus $n_{ij}^+ = 0$, or equivalently, $p_{\ell}(i,j) = 0$. Therefore, $\text{PMI}_{\ell}(i,j)=-\infty = X^*_{ij} - \ln{k}$ holds. For the rest of the proof, we only consider the case when $\bar{n}_{ij}^+ \neq 0$ or, equivalently, when $(i,j)$ are in the same connected component. For the joint distribution, we have
\begin{equation*}
	p_{\ell}(i,j)={}\frac{n_{ij}^+}{\vert\mathcal{D}_{\ell,+}\vert}\xrightarrow[Eq.~\eqref{Eq:Def_Ergolim+} ]{\ell \to \infty \ a.s.} \frac{\bar{n}_{ij}^+}{w}.
\end{equation*}
For the marginal distributions,
\begin{equation*}
	p_{\ell 1}(i)={}\sum_{j\in\mathcal{V}}p_{\ell}(i,j) = {} \frac{\sum_{j\in\mathcal{V}} n_{ij}^+}{\vert\mathcal{D}_{\ell,+}\vert}\xrightarrow[\text{Eq.}~\eqref{Eq:Def_Ergolim+} ]{\ell \to \infty \ a.s.} \frac{\sum_{j\in\mathcal{V}} \bar{n}_{ij}^+}{w} \stackrel{\text{Eq.}~\!\eqref{Eq:Thm_Ergolim+}}{=\joinrel=\joinrel=}{}\pi_i
\end{equation*}
and
\begin{equation*}
	p_{\ell 2}(j)={}\sum_{i\in\mathcal{V}}p_{\ell}(i,j) = {} \frac{\sum_{i\in\mathcal{V}} n_{ij}^+}{\vert\mathcal{D}_{\ell,+}\vert}\xrightarrow[\text{Eq.}~\eqref{Eq:Def_Ergolim+} ]{\ell \to \infty \ a.s.} \frac{\sum_{i\in\mathcal{V}} \bar{n}_{ij}^+}{w}. 
\end{equation*}
Note that since $\bm{\pi}$ is the stationary distribution of the random walk, for any $v$, $\sum_{\in\mathcal{V}} \pi_i (W^v)_{ij} = \pi_j$. Combining this with Eq.~\eqref{Eq:Thm_Ergolim+}, we get
\begin{equation*}
	\sum_{i\in\mathcal{V}}\bar{n}_{ij}^+={} \sum_{v=1}^{w}\sum_{i\in\mathcal{V}}\pi_i(W^v)_{ij} ={}\sum_{v=1}^{w}\pi_j ={}w\pi_j.
\end{equation*}
Therefore, $p_{\ell 2}(j)\xrightarrow[a.s.]{\ell \to \infty}\pi_j$ and
\begin{align*}
	\text{PMI}_{\ell}(i,j) =\joinrel={}&\ln\left(\frac{p_{\ell}(i,j)}{p_{\ell 1}(i)p_{\ell 2}(j)}\right) \\
	 \xrightarrow[a.s.]{\ell \to \infty}{}&\ln\left(\frac{\bar{n}_{ij}^+}{w\pi_i\pi_j}\right) \\
	 \stackrel{\text{Eq.}\eqref{Eq:Thm_Ergolim-}}{=\joinrel=\joinrel=}{}& \ln\left(\frac{\bar{n}_{ij}^+}{\bar{n}_{ij}^-}\right) + \ln k \\
	 =\joinrel={}& X_{ij}^* + \ln{k}
\end{align*}
\end{proof}

\subsection{Proof of Theorem~\ref{Thm:expected_solution}}
\label{app:proof_expected_solution}

{\noindent}The main structure of the proof is as follows:
\begin{enumerate}
\item Part 1 will be shown using \Cref{Lemma:Ergolim} and the eigenvalue decomposition of the {\it diagonal-blockwise-constant} (DBC) matrices (defined below).

\item Part 2 is a direct consequence of combining Part 1) and \Cref{Prop:PMI_solution}.

\item Part 3 is intricate and will be proved in 3-steps: 

	  {\bf 1) }First, we will show that the solution must be a DBC matrix, and thus can be re-parameterized by the three scalars that define a DBC matrix. This will be established by showing that for any feasible solution, a DBC matrix can be constructed that is both feasible and yields a lower objective cost.

      {\bf 2) }Second, we will prove that among the 3 scalar variables in the re-parameterized problem, the optimal value of two of them must equal. This implies that the matrix solution must have a block structure. We will then eliminate one variable and re-parameterize the optimization problem in terms of the remaining two variables. 

      {\bf 3) }Lastly, we will show that the optimal values of the two variables will be opposite numbers of each other which will imply that the solution matrix has rank $1$.
\item Part 4 is a direct consequence of parts (1)---(3).
\end{enumerate}

Before getting into the derivations, we set up some notation and define {\it diagonal-blockwise-constant} (DBC) matrices.

Without loss of generality, we assume that nodes in the two balanced communities are $\{1, \ldots, m\}$ and $\{m+1, \ldots, 2m\}$. Under this labeling, we define the following two subsets of node pairs (edges)
 \begin{align*}
 	\MoveEqLeft[30]
 	\mathcal{E}_0 \defeq{} \{(i,j): i\neq j, \ i, j\leq m \text{ or } i,j\geq m+1\}\\
 	\MoveEqLeft[30]
 	\mathcal{E}_1 \defeq{} \{(i,j): i\leq m,\ j\geq m+1\} \\
 	\MoveEqLeft[17]
    \hfill \bigcup \{(i,j): i\geq m+1,\ j\leq m\}.
 \end{align*}
Then, $\vert \mathcal{E}_0 \vert = 2m(m-1)$ and $\vert \mathcal{E}_1 \vert = 2m^2$.
\vspace{\baselineskip}

{\noindent}Next, we define {\it diagonal-blockwise-constant} (DBC) matrices.

\begin{defi}[DBC matrix]
	\label{Defi:DBC matrix}
	Let $\bm{1}_m \defeq (1,1,\ldots,1)^\top\in\mathbb{R}^m$. Let $\bm{y}_1 = (\bm{1}_m^\top, \bm{1}_m^\top)^\top$ and $\bm{y}_2 =(\bm{1}_m^\top, -\bm{1}_m^\top)^\top$. For $m\geq 2$, a $2m \times 2m$ matrix is called {\it diagonal-blockwise-constant} (DBC) if it has the form 
	\begin{equation}
	\label{Eq:Def_DBC}
		Z_{2m}(c_1, c_2, c_3) \defeq \frac{c_1+c_2}{2}\bm{y}_1\bm{y}_1^\top + \frac{c_1-c_2}{2}\bm{y}_2\bm{y}_2^\top + (c_3-c_1) I_{2m}.
	\end{equation}
\end{defi}

{\noindent} Certain key properties of DBC matrices that we use in our proof are described in the following proposition.

\begin{prp}[Properties of DBC matrices]
\label{Prop:property_DBC}
Let $\mathcal{E}_0$ and $\mathcal{E}_1$ be as stated above and let $X$ be a $2m\times 2m$ matrix for $m\geq 2$. Then, $X = Z_{2m}(c_1, c_2, c_3)$ if, and only if, any one of the following holds:
\begin{enumerate}
	\item $X$ has the following block structure:		
		\begin{equation*}
		X_{ij} = \begin{cases}
		c_1 & \text{if }(i,j)\in\mathcal{E}_0; \\
		c_2 & \text{if }(i,j)\in\mathcal{E}_1; \\
		c_3 & \text{if } i = j
		\end{cases}
		\end{equation*}
	\item The eigenvalues and eigenvectors of $X$ satisfy:
		\begin{enumerate}
		\item  $\lambda_3=\lambda_4 = \ldots = \lambda_{2m}$;
		\item  $\bm{u}_1 = \frac{1}{\sqrt{2m}}\bm{y}_1$, $\bm{u}_2 = \frac{1}{\sqrt{2m}}\bm{y}_2$.
 		\end{enumerate}
\end{enumerate}

{\noindent}In addition, the set of all $DBC$ matrices is closed under matrix addition and multiplication operations.
\end{prp}

\begin{proof}\hfill

{\noindent {\bf Proof of equivalence.}}

{\noindent {\bf1)}} Both if and only if parts can be obtained directly from \Cref{Eq:Def_DBC} in \Cref{Defi:DBC matrix}:
\begin{equation*}
		X_{ij} ={} \begin{cases}
		\frac{c_1+c_2}{2} + \frac{c_1-c_2}{2} = c_1 & \text{if }(i,j)\in\mathcal{E}_0; \\
		\frac{c_1+c_2}{2} - \frac{c_1-c_2}{2} = c_2 & \text{if }(i,j)\in\mathcal{E}_1; \\
		\frac{c_1+c_2}{2} + \frac{c_1-c_2}{2} + (c_3-c_1)= c_3 & \text{if } i = j
		\end{cases}
\end{equation*}

{\noindent {\bf2)}} If $X=Z_{2m}(c_1, c_2, c_3)$, directly from \cref{Eq:Def_DBC}, we can compute the spectral decomposition of $X$. Let $\bm{u}_1 = \frac{1}{\sqrt{2m}}\bm{y}_1$, $\bm{u}_2 = \frac{1}{\sqrt{2m}}\bm{y}_2$ and $\bm{u}_3,\ldots,\bm{u}_{2m}$ be any set of orthonormal vectors that together with $\bm{u}_1$ and $\bm{u}_2$ form an orthonormal basis for $\mathbb{R}^{2m}$. Then,
\begin{align*}
	X ={}& m(c_1+c_2)\bm{u}_1\bm{u}_1^\top + m(c_1-c_2)\bm{u}_2\bm{u}_2^\top + (c_3-c_1) \sum_{i=1}^{2m}\bm{u}_i\bm{u}_i^\top \\
	  ={}& \left(m(c_1+c_2) + (c_3-c_1) \right)\bm{u}_1\bm{u}_1^\top +  \\
	   & \left(m(c_1-c_2) + (c_3-c_1) \right)\bm{u}_2\bm{u}_2^\top + \sum_{i=3}^{2m}(c_3-c_1)\bm{u}_i\bm{u}_i^\top.
\end{align*}
Therefore, $\bm{u}_1,\ldots,\bm{u}_{2m}$ are the eigenvectors of $X$ and the eigenvalues satisfy $\lambda_3=\lambda_4 = \ldots =\lambda_{2m} = c_3 - c_1$.

Reversely, if the eigenvalues and eigenvectors of $X$ have the given property, letting $U = [\bm{u}_1, \ldots, \bm{u}_{2m}]$, we have
	\begin{align*}
	X 	 & ={}  U \ \text{Diag}\{\lambda_1, \lambda_2, \lambda_3, \ldots, \lambda_3\}
		 U^\top  
		\\
	 & ={}U \ \text{Diag}\{\lambda_1 - \lambda_3, \lambda_2 - \lambda_3, 0, \ldots, 0\}
		 U^\top 
		 + U \lambda_3I_{2m} U^\top 
		\\
	 & ={} \frac{\lambda_1 - \lambda_3}{2m}\bm{y}_1\bm{y}_1^\top + \frac{\lambda_2 - \lambda_3}{2m}\bm{y}_2\bm{y}_2^\top + \lambda_3 I_{2m} 
		\\
     & ={} Z_{2m}\left(\frac{\lambda_1 + \lambda_2 - 2 \lambda_3}{2m}, 
	                  \frac{\lambda_1 - \lambda_2}{2m}, 
	                  \frac{\lambda_1 + \lambda_2 + (2m -2)\lambda_3}{2m}\right)
   \end{align*}
   
{\noindent {\bf Proof of set closure}}

Let $X_1$, $X_2$ be two DBC matrices. From part 2), defining $U = [\bm{u}_1, \ldots, \bm{u}_{2m}]$ where $\bm{u}_1 = \frac{1}{\sqrt{2m}}\bm{y}_1$, $\bm{u}_2 = \frac{1}{\sqrt{2m}}\bm{y}_2$ and $\{\bm{u}_3,\ldots,\bm{u}_{2m}\}$ is any set of orthonormal vectors that together with $\bm{u}_1$ and $\bm{u}_2$ form an orthonormal basis for $\mathbb{R}^{2m}$, we have 
\begin{align*}
		X_1 &={}U \ \text{Diag}\{\lambda_1, \lambda_2, \lambda_3, \ldots, \lambda_3\}
		 U^\top  \\
		 X_2 &={}U \ \text{Diag}\{\mu_1, \mu_2, \mu_3, \ldots, \mu_3\}
		 U^\top.
\end{align*}

Therefore, 
\begin{align*}
		X_1 + X_2 &={}U \ \text{Diag}\{\lambda_1 + \mu_1, \lambda_2 + \mu_2, \lambda_3 + \mu_3, \ldots, \lambda_3 + \mu_3\}
		 U^\top  \\
		 X_1X_2 &={}U \ \text{Diag}\{\lambda_1\mu_1, \lambda_2\mu_2, \lambda_3\mu_3, \ldots, \lambda_3\mu_3\}
		 U^\top
\end{align*}
satisfy the conditions a) and b) in part 2), and they are both DBC matrices.
\end{proof}

{\noindent \bf{Notation.}} For ease of reference, for a DBC matrix $X$, we denote $\lambda_i(X)$ as its eigenvalues and $c_i(X)$ $(i = 1,2,3)$ as its entry values in $\mathcal{E}_0$, $\mathcal{E}_1$ and diagonal, respectively. I.e., $X = Z_{2m}(c_1(X),c_2(X),c_3(X))$. The derivation in the proof above gives the transformation formula between them. Specifically, given $X = Z_{2m}(c_1,c_2,c_3)$, we have
	\begin{align}
		\lambda_1(X) & ={} (m - 1)c_1 + c_3 + m c_2 , \label{Eq:Thm2_proof_eig1} \\
		\lambda_2(X) & ={} (m - 1)c_1 + c_3 - m c_2 , \label{Eq:Thm2_proof_eig2} \\
		\lambda_3(X) & ={} c_3-c_1. \label{Eq:Thm2_proof_eig3}
	\end{align}
	And given the eigenvalues $\lambda_1,\lambda_2,\lambda_3=\ldots=\lambda_{2m}$ of $X$, we have
	\begin{align}
		c_1(X) & ={} \frac{\lambda_1 + \lambda_2 - 2 \lambda_3}{2m} \label{Eq:Thm2_proof_c1}\\
		c_2(X) & ={} \frac{\lambda_1 - \lambda_2}{2m} \label{Eq:Thm2_proof_c2} \\
		c_3(X) & ={} \frac{\lambda_1 + \lambda_2 + (2m -2)\lambda_3}{2m}. \label{Eq:Thm2_proof_c3}
	\end{align}
	
\begin{prp}[P.S.D. condition of DBC matrices]
\label{Prop:PSD_DBC}
Let $X={} Z_{2m}(c_1, c_2, c_3)$ be a DBC matrix with $c_3 \geq c_1$. Denote $\bar{c}_{13}\defeq{}\frac{m-1}{m}c_1 + \frac{1}{m}c_3$ and let $Y ={} Z_{2m}(\bar{c}_{13}, c_2, \bar{c}_{13})$.
Then, if $Y \succeq 0$, we have $X \succeq 0$.	
\end{prp}

\begin{proof}\hfill

By \Cref{Eq:Thm2_proof_eig1,Eq:Thm2_proof_eig2}, we have
	\begin{align*}
		\lambda_1(Y) & ={} (m - 1)\bar{c}_{13} + \bar{c}_{13} + m c_2	\\
					 & ={} (m-1)c_1 + c_3 + m c_2 \\
					 & ={} \lambda_1(X)
	\end{align*}
	\begin{align*}
		\lambda_2(Y) & ={} (m - 1)\bar{c}_{13} + \bar{c}_{13} - m c_2 	\\
					 & ={} (m-1)c_1 + c_3 + m c_2 \\
					 & ={} \lambda_2(X)
	\end{align*}
	Since $Y \succeq 0$, we have $\lambda_1(X) \geq 0$ and $\lambda_2(X) \geq 0$. Since $c_3 \geq c_1$, we have $\lambda_3(X) = c_3 - c_1 \geq 0$. And therefore, $X \succeq 0$.
\end{proof}

Now, we are ready to prove \Cref{Thm:expected_solution}.

{\noindent{\bf Part 1)}}

Note that for expected graph, the adjacency matrix $A$ and random walk transition matrix $W$ are both DBC matrices, and the stationary distribution $\bm{\pi}$ is uniform distribution. Specifically, we have 
	\begin{align*}
		A & ={} Z_{2m}(a, b, 0) \\
		W & ={} Z_{2m}\left(\frac{a}{(m-1)a + mb}, \frac{a}{(m-1)a + mb}, 0\right) \\
		\bm{\pi} & = \frac{1}{2m}\bm{1}_{2m}
	\end{align*}
By \Cref{Lemma:Ergolim}, we can compute the positive and negative coefficient matrices $\bar{N}^+$ and $\bar{N}^-$ as
	\begin{align}
		\bar{N}^+ & ={} \frac{1}{2m} \sum_{v=1}^w W^v, \label{Eq:Thm2_proof_N+} \\
		\bar{N}^- & ={} kw\bm{\pi}\bm{\pi}^\top 
						= \frac{kw}{4m^2}\bm{1}_{2m}\bm{1}_{2m}^\top.
					\label{Eq:Thm2_proof_N-}
	\end{align}
\Cref{Eq:Thm2_proof_N-} gives us $\bar{n}_{ij}^- = \frac{kw}{4m^2} = \frac{kw}{n^2} \eqdef \beta$. It remains to show $\bar{n}_{ij}^+$.
		  
Since $\bar{N}^+$ is a sum of products of DBC matrices, by closure of DBC set (\Cref{Prop:property_DBC}), $\bar{N}^+$ is a DBC matrix. In order to compute $c_1(\bar{N}^+)$, $c_2(\bar{N}^+)$, and $c_3(\bar{N}^+)$, we begin from its eigenvalues. Since $W$ is a DBC matrix, by \Cref{Eq:Thm2_proof_eig1,Eq:Thm2_proof_eig2,Eq:Thm2_proof_eig3}, we have	  
\begin{align*}
	\lambda_1(W) &={} 1,\\
	\lambda_2(W) &={} \frac{(m-1)a - mb}{(m-1)a + mb}, \\
	\lambda_3(W) &={} \ldots ={} \lambda_{2m}(W)={} -\frac{a}{(m-1)a + mb}.
\end{align*}
Note that since we assumed $a > \frac{m}{m-1}b$, we have $\lambda_1(W) > 0$, $\lambda_2(W)>0$, $\lambda_3(W)<0$.

From \Cref{Eq:Thm2_proof_N+}, we obtain the eigenvalues of $\bar{N}^+$
\begin{align*}
	\lambda_1(\bar{N}^+) &={} \frac{w}{2m},\\
	\lambda_2(\bar{N}^+) &={} \frac{1}{2m}\sum_{v=1}^w \lambda_2(W)^v, \\
	\lambda_3(\bar{N}^+) &={} \ldots ={} \lambda_{2m}(W)={} \frac{1}{2m}\sum_{v=1}^w \lambda_3(W)^v.
\end{align*}
Given the sign of $\lambda_i(W)$, we have $\lambda_3(\bar{N}^+)< 0 < \lambda_2(\bar{N}^+) < \lambda_1(\bar{N}^+)$. With \Cref{Eq:Thm2_proof_c1,Eq:Thm2_proof_c2,Eq:Thm2_proof_c3}, the entry values are given as
\begin{align}
c_1(\bar{N}^+) &={} \frac{1}{4m^2}\left[w + \sum_{v=1}^w \lambda_2(W)^v - 2\sum_{v=1}^w \lambda_3(W)^v \right]     \label{alpha_1}
			   \\
			   &\defeq{} \alpha_1, \notag  \\
c_2(\bar{N}^+) &={} \frac{1}{4m^2}\left[w - \sum_{v=1}^w \lambda_2(W)^v\right]    \label{alpha_2}
			   \\
			   &\defeq{} \alpha_2, \notag  \\
c_3(\bar{N}^+) &={} \frac{1}{4m^2}\left[w + \sum_{v=1}^w \lambda_2(W)^v  + (2m - 2)\sum_{v=1}^w \lambda_3(W)^v\right] 
     		\label{alpha_3} \\
     		   &\defeq{} \alpha_3. \notag
\end{align}
This completes the proof of Part 1). 
Note that $\lambda_1(W) = 1$, $\lambda_2(W) = 1 - O(1/n)$, $\lambda_3(W)= O(1/n)$. Therefore, from \Cref{alpha_1,alpha_2,alpha_3}, we have $\alpha_i=C_i/n^2+o(1/n^2)$ for $i=1,2,3$, where $C_i$'s are functions of only $a$, $b$ and $w$. Given the sign of $\lambda_i(W)$, we have
\begin{align}
	\alpha_1 &> \alpha_3 > 0, \label{Eq:Thm2_proof_alpha_relation} \\
	\alpha_1 &> \alpha_2 > 0. \notag
\end{align}
\vspace{\baselineskip}

{\noindent{\bf Part 2)}}

Applying \Cref{Prop:PMI_solution}, since $\bar{n}_{ij}^+>0$ and $\bar{n}_{ij}^->0$ hold for all $i,j$, we have
\begin{equation*}
	X_{ij}^* = \ln\left(\frac{\bar{n}_{ij}^+}{\bar{n}_{ij}^-}\right) =  \begin{cases}
							\ln\left(\frac{\alpha_1}{\beta}\right), &\text{if }(i,j)\in\mathcal{E}_0 \\
							\ln\left(\frac{\alpha_2}{\beta}\right), &\text{if }(i,j)\in\mathcal{E}_1 \\
						\ln\left(\frac{\alpha_3}{\beta}\right), &\text{if } i = j
						 								\end{cases}
\end{equation*}	

{\noindent{\bf Part 3)}}

When $\mathcal{H}=\{X \ | \ X \succeq 0\}$, we first establish structures that $X^*$ must have, and then solve it explicitly. We take three major steps:
\begin{enumerate}
\item[Step 1] We show that $X^*$ must be a DBC matrix, and thus we can re-parameterize the optimization problem into three scalars variables: $c_1$, $c_2$, and $c_3$. 
\item[Step 2] We prove that among the optimal solution of this re-parameterized problem must satisfy $c_1^* = c_3^*$. Then, we substitute $c_3$ by $c_1$ and only keep $c_1$ and $c_2$ as optimizing variables.
\item[Step 3] We show that $c_1^* = -c_2^*$ must hold. After eliminating $c_2$, we solve the optimization explicitly.
\end{enumerate}

{\noindent{\bf Step 1}}.

For any matrix $X\in S_+$, let $c_1$, $c_2$, and $c_3$ be the average of its entries in region $\mathcal{E}_0$, $\mathcal{E}_1$ and on diagonal, respectively. I.e.,
\begin{align}
	x_1 &\defeq{} \frac{1}{2m^2-2m}\sum_{(i,j)\in\mathcal{E}_0}X_{ij},  \label{Eq:Thm2_proof_x_tilde_c1}       \\
	x_2 &\defeq{} \phantom{-2m^2}\frac{1}{2m^2}\sum_{(i,j)\in\mathcal{E}_1}X_{ij}, \label{Eq:Thm2_proof_x_tilde_c2}\\
	x_3 &\defeq{} \phantom{-2m^2}\frac{1}{2m}\phantom{=}\sum_{i=1}^{2m}X_{ii}. \label{Eq:Thm2_proof_x_tilde_c3}
\end{align}
Then, we construct a DBC matrix $\tilde{X}$ as 
\begin{equation*}
	\tilde{X} = Z_{2m}(x_1, x_2, x_3).
\end{equation*}
Denoting our objective function in \Cref{Eq:Thm2_optimization} as $f$, i.e.,
\begin{equation*}
	f(X) \defeq{} \sum_{(i,j)} \left[\bar{n}_{ij}^+\ln\left(1+e^{-X_{ij}}\right)+\bar{n}_{ij}^-\ln\left(1+e^{X_{ij}}\right)\right],
\end{equation*}
we claim that 
\begin{enumerate}
	\item[a)] $\tilde{X} \succeq 0$. I.e., $\tilde{X}\in\mathcal{H}$ is feasible.
	\item[b)] $f(\tilde{X}) \leq f(X)$. I.e., $\tilde{X}$ will be no worse than $X$.
\end{enumerate}
Combining a) and b) will show that the optimal solution $X^*$ must be a DBC matrix. Below, we will prove these claims.
\vspace{\baselineskip}

{\noindent{a)}}
We begin by computing the eigenvalues of the DBC matrix $\tilde{X}$ and substituting the \Cref{Eq:Thm2_proof_x_tilde_c1,Eq:Thm2_proof_x_tilde_c2,Eq:Thm2_proof_x_tilde_c3}:
	\begin{align*}
		\lambda_1(\tilde{X}) & ={} (m - 1)x_1 + x_3 + m x_2	\\
					 & ={} \frac{1}{2m}\sum_{(i,j)\in\mathcal{E}_0}X_{ij} + \frac{1}{2m}\sum_{i=1}^{2m}X_{ii} + \frac{1}{2m}\sum_{(i,j)\in\mathcal{E}_1}X_{ij} \\
					 & ={} \frac{1}{2m}\sum_{i,j}X_{ij} \\
					 & ={} \frac{1}{2m} \bm{1}_{2m}^\top X \bm{1}_{2m} \\
					 & \geq{} 0 \\
					 \\
		\lambda_2(\tilde{X}) & ={} (m - 1)x_1 + x_3 - m x_2	\\
					 & ={} \frac{1}{2m}\sum_{(i,j)\in\mathcal{E}_0}X_{ij} + \frac{1}{2m}\sum_{i=1}^{2m}X_{ii} - \frac{1}{2m}\sum_{(i,j)\in\mathcal{E}_1}X_{ij} \\
					 & ={} \frac{1}{2m} \left[\bm{1}_{m}^\top, -\bm{1}_m^\top\right] X 
						   \begin{bmatrix}
						 	  \phantom{-}\bm{1}_m \\
							  -\bm{1}_m
						   \end{bmatrix} \\
					 & \geq{} 0 \\
					 \\
		\lambda_3(\tilde{X}) & ={} x_3 - x_1	\\
					 & ={} \frac{1}{2m^2-2m} \left[(m-1)\sum_{i=1}^{2m}X_{ii} - \sum_{(i,j)\in\mathcal{E}_0}X_{ij}\right]. \\	
	\end{align*}

To show that $\lambda_3(\tilde{X}) \geq 0$, we first prove the below propostion:
\begin{prp}
\label{Prop: PSD_trace_geq_avg}
	If an $m \times m$ matrix $X \succeq 0$, then 
	\begin{equation*}
		\text{Tr}(X)\geq \frac{1}{m}\bm{1}_m^\top X \bm{1}_m
	\end{equation*}
\end{prp}

\begin{proof}
Let the eigen-decomposition of $X$ be given as follows
	\begin{equation*}
		X = U \Lambda U^\top,
	\end{equation*}
where $U = [\bm{u}_1,\ldots,\bm{u}_m]$ and $\Lambda = \text{Diag}\{\lambda_1,\ldots,\lambda_m\}$.
Then, we have 
	\begin{equation*}
		\text{Tr}(X) = \text{Tr}(U \Lambda U^\top) = \text{Tr}(\Lambda U^\top U) = \sum_{i=1}^n\lambda_i.
	\end{equation*}
And
	\begin{align*}
		\frac{1}{n}\bm{1}_m^\top X \bm{1}_m &={} \left(\frac{1}{\sqrt{n}}\bm{1}_m^\top U \right)\Lambda \left(U^\top\frac{1}{\sqrt{n}}\bm{1}_m\right) \\
		&={} \sum_{i=1}^n \lambda_i \left(\frac{1}{\sqrt{n}}\bm{1}_m^\top \bm{u}_i\right)^2 \\
		&\leq{} \sum_{i=1}^n\lambda_i \left\Vert\frac{1}{\sqrt{n}}\bm{1}_m\right\Vert \left\Vert\bm{u}_i\right\Vert \\
		&={} \sum_{i=1}^n\lambda_i.
	\end{align*}
Therefore,
	\begin{equation*}
		\text{Tr}(X)\geq \frac{1}{n}\bm{1}_m^\top X \bm{1}_m.
	\end{equation*}
\end{proof}

We divide $X$ into $4$ $m \times m$ block matrices as
\begin{equation*}
  X ={} \left(\begin{array}{c|c}
    X_{11} & X_{12} \\\hline
    X_{21} & X_{22}
  \end{array}\right).
\end{equation*}
Note that $X_{11} \succeq 0$ and $X_{22} \succeq 0$. To see this, for any $\bm{a}\in\mathbb{R}^m$, we have $\bm{a}^\top X_{11}\bm{a} = \left[\bm{a}^\top, \bm{0}\right] X \begin{bmatrix} \bm{a} \\ \bm{0} \end{bmatrix} \geq 0$ and $\bm{a}^\top X_{22} \bm{a}= \left[\bm{0}, \bm{a}^\top\right] X \begin{bmatrix} \bm{0} \\ \bm{a}  \end{bmatrix} \geq 0$. Therefore, by \Cref{Prop: PSD_trace_geq_avg}, we have
\begin{equation*}
	\text{Tr}(X_{11}) \geq \frac{1}{m}\bm{1}_m^\top X \bm{1}_m.
\end{equation*}
Or equivalently,
\begin{equation*}
	 (m-1)\sum_{i=1}^m X_{ii} \geq \sum_{i\neq j,\ i, j\leq m}X_{ij}.  	
\end{equation*}
Similarly with $X_{22}$, we have
\begin{equation*}
	 (m-1)\sum_{i=m+1}^{2m} X_{ii} \geq \sum_{i\neq j,\ i,j\geq m+1}X_{ij}.  	
\end{equation*}

Note that $\mathcal{E}_0 = \{i\neq j \ \vert \ i, j\leq m \text{ or } i,j\geq m+1\}$. Adding the above two equations yields
\begin{equation*}
	 (m-1)\sum_{i=1}^{2m} X_{ii} \geq \sum_{(i,j)\in\mathcal{E}_0}X_{ij},
\end{equation*}
which shows that $\lambda_3(\tilde{X})\geq 0$. This concludes our proof of $\tilde{X} \succeq 0$.
\vspace{\baselineskip}

{\noindent{b)}}
To show that $\tilde{X}$ has a better cost, we will use convexity. Specifically, we define 
\begin{equation*}
	\Psi(x; \alpha, \beta) \defeq \alpha \ln(1 + e^{-x}) + \beta\ln(1+e^x).
\end{equation*}
And we can rewrite $f(X)$ as 

\begin{align*}
	f(X) ={}& \phantom{k}\sum_{(i,j)} \left[\bar{n}_{ij}^+\ln\left(1+e^{-X_{ij}}\right)+\bar{n}_{ij}^-\ln\left(1+e^{X_{ij}}\right)\right] \\
	     ={}& \sum_{(i,j)\in\mathcal{E}_0}\Psi(X_{ij}; \alpha_1, \beta) + \sum_{(i,j)\in\mathcal{E}_1}\Psi(X_{ij}; \alpha_2, \beta) + \\
	        &\phantom{k}\sum_{i=1}^{2m}\Psi(X_{ii}; \alpha_3, \beta).
\end{align*}

Since for any $\alpha, \beta>0$,
\begin{equation*}
	\Psi''(x; \alpha, \beta) = \frac{e^x(\alpha + \beta)}{(1+e^{x})^2} > 0.
\end{equation*}
We know that $\Psi(x; \alpha, \beta)$ is strictly convex with respect to $x$ for any positive $\alpha$ and $\beta$. With \Cref{Eq:Thm2_proof_c1,Eq:Thm2_proof_c2,Eq:Thm2_proof_c3} in mind, we have
\begin{align*}
	\frac{1}{2m^2-2m}\sum_{(i,j)\in\mathcal{E}_0}\Psi(X_{ij}; \alpha_1, \beta) & \geq \Psi\left(x_1; \alpha_1, \beta\right), \\
	\frac{1}{2m^2}\sum_{(i,j)\in\mathcal{E}_1}\Psi(X_{ij}; \alpha_2, \beta) & \geq \Psi\left(x_2; \alpha_2, \beta\right), \\
	\frac{1}{2m}\sum_{i=1}^{2m}\Psi(X_{ii}; \alpha_3, \beta) & \geq \Psi\left(x_3; \alpha_3, \beta\right).
\end{align*}

Therefore
\begin{align*}
	f(\tilde{X}) ={}& \sum_{(i,j)\in\mathcal{E}_0}\Psi(x_1; \alpha_1, \beta) + \sum_{(i,j)
						\in\mathcal{E}_1}\Psi(x_2; \alpha_2, \beta) + \\
	       				 &\phantom{k}\sum_{i=1}^{2m}\Psi(x_3; \alpha_3, \beta) \\
   				 ={}& (2m^2-2m)\Psi(x_1; \alpha_1, \beta) + 2m^2 \Psi(x_2; \alpha_2, \beta) + \\
   				 		&\phantom{(}2m\Psi(x_3; \alpha_3, \beta) \\
   				\leq{}& \sum_{(i,j)\in\mathcal{E}_0}\Psi(X_{ij}; \alpha_1, \beta) + \sum_{(i,j)\in\mathcal{E}_1}\Psi(X_{ij}; \alpha_2, \beta) + \\
	        			&\phantom{k}\sum_{i=1}^{2m}\Psi(X_{ii}; \alpha_3, \beta)  \\
	        	 ={}& f(X).
\end{align*}

By far, we have shown that the DBC matrix $\tilde{X}$ we constructed is in the feasible set and has a lower cost. Therefore, we conclude that the optimal solution matrix $X^*$ must be a DBC matrix. Without the loss of generality, we can assume $X=Z_{2m}(c_1, c_2, c_3)$, and $f(X)$ reduces to (up to a constant scaling)
\begin{align*}
	        f_3(c_1,c_2&,c_3)\defeq{} \\ 
	& (m-1) \Psi(c_1, \alpha_1, \beta) + m \Psi(c_2, \alpha_2, \beta) + \Psi(c_3, \alpha_3, \beta).
\end{align*}

The optimization problem \Cref{Eq:Thm2_optimization} is equivalently transformed into
\begin{align}
	(c_1^*, c_2^*, c_3^*) ={} &\argmin \  f_3(c_1, c_2, c_3)  \label{Eq:Thm2_proof_3_var_opt} \\
           \text{s.t: }\qquad & c_1 \leq c_3 \notag \\
	                          & \vert c_2\vert \leq \frac{m-1}{m}c_1 + \frac{1}{m}c_3. \notag
\end{align}

{\noindent{\bf Step 2}}.

In this step, we will prove that the optimal solution to \eqref{Eq:Thm2_proof_3_var_opt} must satisfy $c_1^* = c_3^*$. Specifically, we have the below proposition: 
\begin{prp}
\label{Prop:Thm2_proof_3to2}
	Let $\bar{c}_{13}\defeq \frac{m-1}{m}c_1 + \frac{1}{m}c_3$. If $(c_1, c_2, c_3)$ is a feasible solution to optimization problem \eqref{Eq:Thm2_proof_3_var_opt}, then $(\bar{c}_{13}, c_2, \bar{c}_{13})$ is also feasible, and its cost is no worse than $(c_1, c_2, c_3)$. I.e.,
	\begin{equation*}
		f_3(\bar{c}_{13}, c_2, \bar{c}_{13}) \leq f_3(c_1, c_2, c_3).
	\end{equation*}
\end{prp}
		
\begin{proof} \hfill

We first show that $(\bar{c}_{13}, c_2, \bar{c}_{13})$ is feasible. The first constraint of \eqref{Eq:Thm2_proof_3_var_opt} holds as we have the same value in the first and third argument. It remains to verify the second constraint
\begin{align*}
	\frac{m-1}{m}\bar{c}_{13} + \frac{1}{m}\bar{c}_{13} = \bar{c}_{13} = \frac{m-1}{m}c_1 + \frac{1}{m}c_3 \geq \vert c_2 \vert,
\end{align*}
where the last inequality is exactly the second constraint for $(c_1, c_2, c_3)$ and holds because of its feasibility.

Next, we show that $f_3(\bar{c}_{13}, c_2, \bar{c}_{13}) \leq f_3(c_1, c_2, c_3)$.
Expanding both sides, our goal is equivalent to
\begin{align*}
     (m-1)\Psi(\bar{c}_{13}, \alpha_1&, \beta) + \Psi(\bar{c}_{13}, \alpha_3, \beta) \\
 \leq{} &    (m-1)\Psi(c_{1\phantom{3}}, \alpha_1, \beta) + \Psi(c_{3\phantom{1}}, \alpha_3, 
             \beta).
\end{align*}
Collecting terms, it is equivalent to show that
\begin{align}
       (m-1)(\Psi(\bar{c}_{13}, \alpha_1&, \beta) - \Psi(c_1, \alpha_1, \beta))  \notag \\
  \leq{} & \Psi(c_3, \alpha_3, \beta) - \Psi(\bar{c}_{13}, \alpha_3, \beta).\label{Eq:Thm2_proof_3to2_goal}
\end{align}

Let $\delta \defeq \bar{c}_{13} - c_1$, and expanding $\bar{c}_{13}$ we can verify that $c_3 - \bar{c}_{13} = (m-1) \delta$. The right hand side of \Cref{Eq:Thm2_proof_3to2_goal} can be rewritten as
\begin{align*}
	& \Psi(c_3, \alpha_3, \beta) - \Psi(\bar{c}_{13}, \alpha_3, \beta) \\
	= & \Psi(\bar{c}_{13} + (m-1) \delta, \alpha_3, \beta) - \Psi(\bar{c}_{13}, \alpha_3, \beta) \\
	= & \sum_{i=1}^{m-1}\Psi(\bar{c}_{13} + i \delta, \alpha_3, \beta) - \Psi(\bar{c}_{13} + (i-1)\delta, \alpha_3, \beta).
\end{align*}
In order to show that it is greater or equal than the left hand side of \Cref{Eq:Thm2_proof_3to2_goal}, it suffices to show that $\forall i\in\{1, \ldots, m-1\}$, 
\begin{align}
	\Psi(\bar{c}_{13} + i \delta, \alpha_3, \beta) -& \Psi(\bar{c}_{13} + (i-1)\delta, \alpha_3, \beta) \notag \\
	\geq & \Psi(\bar{c}_{13}, \alpha_1, \beta)- \Psi(\bar{c}_{13} - \delta, \alpha_1, \beta).
		\label{Eq:Thm2_proof_3to2_goal2}
\end{align}
Both sides of \Cref{Eq:Thm2_proof_3to2_goal2} are in the form of the difference between the $\Psi()$ function value of two points. Since $\Psi()$ is smooth with respect to $x$, the difference can be written as an integral of the derivative $\Psi'()$ between the two points. Specifically, the left hand side of \eqref{Eq:Thm2_proof_3to2_goal2}
\begin{align*}
	&\Psi(\bar{c}_{13} + i \delta, \alpha_3, \beta) - \Psi(\bar{c}_{13} + (i-1)\delta, \alpha_3, \beta) \qquad \qquad \\
	={}& \int_{\bar{c}_{13} + (i-1)\delta}^{\bar{c}_{13} + i \delta} \Psi'(t, \alpha_3, \beta)\text{d}t \\
	={}& \int_{\bar{c}_{13} -\delta}^{\bar{c}_{13}} \Psi'(t + i \delta, \alpha_3, \beta)\text{d}t.
\end{align*}
And the right hand side
\begin{align*}
	&\Psi(\bar{c}_{13}, \alpha_1, \beta)- \Psi(\bar{c}_{13} - \delta, \alpha_1, \beta) \qquad \qquad  \qquad \qquad \quad \\
	=& \int_{\bar{c}_{13} - \delta}^{\bar{c}_{13}} \Psi'(t, \alpha_1, \beta)\text{d}t.
\end{align*}
Thus, \Cref{Eq:Thm2_proof_3to2_goal2} is equivalent to
\begin{equation}
\label{Eq:Thm2_proof_3to2_goal3}
	\int_{\bar{c}_{13} -\delta}^{\bar{c}_{13}} \Psi'(t + i \delta, \alpha_3, \beta)\text{d}t \geq{} \int_{\bar{c}_{13} - \delta}^{\bar{c}_{13}} \Psi'(t, \alpha_1, \beta)\text{d}t.
\end{equation}
To prove \eqref{Eq:Thm2_proof_3to2_goal3}, it suffices to show that $\forall t\in[\bar{c}_{13} - \delta, \bar{c}_{13}]$,
\begin{equation}
\label{Eq:Thm2_proof_3to2_goal4}
	\Psi'(t + i \delta, \alpha_3, \beta) \geq{} \Psi'(t, \alpha_1, \beta).
\end{equation}
We can compute $\Psi'(x, \alpha, \beta)$ explicitly as
\begin{equation*}
	\Psi'(x, \alpha, \beta) = \beta - \frac{\alpha + \beta}{1+e^{x}}.
\end{equation*}
Thus, \eqref{Eq:Thm2_proof_3to2_goal4} is equivalent to
\begin{equation*}
	\beta - \frac{\alpha_3 + \beta}{1+e^{t + i \delta}} \geq \beta - \frac{\alpha_1 + \beta}{1+e^{t}}
\end{equation*}
or
\begin{equation*}
	\frac{\alpha_3 + \beta}{1+e^{t + i \delta}} \leq \frac{\alpha_1 + \beta}{1+e^{t}}.
\end{equation*}
Given $\alpha_1\geq\alpha_3$ ({\it cf.}\Cref{Eq:Thm2_proof_alpha_relation}) and $t\in[\bar{c}_{13} - \delta, \bar{c}_{13}]$, this inequality holds, which concludes the proof.
\end{proof}

\Cref{Prop:Thm2_proof_3to2} shows that the optimal solution to \eqref{Eq:Thm2_proof_3_var_opt} must satisfy $c_1^* = c_3^*$. Therefore, we can substitue $c_1 = c_3$ and remove $c_3$ in \eqref{Eq:Thm2_proof_3_var_opt}. This reduces $f_3(c_1, c_2, c_3)$ to (up to a constant scaling)
\begin{align*}
f_2(c_1,&c_2)\defeq{} \\ 
&(m-1) \Psi(c_1, \alpha_1, \beta) + m \Psi(c_2, \alpha_2, \beta) + \Psi(c_1, \alpha_3, \beta).
\end{align*}
And the optimization problem \Cref{Eq:Thm2_proof_3_var_opt} is equivalently transformed into
\begin{align}
	(c_1^*, c_2^*) ={} &\argmin \  f_2(c_1, c_2)  \label{Eq:Thm2_proof_2_var_opt} \\
          \text{s.t: }\qquad  & \vert c_2\vert \leq c_1. \notag
\end{align}

{\noindent{\bf Step 3}}.

We first consider the unconstrained optimal solution $(\tilde{c}_1,\tilde{c}_2)$ of optimization problem \eqref{Eq:Thm2_proof_2_var_opt}. Since $\alpha_1,\alpha_2,\alpha_3,\beta > 0$, all the $\Psi()$ functions are strictly convex and thus $f_2(c_1, c_2)$ is strictly convex. The unconstrained optimal solution is unique and can be computed by the $\nabla f_2(c_1,c_2)=0$. Denote $\bar{\alpha}_{13}\defeq{} \frac{m-1}{m}\alpha_1 + \frac{1}{m}\alpha_3$, we have 
\begin{align*}
	\frac{\partial f_2(c_1, c_2)}{\partial c_1} & ={} m\beta - \frac{m\bar{\alpha}_{13} + m\beta}{1+e^{c_1}} = 0 \\
	\frac{\partial f_2(c_1, c_2)}{\partial c_2} & ={} m\beta - \frac{m\alpha_2 + m\beta}{1+e^{c_1}} = 0,
\end{align*}
which gives us 
\begin{align*}
	\tilde{c}_1 & ={} \ln\left(\frac{\bar{\alpha}_{13}}{\beta}\right) \\
	\tilde{c}_2 & ={} \ln\left(\frac{\alpha_2}{\beta}\right).
\end{align*}
We claim that $(\tilde{c}_1, \tilde{c}_2)$ is infeasible. I.e., $\vert \tilde{c}_2\vert > \tilde{c}_1$. Given $k\geq 1$, with \Cref{alpha_2}, we have 
\begin{equation*}
	\alpha_2 ={}\frac{1}{4m^2}\left[w - \sum_{v=1}^w \lambda_2(W)^v\right]<{}\frac{w}{4m^2}\leq{}\frac{kw}{4m^2}=\beta.
\end{equation*}
Thus, $\tilde{c}_2 < 0$ and $\vert \tilde{c}_2\vert = -\tilde{c}_2$. Therefore, to show that $(\tilde{c}_1, \tilde{c}_2)$ is infeasible, we only need to prove
\begin{equation*}
	\ln\left(\frac{\beta}{\alpha_2}\right) > \ln\left(\frac{\bar{\alpha}_{13}}{\beta}\right).
\end{equation*}
Or equivalently,
\begin{equation}
\label{Eq:Thm2_proof_2to1_goal}
	\beta^2 \geq \alpha_2 \bar{\alpha}_{13}.
\end{equation}
 Recall the definition of $\alpha_1$, $\alpha_2$, and $\alpha_3$ in \Cref{alpha_1,alpha_2,alpha_3}, we have
\begin{equation}
\label{alpha_13}
\bar{\alpha}_{13}\defeq{} \frac{m-1}{m}\alpha_1 + \frac{1}{m}\alpha_3 ={} \frac{1}{4m^2}\left[w + \sum_{v=1}^w \lambda_2(W)^v\right]
\end{equation}
and
\begin{equation*}
\bar{\alpha}_{2}={}\frac{1}{4m^2}\left[w - \sum_{v=1}^w \lambda_2(W)^v\right].
\end{equation*}
Therefore,
\begin{align*}
\bar{\alpha}_{13}\alpha_2 & ={} \frac{1}{(4m^2)^2}\left[w^2 - \left( \frac{\lambda_2(W)-\lambda_2(W)^{w+1}}{1-\lambda_2(W)}\right)^2\right] \\
    & <{} \frac{w^2}{(4m^2)^2} \\
    & \leq{} \frac{k^2w^2}{(4m^2)^2} \\
    & ={} \beta^2.
\end{align*}
Thus, we have shown \Cref{Eq:Thm2_proof_2to1_goal} and thus, $(\tilde{c}_1, \tilde{c}_2)$ is infeasible.

Since the unconstrained optimal solution $(\tilde{c}_1, \tilde{c}_2)$ is infeasible, the constrained optimal solution $(c_1^*, c_2^*)$ must activate the constraint. Next, we will show that the activated constraint must be $c_1 = -c_2$.

Denote $\mathcal{L}$ the line segment joining $(c_1^*, c_2^*)$ and $(\tilde{c}_1, \tilde{c}_2)$, and $\mathcal{G}$ the feasible set of \eqref{Eq:Thm2_proof_2_var_opt}. We first claim that $\mathcal{L} \bigcap \mathcal{G} = \{(c_1^*, c_2^*)\}$ must hold. If not, assume there exists $(c_1^0, c_2^0)\neq (c_1^*, c_2^*)$ and $(c_1^0, c_2^0) \in \mathcal{L} \bigcap \mathcal{G}$. Since $(c_1^0, c_2^0)\in\mathcal{L}$, there exist $\gamma\in(0,1)$ such that 
\begin{equation*}
	(c_1^0, c_2^0) ={} \gamma(c_1^*, c_2^*) + (1-\gamma)(\tilde{c}_1, \tilde{c}_2).
\end{equation*}
Then, by convexity of $f_2$ and global optimality of $(\tilde{c}_1, \tilde{c}_2)$,
\begin{equation*}
	f_2(c_1^0, c_2^0) \leq{} \gamma f_2(c_1^*, c_2^*) + (1-\gamma)f_2(\tilde{c}_1, \tilde{c}_2)<{}f_2(c_1^*, c_2^*).
\end{equation*}
It gives us a feasible $(c_1^0, c_2^0)$ that has a lower cost, which contradicts with the constrained optimality of $(c_1^*, c_2^*)$. And thus, by contradiction, we have shown that $\mathcal{L} \bigcap \mathcal{G} = \{(c_1^*, c_2^*)\}$.

Note that, for the global optimizer $(\tilde{c}_1, \tilde{c}_2)$, we have $\tilde{c}_1 >  \tilde{c}_2$. To see this, note it is equivalent to $\bar{\alpha}_{13} > \alpha_2$, which is shown from \Cref{alpha_13,alpha_2}. For any points on the $\{(c_1,c_2)|\ c_1 = c_2 > 0\}$, the line segment joining $(c_1, c_2)$ and $(\tilde{c}_1, \tilde{c}_2)$ will intersect the feasible set $\mathcal{G}$ on infinite points, which contradicts with the claim we just proved above. Therefore, the constrained optimizer $(c_1^*, c_2^*)$ must satisfy $c_1^* = -c_2^*$.
		
Therefore, we can substitue $c_1 = -c_2$ and remove $c_2$ in \eqref{Eq:Thm2_proof_2_var_opt}. This reduces $f_2(c_1, c_2)$ to (up to a constant scaling)
\begin{align*}
f_1(&c_1)\defeq{} \\ 
&(m-1) \Psi(c_1, \alpha_1, \beta) + m \Psi(c_1, \alpha_2, \beta) + \Psi(c_1, \alpha_3, \beta).
\end{align*}
And the optimization problem reduces to
\begin{equation}
\label{Eq:Thm2_proof_1_var_opt}
	c_1^* ={} \argmin \  f_1(c_1). 
\end{equation}
		
Optimization problem \eqref{Eq:Thm2_proof_1_var_opt} has a unique optimal solution given by $f_1'(c_1)={}0$:
\begin{equation*}
c_1^* = \ln\left(\frac{\bar{\alpha}_{13} + \beta}{\alpha_2 + \beta}\right).
\end{equation*}

This gives the optimal solution $X^*$ to \eqref{Eq:Thm2_optimization} when $\mathcal{H}=\{X \ | \ X \succeq 0\}$:
\begin{equation*}
	X^*= \begin{cases}
			\phantom{-}\ln\left(\frac{\bar{\alpha}_{13}+\beta}{\alpha_2+\beta}\right), &\text{if $(i,j)\in\mathcal{E}_0$ or $i=j$} \\
			-\ln\left(\frac{\bar{\alpha}_{13}+\beta}{\alpha_2+\beta}\right), &\text{if $(i,j)\in\mathcal{E}_1$}.
 		 \end{cases} \\
\end{equation*}

{\noindent{\bf Part 4)}}

Since both $X^*(\mathbb{R}^{n\times n}) $ and $X^*(\mathbb{S}^n_{+})$ are DBC matrices, we can compute their nuclear norms from the proof of \Cref{Prop:property_DBC}. Specifically, for a DBC matrix $X = Z_{2m}(c_1, c_2, c_3)$, we have
\begin{align*}
	\Vert X \Vert_* ={}& \vert m(c_1+c_2) + (c_3-c_1) \vert +  \vert m(c_1-c_2) + (c_3-c_1) \vert  \\
	  & + (2m-2) \vert c_3-c_1 \vert.
\end{align*}
Since from part 2) we have
\begin{align*}
	X^*(\mathbb{R}^{n\times n}) = Z_{2m}\left(
	\ln\left(\frac{\alpha_1}{\beta}\right), 
	\ln\left(\frac{\alpha_2}{\beta}\right), 
	\ln\left(\frac{\alpha_3}{\beta}\right)
	\right)
\end{align*}
with 
\begin{align*}
	\alpha_1 &> \alpha_3 > 0 \\
	\alpha_1 &> \alpha_2 > 0,
\end{align*}
we have
\begin{align*}
	\Vert X^*(\mathbb{R}^{n\times n}) \Vert_* ={}& m\ln\left(\frac{\alpha_1\alpha_2}{\beta^2}\right) + \ln\left(\frac{\alpha_3}{\alpha_1}\right) +  
	\\
	& \left\vert m\ln\left(\frac{\alpha_1}{\alpha_2}\right) +  \ln\left(\frac{\alpha_3}{\alpha_1}\right) \right\vert + 
	  \\
	  & (2m-2) \ln\left(\frac{\alpha_1}{\alpha_3}\right).
\end{align*}
Note that $\alpha_i = \Theta(1/n^2)$ and $\beta = \Theta(1/n^2)$.
Therefore, $\Vert X^*(\mathbb{R}^{n\times n}) \Vert_* = \Theta(n)$. 
Similarly, from part 3) we have
\begin{align*}
	X^*(\mathbb{S}^n_{+}) = Z_{2m}\left(
	\nu_1, 
	-\nu_1, 
	\nu_1
	\right)
\end{align*}
where $\nu_1 := \ln\left(\frac{\bar{\alpha}_{13}+\beta}{\alpha_2+\beta}\right)$ with $\bar{\alpha}_{13} \defeq \frac{m-1}{m}\alpha_1 + \frac{1}{m}\alpha_3$. We can get $\bar{\alpha}_{13} = \Theta(1/n^2)$ which leads to $\nu_1 = \Theta(1)$. And
\begin{align*}
	\Vert X^*(\mathbb{S}^n_{+}) \Vert_* ={}& 2m\nu_1 = \Theta(n).
\end{align*}
$\hfill \square$

\section{Neural Network Implementation}
\label{appendix_exp_details}
In this appendix we detail our neural network implementation of VEC (or {\ErgoVEC}). We first list the structure of the neural network. After that, we will detail our construction of training set (samples and labels). Next, we show that the neural network optimization objective is exactly the objective of VEC (or {\ErgoVEC}). Lastly, we provide the learning and optimization settings we used in training.

{\bf{Structure of the neural network.}} \Cref{fig:app_neural_network_structure} illustrates the structure of our neural network. There are four layers in the neural network.
\begin{enumerate}
	\item Input layer. This layer receives a one-hot vector encoders for each nodes in a pair $(i,j)\in\mathcal{V}^2$ as the input of this neural network
	\item Embedding layer. The embedding layer is a $n \times d$ matrix where the rows represent the $d$ dimensional embedding vectors of nodes. These vectors are updated in the optimization iteration after each epoch. After the optimization process, they will be used as final output of the VEC (or {\ErgoVEC}) algorithm. Please note that this is the only layer that will be updated in the entire optimization process. In the neural network, this layer takes the two one-hot vectors from the input layer and returns the two corresponding row vectors to the next layer. 
	\item Dot Product layer. This layer takes two embedding vectors and returns the dot product between them.
	\item Output layer. This layer takes a scalar (the dot product from previous layer), and returns the sigmoid function value $S(x) \defeq \frac{1}{1 + e^{-x}}$ of it as the output of this neural network.
\end{enumerate}
To sum up, this neural network takes a pair of nodes $(i,j)$ as input and returns the sigmoid function of their dot product $\hat{y}_{(i,j)}$as output.

\begin{figure}[H]
\centering
	\begin{subfigure}{\linewidth}
        \centering
		\includegraphics[width=\linewidth]{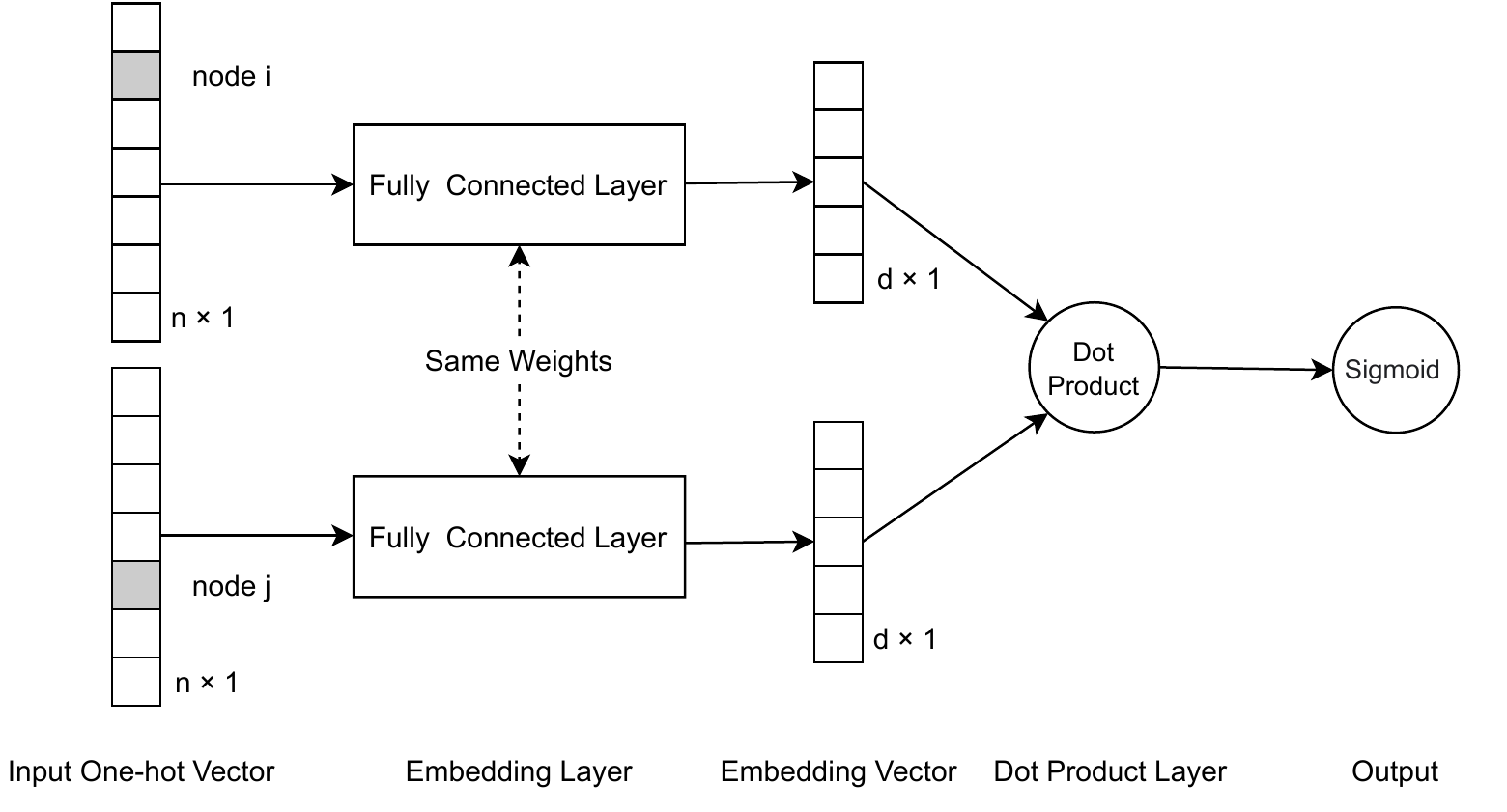}
    \end{subfigure}
\caption{\small Structure diagram of the neural network implemented with Keras package.}
\label{fig:app_neural_network_structure}
\end{figure}

{\bf{Training set and loss function.}} We use a weighted training set $D=\{((i,j), w_{(i,j)}, y_{(i,j)})\}$ obtained from the union of two parts: a weighted positive set and a weighted negative set. Both sets contain all node pairs $(i,j)\in\mathcal{V}^2$, but label and weigh them differently. All node pairs in the positive set arelabeled $1$ with weights equal $n_{ij}^+$ ($\bar{n}_{ij}^+$ for {\ErgoVEC}), whereas node pairs in the negative set are labeled as $0$ with weights equal $n_{ij}^-$ ($\bar{n}_{ij}^-$ for {\ErgoVEC}). The training set is randomly shuffled and fed in the neural network during each epoch. For the loss function, we choose binary cross entropy
\begin{equation*}
	H(D) \defeq{} \frac{1}{N}\sum_{(i,j)\in D} H(i,j),
\end{equation*}
where
\begin{equation}
H(i,j) ={} - w_{(i,j)}\Big[y_{(i,j)}\ln(\hat{y}_{(i,j)}) + (1-y_{(i,j)})\ln(1-\hat{y}_{(i,j)})\Big].
\label{Eq:app_bin_entropy}
\end{equation}
{\bf Equivalence proof.} Here we show that the neural network equipped with this training set and loss function has the exact same objective as VEC. (For {\ErgoVEC}, the same holds after replacing $n_{ij}^+$ with $\bar{n}_{ij}^+$ in the following equations.) First note that, for $(i,j)$ in positive set, $y_{(i,j)} = 1$, 
\begin{align*}
	H(i,j) &= - n_{ij}^+\ln(\hat{y}_{(i,j)}) = - n_{ij}^+\ln{(S(\bm{u}_i^\top\bm{u}_j)}) \\
	&= n_{ij}^+\sigma(+\bm{u}_i^\top\bm{u}_j)
\end{align*}
and for $(i,j)$ in negative set, $y_{(i,j)} = 0$,
\begin{align*}
	H(i,j) = - n_{ij}^-\ln(1 - \hat{y}_{(i,j)}) &= - n_{ij}^-\ln(1-S(\bm{u}_i^\top\bm{u}_j)) 
	\\
	&= n_{ij}^+\sigma(-\bm{u}_i^\top\bm{u}_j).
\end{align*}
Therefore, 
\begin{align*}
	 H(D) &\defeq{} \frac{1}{N}\sum_{(i,j)\in D} H(i,j) \\
	 &={} \sum_{(i,j)\in \mathcal{V}^2 }\left[n_{ij}^+ \;\sigma(\mathbf{u}_{i}^{\top}\mathbf{u}_{j})
  + n_{ij}^{-}\; \sigma( -\mathbf{u}_{i}^{\top}\mathbf{u}_{j} )\right],
\end{align*}
which is the same as \eqref{Eq:VEC}.

{\bf{Optimization paramters.}} We used the Adam optimizer with default parameter choice except for learning rate. We set learning rate as described in Table~\ref{Table:learning_rate}, although we want to make a note that the optimal learning rates do depend on specific graph realizations.

\begin{table}[!ht]
\centering
\vspace{0.5em}
\renewcommand{\arraystretch}{1.5}
\begin{tabular}{|c|c|c|c|c|}
	\hline
	 & Algorithm & $n$ & l.r. & \# epochs \\
	\hline
	\multirow{8}{*}{\begin{tabular}{@{}c@{}}Linear \\ Degree \\ Regime\end{tabular} } & VEC & 100 & 0.001 & 400 \\
	& VEC & 200 & 0.001 & 200 \\
	& VEC & 500 & 0.001 & 80 \\
	& VEC & 1000 & 0.001 & 40 \\
	\cline{2-5}
	& {\ErgoVEC} & 100 & 0.02 & 400 \\
	& {\ErgoVEC} & 200 & 0.02 & 200 \\
	& {\ErgoVEC} & 500 & 0.02 & 80 \\
	& {\ErgoVEC} & 1000 & 0.02 & 40 \\
	 \hline
	\multirow{8}{*}{\begin{tabular}{@{}c@{}}Logarithmic \\ Degree \\ Regime\end{tabular} } & VEC & 100 & 0.001 & 400 \\
	& VEC & 200 & 0.00021 & 1500 \\
	& VEC & 500 & 0.001 & 200 \\
	& VEC & 1000 & 0.001 & 200 \\
	\cline{2-5}
	& {\ErgoVEC} & 100 & 0.0025 & 400 \\
	& {\ErgoVEC} & 200 & 0.00021 & 1500 \\
	& {\ErgoVEC} & 500 & 0.0025 & 200 \\
	& {\ErgoVEC} & 1000 & 0.0025 & 200 \\
	\hline
\end{tabular}
\renewcommand{\arraystretch}{1}
\caption{\small Learning rates and number of epochs used in each experiment.}
\label{Table:learning_rate}
\end{table}

{\bf{Remarks on convergence.}} In our experiments, we note that the objective functions seem to converge after a number of epochs, but the embedding vectors do not. The convergence behavior over epochs is shown in Fig.~\ref{fig:app_neural_network_keras_convergence} with the plot of the loss function as a function of number of epochs displayed in Fig.~\ref{fig:app_neural_network_keras_convergence}~(a). Changes in the embedding vectors measured by the ratio of the Procrustes distance between embedding vectors in consecutive epochs and the Frobenius norm of the embedding vectors in the previous epoch are displayed in Fig.~\ref{fig:app_neural_network_keras_convergence}~(b). We observe that the loss function drops quickly after the first few epochs and remains essentially flat after $1000$ epochs, but the change in the embedding vectors is bounded away from $0$ even after $1500$ epochs. A possible explanation for this behavior is that many neural network implementations and optimization procedures, including the Keras package that we used, focus on the convergence of the objective loss rather than the convergence of layer weights. Although this is very useful in various applications, it may be inadequate for finding the optimal numerical solution (the minimizing weights). Future work could attempt improving our implementation to overcome such limitations.
\begin{figure}[H]
\centering
\centering
    \begin{subfigure}{0.9\linewidth}
		\includegraphics[width=0.9\linewidth]{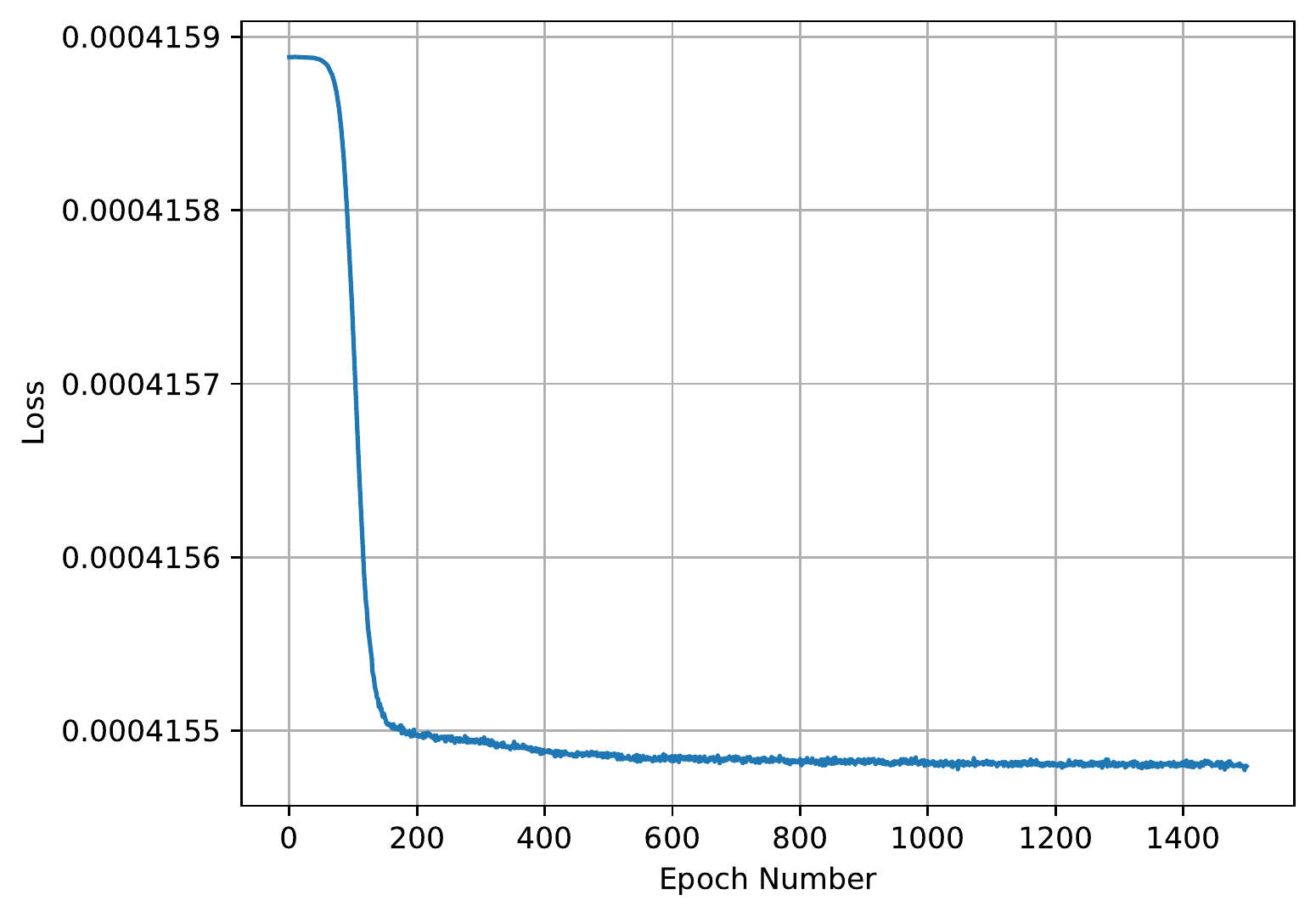}
        \caption{\small Loss function value as a function of number of epochs.}
    \end{subfigure}
    
    \begin{subfigure}{0.9\linewidth}
		\includegraphics[width=0.9\linewidth]{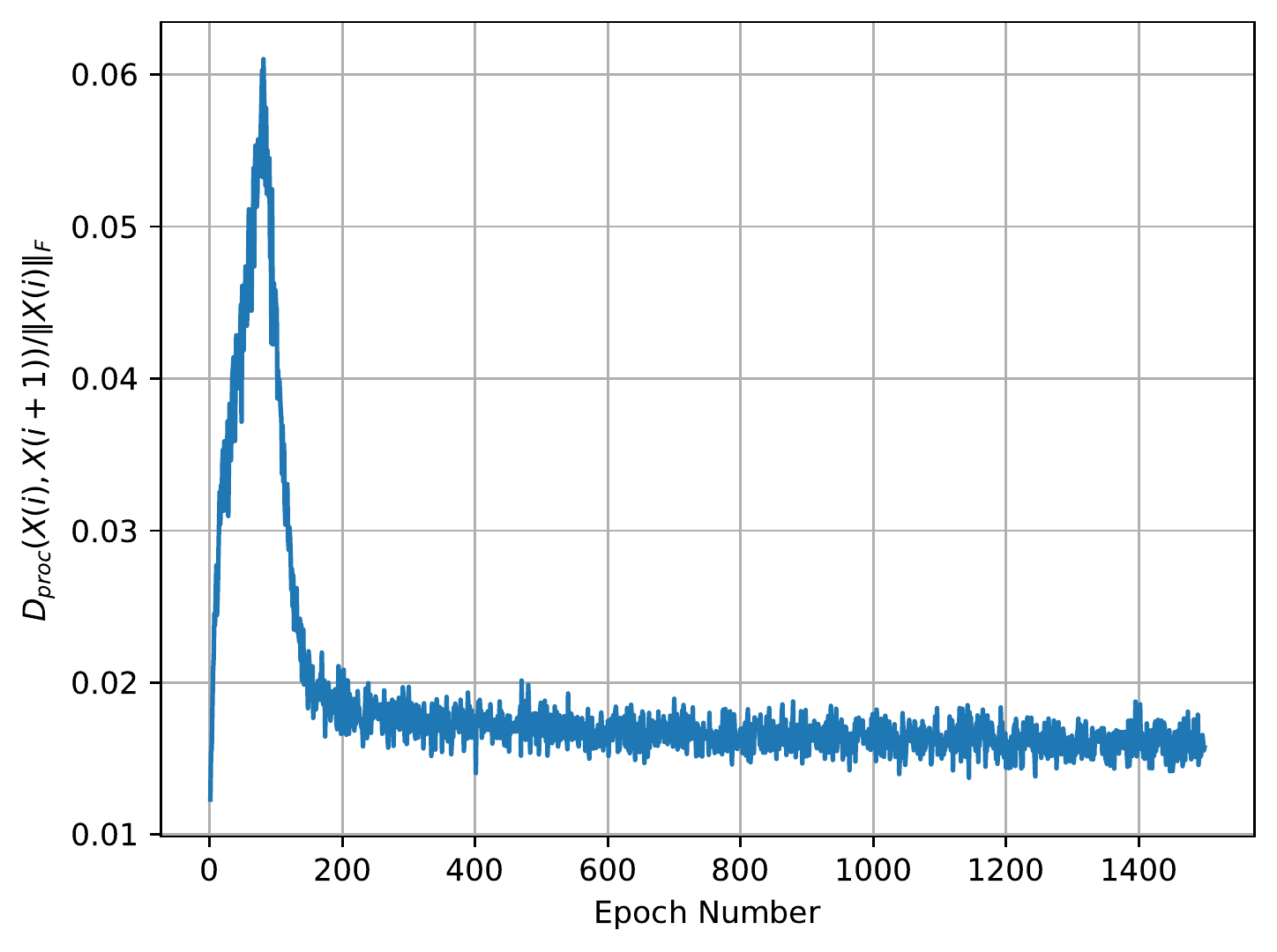}
        \caption{\small Change in embedding vectors versus epochs. The change is computed as the ratio of the Procrustes distance between embedding vectors in epoch $i$ and $i+1$ and the Frobenius norm of the embedding vectors in epoch $i$.}
    \end{subfigure}
\caption{\small Illustrating potential convergence issues associated with neural-network-based optimization of node embedding objectives.}
\label{fig:app_neural_network_keras_convergence}
\end{figure}

\ifCLASSOPTIONcompsoc
  \section*{Acknowledgments}
\else
  \section*{Acknowledgment}
\fi

This work was supported in part by the U.S.~National Science Foundation under grant 1527618, the Department of Electrical and Computer Engineering, and the Division of Systems Engineering at Boston University. Any opinions, findings, and conclusions or recommendations expressed in this material are those of the author(s) and do not necessarily reflect the views of the supporting institutions.

\ifCLASSOPTIONcaptionsoff
  \newpage
\fi

\bibliographystyle{IEEEtran}

\bibliography{references.bib}

\end{document}